\documentclass[10pt,twocolumn,letterpaper]{article}
\pdfoutput=1

\usepackage{cvpr}
\usepackage{times}
\usepackage{epsfig}
\usepackage{graphicx}
\usepackage{amsmath}
\usepackage{amssymb}
\usepackage{amsthm}

\usepackage[font=normalsize]{subfig}
\usepackage{xcolor}
\usepackage{lipsum}
\usepackage[ruled,vlined]{algorithm2e}
\usepackage{multirow}

\usepackage[pagebackref=true,breaklinks=true,letterpaper=true,colorlinks,bookmarks=false]{hyperref}

\cvprfinalcopy 

\def\cvprPaperID{3475} 

\graphicspath{./figs}

\newcommand{\pder}[2][]{\frac{\partial#1}{\partial#2}}
\newcommand{\ourmodel}{DARNet}
\newtheorem{prop}{Proposition}

\ifcvprfinal\fi

\begin{document}

\title{DARNet: Deep Active Ray Network for Building Segmentation}

\author{Dominic Cheng\textsuperscript{1, 2} \qquad Renjie Liao\textsuperscript{1, 2, 3} \qquad Sanja Fidler\textsuperscript{1, 2, 4} \qquad Raquel Urtasun\textsuperscript{1, 2, 3}\\
University of Toronto\textsuperscript{1} \quad Vector Institute\textsuperscript{2} \quad Uber ATG Toronto\textsuperscript{3} \quad NVIDIA\textsuperscript{4}\\
{\tt\small \{dominic, rjliao, fidler\}@cs.toronto.edu} \quad {\tt\small urtasun@uber.com}
}

\maketitle
\thispagestyle{empty}

\begin{abstract}
In this paper, we propose a Deep Active Ray Network (DARNet) for automatic building segmentation. Taking an image as input, it first exploits a deep convolutional neural network (CNN) as the backbone to predict energy maps, which are further utilized to construct an energy function. A polygon-based contour is then evolved via minimizing the energy function, of which the minimum defines the final segmentation. Instead of parameterizing the contour using Euclidean coordinates, we adopt polar coordinates, i.e., rays, which not only prevents self-intersection but also simplifies the design of the energy function. Moreover, we propose a loss function that directly encourages the contours to match building boundaries. Our DARNet is trained end-to-end by back-propagating through the energy minimization and the backbone CNN, which makes the CNN adapt to the dynamics of the contour evolution. Experiments on three building instance segmentation datasets demonstrate our DARNet achieves either state-of-the-art or comparable performances to other competitors.
\end{abstract}


\section{Introduction}
\label{sec:intro}

The ability to automatically extract building footprints from aerial imagery is an important task in remote sensing. 
It has many applications such as cartography, urban planning, and humanitarian aid.
While maps in well-established urban areas provide precise definitions of building outlines, in more general situations, this information may be neither up-to-date nor available altogether. 
Such is the case in remote population centers and in dynamic scenarios caused by rapid urban development or natural disasters. 
These situations motivate the use of automatic building segmentation to form an understanding of an area.

Convolutional neural networks (CNNs) have rapidly established themselves as the de facto standard for tasks of semantic and instance segmentation, as demonstrated by their impressive performance across a variety of datasets \cite{zhao2017pyramid, he2017mask, liu2018path}. 
When applied to the task of building segmentation, however, there exists room for improvement. 
Specifically, as demonstrated in \cite{wang2017torontocity}, while CNNs are able to generate dense, pixel-based predictions with high recall, they have issues with precise delineation of building boundaries. 

Arguably, these boundaries are the most useful features in defining the shape and location of a building. Motivated by this, Marcos \etal proposed the deep structured active contours (DSAC)~\cite{marcos2018learning} model, which combines the power of deep CNNs with the classic polygon-based active contour model of Kass \etal~\cite{kass1988snakes}. This fits in particularly well with our prior understanding of buildings, which are generally laid out as relatively simple polygons. DSAC uses a deep CNN to predict an energy landscape on which the active contour, also known as a snake, crawls.
When the energy reaches the minimum, the snake stops and the enclosed region is the final predicted segmentation. 

\begin{figure}[t]
	\begin{center}
		\includegraphics[width=0.99\linewidth]{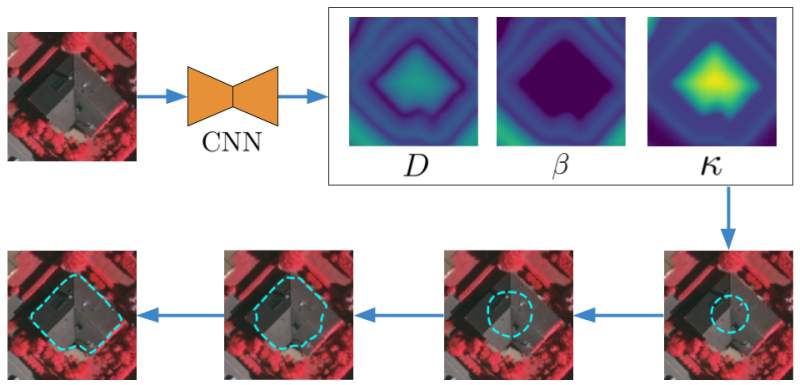}
	\end{center}	
	\caption{We introduce our {\ourmodel} framework, which approaches the problem of instance segmentation by defining a contour using active rays that evolve according to energies parameterized by a CNN. Given an input image, a CNN outputs three maps that define energies. An initialized contour then moves to minimize the energy, yielding an instance segmentation.}
	\label{fig:method}
	\vspace{-0.3cm}
\end{figure}

Although DSAC's contours exhibit improved coverage compared to CNN based segmentation, they can still be blob-like without strict adherence to building boundaries. 
Additionally, the representation of contour points with two degrees of freedom presents some challenges. Most notably, it results in extra computational overhead to minimize the proposed energy, and also allows for some contours to exhibit self-intersections.
To address these limitations, we introduce the Deep Active Ray Network ({\ourmodel}), a framework based on a polar representation of active contours, traditionally known as active rays. 
This ray-based parameterization provides several advantages: 1) contour self-intersections are completely eliminated; 2) it allows us to propose a simpler energy function; 3) the parameterization lends itself to a new loss function that encourages contours to match building boundaries.
Our {\ourmodel} also exploits a deep CNN as the backbone network for predicting the energy landscape.
We train the whole model end-to-end by back-propagating through the energy minimization and the backbone CNN.
We compare {\ourmodel} against DSAC on three datasets, Vaihingen, Bing Huts, and TorontoCity, and demonstrate its improved effectiveness in producing segmentations that better align with building boundaries.

\section{Related Work}
\label{sec:related}

\textbf{Active Contour Models:} 
First introduced in \cite{kass1988snakes} by the name of \textit{snakes}, active contour models proved to be an extremely popular approach to image segmentation. 
In these models, an energy is defined as a functional, and its minimization yields a contour that describes a segmentation. The description of this energy is based on intuitive geometric priors and leverages features such as the intensity image and its gradient. 
Of the myriad works that followed, \cite{cohen1991active} proposed a balloon force to avoid the tendency for snakes to collapse when initialized far from object edges. 
\cite{denzler1999active} reformulated snakes in a polar representation known as \textit{active rays} to reduce the energy optimization from two dimensions to one, and addressed the issue of contour collapse by introducing energies that encouraged circular shapes. 
Our approach leverages the parameterization of active rays, but we elect to use the curvature term proposed in snakes, as our application of interest does not typically contain circular boundaries. 
Furthermore, we propose a novel balloon energy that does not involve computing normals at every point in our contour, but rather exploits the properties of a polar parameterization to achieve desired effect in a manner that can be efficiently computed, and fits in seamlessly with our contour inference.

\textbf{Deep Active Contour Models:} 
\cite{rupprecht2016deep} proposed to combine deep learning with the active contours framework of \cite{sundaramoorthi2007sobolev} by having a CNN predict a vector field for a patch around a given contour point to guide it to the closest boundary. 
However, this method is unable to be learned end-to-end, as CNN learning is separated from the active contour inference. 
\cite{hu2017deep,zian19levelset} propose to combine a CNN with a level set method~\cite{osher1988fronts} in an end-to-end differentiable manner. 
In contrast to level sets, which define contours implicitly, snakes provide an explicit representation of contour points allowing for the definition of energies based on geometric intuition. 
\cite{marcos2018learning} uses the snakes framework and replaces the engineered features with ones learned by a CNN. 
The problem is posed under the formulation of structured output prediction, and the CNN is trained using a structured support vector machine (SSVM) hinge loss \cite{tsochantaridis2005large} to optimize for intersection-over-union. 
In contrast, we propose to use an active rays parameterization alongside a largely simplified energy functional.
We also propose a loss function that encourages sharper, better aligned contours. 
Additionally, we back-propagate our loss through the contour evolution, \ie, the energy minimization.
It is interesting to note that there are other deep learning based approaches for predicting polygon-based contours.
For example, rather than representing a polygon as a discretization of some continuous function,~\cite{castrejon2017annotating, acuna2018efficient} use a recurrent neural network (RNN) to directly predict the polygon vertices in a sequential manner. In~\cite{curvegcn}, the authors predict the polygon or spline outlining the object using a Graph Convolutional Network.

\textbf{Building Segmentation:} Current state-of-the-art approaches to building segmentation typically incorporate CNNs in a two stage pipeline: identify instances, and extract polygons. 
As shown in \cite{wang2017torontocity}, instances can be extracted from connected components of the semantic mask predicted by a semantic segmentation network \cite{long2015fully, he2016deep}, or directly predicted with an instance segmentation network \cite{bai2017deep}. 
A polygon is then extracted from the instance mask. Because the output space of these approaches are individual pixels, the networks do not reason about the geometry of its predictions, resulting in segmentations that are blob-like around building boundaries. In contrast, the concept of a polygonal output is directly embedded in our model, where we can encourage our outputs to match ground truth shapes.

\section{Our Approach}
\label{sec:method}

In this section, we introduce our {\ourmodel} model. 
Given an input image, our CNN predicts feature maps that define an energy landscape. 
We place a polygon-based contour, parameterized using polar coordinates, at an initial position on this landscape. 
This contour then evolves to minimize its energy via gradient descent, and its resting position defines the predicted instance segmentation. 
We refer the reader to Figure~\ref{fig:method} for an illustration of our model.
In the subsequent sections, we first describe our parametrization of the contour, highlighting its advantages, such as avoiding self-intersections. 
We then introduce our energy formulation, in particular our novel balloon energy, and the contour evolution process. 
Last, we explain how to learn our model in an end-to-end manner.

\subsection{Contour Representation}\label{ssec:contour_repr}

Recall that in the active contour (or snake) model a contour point $v$ is represented as 
\begin{equation} \label{eq:snake}
v(s) = 
\begin{bmatrix}
x_v(s) \\ y_v(s)
\end{bmatrix}
\end{equation} 
where $s$ denotes the arc length and is generally defined over an interval $s \in [0, 1]$. 
Note that by varying the arc length $s$ from $0$ to $1$, we obtain the contour.
Since this parameterization adopts separate functions for $x$ and $y$ coordinates, the contour point is free to move in any direction, which may cause self-intersection as shown in Figure~\ref{fig:multiple_crossing}. 
In contrast, in this paper, we propose to use a parametrization that implicitly avoids self-intersection. 

\begin{figure}[t]
	\begin{center}
		\subfloat[]{\includegraphics[height=0.4\linewidth]{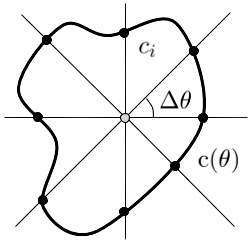}} \hspace{0.5cm}
		\subfloat[]{\includegraphics[height=0.4\linewidth]{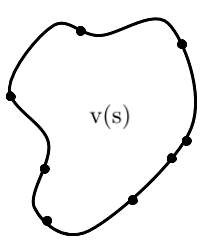}} \\
	\end{center}
	\vspace{-0.1cm}
	\caption{Contours defined by (a) active rays and (b) snake.}
	\label{fig:active_ray}
\end{figure}

\paragraph{Active Rays:} Inspired by~\cite{denzler1999active}, we parameterize the contour via rays. 
In particular, we define a contour point $c$ as
\begin{equation}
c(\theta) = 
\begin{bmatrix}
x_c + \rho(\theta) \cos \theta \\
y_c + \rho(\theta) \sin \theta
\end{bmatrix}
\end{equation}
where $(x_c, y_c)$ define the reference point of the contour, $\rho \geq 0 $ is the  radius and $\theta$ is the angle tracing out from the x-axis and ranging  $[0, 2\pi)$.
We assume $(x_c, y_c)$ to be fixed within the interior of the object of interest.
To ease the computation, we discretize the contour as a sequence of $L$ points $\{c_i\}_{i=1}^L$.
The discretization is chosen such that points have equal angular spacing,
\begin{equation} \label{eq:polar_point}
c_i = 
\begin{bmatrix}
x_c + \rho_i \cos (i \Delta\theta) \\
y_c + \rho_i \sin (i \Delta\theta)
\end{bmatrix}
\end{equation}
where $\Delta \theta = \frac{2\pi}{L}$ and $\rho_i = \rho(i\Delta \theta)$. 
The above ray based parameterization is called \textit{active rays}.
Importantly, if the region enclosed by the contour forms a convex set,  we can guarantee that for any interior reference point, given any angle $\theta$, there is only one corresponding radius $\rho$ based on the following proposition.
\begin{prop}
Given a closed convex set $X$, a ray starting from any interior point of $X$ will intersect with the boundary of $X$ once.
\end{prop}
We leave the proof to the supplementary material.
If the region is non-convex, a ray may possibly have more than one intersecting point with the boundary.
In that case, we  pick  the one with the minimum distance from the reference point, thus eliminating the possibility that there are multiple $\rho$ corresponding to the same angle $\theta$.
Therefore, compared to snakes, an active rays parameterization  avoids  self-intersections  as shown in Figure~\ref{fig:multiple_crossing}.
Moreover, since we fix the angle at which rays can emanate, the contour points possess an inherent order. 
Such ordering does not exist for snakes; for example, any cyclic permutation of the snake points produces an identical contour. As we see in Section~\ref{ssec:learning}, this allows us to naturally use a loss function that encourages our contours to match building boundaries.

\paragraph{Multiple Sets of Active Rays:}
Note that active rays largely preclude contours that enclose non-convex regions. 
While this is not the dominating case in our application domain, we would like to create a solution that can handle non-convex shapes. 
Towards this goal, we propose to use multiple sets of active rays, where each set has its own fixed reference point.
First, we exploit an instance segmentation to generate a segment over the region of interest (RoI).
We use the method in~\cite{bai2017deep} as it tends to under segment the RoIs, thus largely guaranteeing our reference point to lie in the interior.
Using this segment, we calculate a distance transform on the mask, and select the location of the largest value as the reference point of our initial contour. 
If this evolved contour cannot cover the whole segment,  we then repeat the process using the distance transform of the uncovered regions, until we cover the whole segment. This is illustrated in Figure~\ref{fig:multi_init}.
The union of evolved contours is used to construct the final prediction.

\begin{figure}[t]
	\begin{center}
		\subfloat[]{\includegraphics[height=0.4\linewidth]{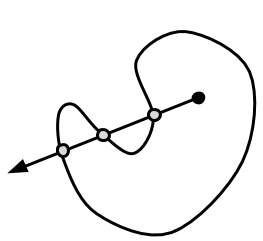}} \hspace{0.5cm}
		\subfloat[]{\includegraphics[height=0.4\linewidth]{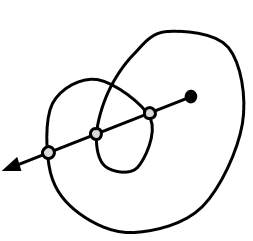}} \\
	\end{center}
	\vspace{-0.1cm}
	\caption{Multiple boundary intersections can occur for (a)~non-convex region, (b)~self-intersection.}
	\label{fig:multiple_crossing}
\end{figure}

\subsection{Contour Energy}

We now introduce the formulation of the contour energy functional, of which the minimization yields the final contour.
In particular, we encourage the contour to follow boundaries, prefer low-curvature solutions, and expand outwards from a small initialization. 
With the aforementioned discretization, the overall energy is defined as follows,
\begin{equation} \label{eq:e_total}
E(c) = \sum_{i=1}^{L} \left[ E_{\text{data}}(c_i) + E_{\text{curve}}(c_i) + E_{\text{balloon}}(c_i) \right].
\end{equation}
Unlike traditional active contour models, we parameterize the energy with the help of a backbone CNN.
The hope is that the great power of representation learning of CNN can make the contour evolution more accurate and efficient.
We use a CNN architecture based on Dilated Residual Networks~\cite{yu2017dilated}. 
Specifically, we use DRN-D-22, with weights pretrained on ImageNet~\cite{russakovsky2015imagenet}, and append additional learned upsampling layers using transposed convolutions. 
At the last layer, we predict three outputs that match the input image size, corresponding to three maps, which we denote as $(D, \beta, \kappa)$.
We now describe each energy term in detail. 

\begin{figure}[t]
	\begin{center}
		\subfloat[]{\includegraphics[width=0.4\linewidth]{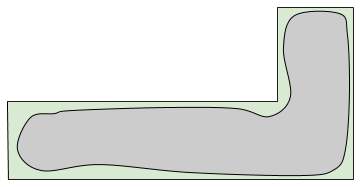}} \hspace{0.5cm}
		\subfloat[]{\includegraphics[width=0.4\linewidth]{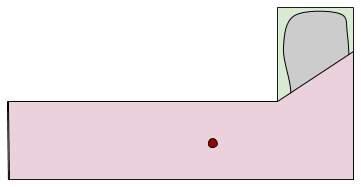}} \\
		\subfloat[]{\includegraphics[width=0.4\linewidth]{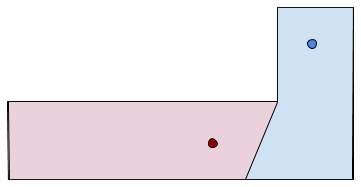}} \hspace{0.5cm}
		\subfloat[]{\includegraphics[width=0.4\linewidth]{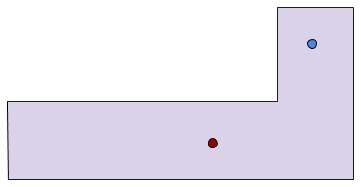}} \\
	\end{center}
	\vspace{-0.1cm}
	\caption{Multiple initialization scheme. (a) Instance segmentation from \cite{bai2017deep} (gray), and ground truth (green); (b) First initialization and segmentation (red); (c) Initialization from remaining segment (blue); (d) Final output.}
	\label{fig:multi_init}
\end{figure}

\paragraph{Data Term:}
Given an input image, the data term is defined as
\begin{equation} \label{eq:e_data}
E_{\text{data}}(c_i) = D(c_i),
\end{equation} 
where $D$ is a non-negative feature map output by the backbone CNN.
Note that $D$ and all subsequently described feature maps are of the same shape as the input image. 
Since we are minimizing the energy, the contour should seek out places where $D$ is low. 
Therefore, the CNN should ideally predict low values at the boundaries.

\paragraph{Curvature Term:}
Intuitively, this term models the resistance of the contour towards bending as follows
\begin{equation} \label{eq:e_curvature}
E_{\text{curve}}(c_i) = \beta(c_i) \vert c_{i+1} - 2c_i + c_{i-1} \vert^2,
\end{equation} 
where $\beta$ is a non-negative feature map output by the backbone CNN and the squared term is a discrete approximation of the second order derivative of the contour.
This term is flexible since the weighting scheme induced by $\beta$ can make the energy locally adaptive.
This energy term will force the contour to straighten out wherever the value is high.
Our curvature term is simpler compared to the one in DSAC~\cite{marcos2018learning}, where in order to prevent snake points from clustering too close together in low-energy areas, they employ an additional membrane term based on the first order derivative of the contour. 
In contrast, we do not need such a term as our evenly spaced angles guarantee that contour points will not group too closely.

\paragraph{Balloon Term:}
Our balloon term is defined by
\begin{equation} \label{eq:e_balloon}
E_{\text{balloon}}(c_i) = \kappa(c_i) ( 1 - \rho_i / \rho_{\text{max}} )
\end{equation} 
where $\kappa$ is a non-negative feature map output by the backbone CNN and $\rho_{\text{max}}$ is the maximum radius a contour can reach without crossing the image boundary.
The balloon term is designed to propel a given contour point outwards from its reference point, conditioned on the value of the underlying $\kappa$ map. 
It is necessary due to two reasons. 
First, the data term may be insufficient to guide the contour towards the boundaries, especially if the contour was initialized far away from them. 
Second, as noted in \cite{denzler1999active}, the curvature term has an auxiliary effect of shrinking the contour, which can lead to its collapse. 

It is interesting to note that DSAC~\cite{marcos2018learning} also employs a balloon term which can be expressed using our notation as below, 
\begin{equation}
E_{\text{balloon}}^{\textrm{DSAC}}(v) = - \sum_{(x, y) \in \Omega(v)} \kappa
\end{equation} 
where $\Omega(v)$ denotes the area enclosed by the contour $v$. 
This term pushes the contour to encapsulate as much of the $\kappa$ map as possible. 
In our case, due to the active ray parameterization, a contour point can only move along one axis, either towards the reference point or away.
Therefore, our balloon term is much simpler as it avoids the need to perform an area integral.  
Also, as we will see in the following section, our balloon term fits in seamlessly with our inference procedure.

\subsection{Contour Evolution}
Conditioned on the energy terms predicted by the backbone CNN, the second inference step is achieved through energy minimization. 
To evolve the initial contour towards the optimal one, we first derive the partial derivatives and set them to zero.
We then resort to an iterative algorithm to solve the system of partial derivatives.

Specifically, the partial derivatives of the energy terms w.r.t. the contour are derived as below.
For the data term,
\begin{align}\label{eq:data_deriv}
\pder[E_{\text{data}}(c)]{\rho_i} = \pder[D(c_i)]{x} \cos(i\Delta\theta) + \pder[D(c_i)]{y} \sin(i\Delta\theta) 
\end{align}
where we change the coordinates back to Cartesian to facilitate the computation of derivatives, \eg with a Sobel filter.

For the curvature term, substituting the Cartesian expression of contour points (Equation~\ref{eq:polar_point}) into the expressions for the energy (Equation~\ref{eq:e_curvature} and Equation~\ref{eq:e_total}), we have,
\begin{align}
\begin{split}
\pder[E_{\text{curve}}(c)]{\rho_i} ={}& \pder[]{\rho_i}\sum_{i=0}^{L-1} \beta(c_i) [\rho_{i+1}^2 + 4\rho_i^2 + \rho_{i-1}^2 - \\ & 4\rho_i(\rho_{i+1} + \rho_{i-1}) \cos(\Delta\theta) + \\
& 2\rho_{i+1}\rho_{i-1}\cos(2\Delta\theta)]
\end{split} \nonumber \\
\begin{split} \label{eq:beta_deriv}
\approx{}& [2\beta(c_{i+1})\cos(2\Delta\theta)]\rho_{i+2} + \\ 
& [-4(\beta(c_{i+1})+\beta(c_{i-1})\cos(\Delta\theta)]\rho_{i+1} + \\
& 2(\beta(c_{i+1}) + 4\beta(c_i) + \beta(c_{i-1}))]\rho_i + \\ 
& [-4(\beta(c_i) + \beta(c_{i-1}))\cos(\Delta\theta)]\rho_{i-1} + \\
& [2\beta(c_{i-1})\cos(2\Delta\theta)]\rho_{i-2}
\end{split}
\end{align}
where we discard the term arising from the product rule of differentiation as in~\cite{marcos2018learning}. 
We interpret this approximation as treating the $\beta$ map as not varying within the small vicinity of the contour points. 
Alternatively, we do not wish for the gradient of the $\beta$ map to exert pressure on the contour. 
Empirically, we found that doing so stabilizes learning, as the network only needs to adjust values of the $\beta$ map without attention to its gradients.

For the balloon term, we use the same approach as the curvature term to obtain the partial derivative,
\begin{equation} \label{eq:kappa_deriv}
\pder[E_{\text{balloon}}(c)]{\rho_i} \approx - \frac{\kappa(c_i)}{\rho_{\text{max}}}
\end{equation}
With the above derivation, we have a collection of $L$ partial differential equations w.r.t. individual contour points. 
We can summarize this system of equations in a compact matrix form,
\begin{equation}
\pder[E]{\mathbf{\rho}} = A\mathbf{\rho} + \mathbf{f}
\end{equation}
where $\rho$ is a column vector of size $L$, $A$ is an $L \times L$ cyclic pentadiagonal matrix comprised of $E_{\text{curve}}$ derivatives, and $\mathbf{f}$ is a column vector comprised of $E_{\text{data}}$ and $E_{\text{balloon}}$ derivatives.

This system of partial differential equations can be solved with an iterative method. 
The approach taken by \cite{kass1988snakes} and \cite{marcos2018learning} is an implicit-explicit method. For the purposes of our implementation, we adopt an explicit method instead, as it avoids the matrix inverse operation.
Specifically, the contour evolves according to
\begin{equation} \label{eq:evolution}
\rho^{(t+1)} = \rho^{(t)} - \Delta t \, (A\rho^{(t)} + f)
\end{equation}
where $\Delta t$ is a time step hyperparameter. 
In practice, we found setting $\Delta t$ as $2e^{-4}$ is stable enough for solving the system.

\subsection{Learning} \label{ssec:learning}

Since there exists an explicit ordering of the contour points, we can naturally generate a ground truth active ray and use it to supervise the prediction. 
Using the same reference point $(x_c, y_c)$ and angle discretization $\Delta \theta$, we cast rays outwards and record the distances at which they intersect the ground truth polygon. 
In the case of multiple intersections, we take the smallest distance, to prioritize hitting the correct boundaries over increasing coverage. 
This is illustrated in Figure~\ref{fig:generate_gt}. 
We use this collection of ground truth distances, $\{\rho_i^{GT}\}_{i=1}^L$, to compute an $L_1$ loss:
\begin{equation}
\ell(\rho, \rho^{GT}) = \sum_{i=1}^L \vert \rho_i^{GT} - \rho_i \vert.
\end{equation}

It differs from the loss employed by DSAC, which used a SSVM hinge loss to optimize for intersection-over-union. 
Instead, our loss encourages contour points to target building boundaries which is simpler and more efficient. 

To allow for gradients to backpropagate to the $D$, $\beta$, and $\kappa$ maps, we interpret the value of a given contour point as a floating point number, and compute the value of a map at that point (\eg $D(c_i)$) using  bilinear interpolation from its four adjacent map entries, in a manner similar to what is used in Spatial Transformer Networks~\cite{jaderberg2015spatial}.
We summarize the learning process in Algorithm~\ref{alg:dar}.

\begin{figure}[t]
	\begin{center}
		\includegraphics[width=0.4\linewidth]{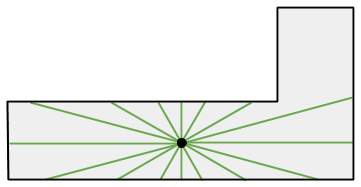}
	\end{center}
	\vspace{-0.1cm}
	\caption{Ground truth rays (green) are defined as the closest distance to the ground truth polygon.}	
	\label{fig:generate_gt}
\end{figure}

\begin{algorithm}[t]
	\SetAlgoLined
	\SetKwInOut{Input}{Given}
	\Input{Input image $I$, CNN $F$ with parameters $\omega$, ground truth rays $\rho^{GT}$, initial rays $\rho^{(0)}$, time step hyperparameter $\Delta t$, number of training steps $N$}
	\For{$j = 1$ \KwTo $N$}{
		$(D, \beta, \kappa) = F(I_j; \theta)$ \\
		\While{not converged}{
			$A, f = $ collection of $\pder[E]{\rho_j^{(t)}}$ using $(D, \beta, \kappa)$ \\
			$\rho_j^{(t+1)} = \rho_j^{(t)} - \Delta t (A \rho_j^{(t)} + f)$
		}
		$L = \ell(\rho_j^{(T)}, \rho_j^{GT})$ \\
		Compute $\pder[L]{\omega}$ using backpropagation \\
		Update $\omega$ with gradient-based optimization
	}
	\caption{{\ourmodel} training algorithm.}
	\label{alg:dar}
\end{algorithm}

\section{Experiments}
\label{sec:results}

\subsection{Experimental Setup}

\paragraph{Datasets:} 
We evaluate {\ourmodel} on several building instance segmentation datasets: Vaihingen~\cite{vaihingen}, Bing Huts~\cite{marcos2018learning}, and TorontoCity~\cite{wang2017torontocity}. 
The Vaihingen dataset consists primarily of detached buildings in a town in southern Germany. 
The original images are $512 \times 512$ at a resolution of $9$~cm/pixel. 
There are $168$ buildings in total, which are divided into $100/68$ examples for train/test according to the same split used in \cite{marcos2018learning}. 
The Bing Huts dataset consists of huts located in a rural area of Tanzania, with an original size of $80 \times 80$ at a resolution of $30$~cm/pixel. 
There are $605$ images in total, divided into $335/270$ examples for train/test, again using the same splits. For these two datasets, we further partition the existing training data in an $80\%/20\%$ split to form a validation set, and use this for model selection. We do this for 5 disjoint validation sets, while the test set remains the same.
The TorontoCity dataset contains aerial images captured over a dense urban area in Toronto. 
The images used have a resolution of $10$~cm/pixel.
The dataset consists of approximately $28,000/12,000$ images for train/test which covers a diverse mixture of buildings, including residential, commercial, and industrial. 
We divide the training set of TorontoCity into a $80\%/20\%$ split for training and validation respectively. 

\paragraph{Metrics:} 
We measure performance using intersection-over-union (IoU) averaged over number of instances, weighted coverage, and polygon similarity as in \cite{wang2017torontocity}. Additionally, we evaluate the boundary F-score (BoundF) introduced in \cite{perazzi2016davis}, averaged over thresholds from $1$ to $5$ pixels, inclusive. 
For Vaihingen and Bing Huts, we aggregate these metrics over all models selected with the various validation sets, measuring their mean and standard deviation.
Lastly, for TorontoCity, we also evaluate the quality of alignment for the predicted boundaries. 
Specifically, we gather predicted contour pixels that match with the ground truth boundaries, within a threshold of 5 pixels. 
For these matches, we evaluate the alignment error with respect to the ground truth, which is determined as the cosine similarity between the ground truth boundary and the predicted boundary. We then rank these pixels by their alignment error, and plot the recall at various thresholds of this error. 

\paragraph{Hyper-parameters:} 
We discretize our contour with $L~=~60$ points. 
For training, we use SGD with momentum, with learning rate $4\times 10^{-5}$, momentum $0.3$, and a batch size of $10$. 
The learning rate decay schedule is explained in supplementary material.
We perform $200$-step inference as it is found to be sufficient for convergence in practice. 
We initialize the contour randomly within the ground truth boundary during training.
For testing, we initialize in image centers for Vaihingen and Bing Huts due to the fact that most buildings of these two datasets are of regular shape.
As for TorontoCity, we leverage the deep watershed transform (DWT)~\cite{bai2017deep} to get instance proposals and initialize as described in Section~\ref{ssec:contour_repr}.
Standard data augmentation techniques (random rotation, flipping, scaling, color jitter) are used. 
We do not adopt common stabilization tricks during inference as they are non-differentiable.
Instead, we found using the Euclidean distance transform to pre-train our $D$, $\beta$ and $\kappa$ maps helps stabilize the inference during training.
We leave more details of pre-training in the supplementary material.

\subsection{Results}

\begin{table*}
	\centering
	\begin{minipage}{0.99\linewidth}
		\resizebox{1\linewidth}{!}{
			\begin{tabular}{l|l|c|c|c|c|c|c|c|c|c|c|c|c|c|c|c|c}
				\hline
				\multicolumn{2}{c|}{Method} & \multicolumn{8}{|c|}{Vaihingen} & \multicolumn{8}{|c}{Bing Huts} \\
				\hline
				\multirow{2}{*}{Approach} & \multirow{2}{*}{Backbone} & \multicolumn{2}{|c|}{mIoU} & \multicolumn{2}{|c|}{WCov} & \multicolumn{2}{|c|}{PolySim} & \multicolumn{2}{|c|}{BoundF} & \multicolumn{2}{|c|}{mIoU} & \multicolumn{2}{|c|}{WCov} & \multicolumn{2}{|c|}{PolySim} & \multicolumn{2}{|c}{BoundF} \\
				\cline{3-18}
				{} & {} & $\mu$ & $\sigma$ & $\mu$ & $\sigma$ & $\mu$ & $\sigma$ & $\mu$ & $\sigma$ & $\mu$ & $\sigma$ & $\mu$ & $\sigma$ & $\mu$ & $\sigma$ & $\mu$ & $\sigma$ \\
				\hline\hline
				FCN & DSAC's \cite{marcos2018learning} & 81.09 & 0.32 & 81.48 & 0.52 & 68.78 & 0.66 & 64.60 & 0.77 & 69.88 & 0.65 & 73.36 & 0.40 & 64.01 & 0.36 & 30.39 & 0.95 \\
				FCN & Ours & 87.27 & 0.50 & 86.89 & 0.40 & 76.41 & 0.77 & \textbf{76.84} & 0.48 & 74.54 & 1.22 & \textbf{77.55} & 0.96 & 67.98 & 2.20 & 37.77 & 2.69 \\
				DSAC \cite{marcos2018learning} & DSAC's \cite{marcos2018learning} & 71.10 & 3.88 & 70.76 & 4.07 & 61.71 & 3.97 & 36.44 & 5.16 & 38.74 & 1.07 & 44.61 & 3.00 & 38.91 & 2.66 & 37.16 & 0.85 \\
				DSAC \cite{marcos2018learning} & Ours & 60.37 & 5.97 & 61.12 & 6.22 & 52.96 & 5.89 & 24.34 & 5.73 & 57.23 & 3.89 & 63.09 & 2.25 & 55.43 & 2.20 & 15.98 & 2.62 \\
				Ours & Ours & \textbf{88.24} & 0.38 & \textbf{88.16} & 0.35 & \textbf{81.33} & 0.37 & 75.91 & 0.89 & \textbf{75.29} & 0.45 & 77.07 & 0.63 & \textbf{72.12} & 0.57 & \textbf{38.08} & 0.76 \\
				\hline 
			\end{tabular}
		}
		\vspace{0.1cm}
		\caption{Test set results (mean and standard deviation) on Vaihingen and Bing Huts from models selected across 5 disjoint validation sets. mIoU is averaged over instances. WCov is weighted coverage. PolySim is polygon similarity. BoundF is average boundary-F score at 1 pixel to 5 pixel thresholds.}
		\label{tab:results_vaihingen_bing}
	\end{minipage}
	\vspace{-0.2cm}
\end{table*}

\begin{table}
	\begin{center}
	\resizebox{\linewidth}{!}{
		\begin{tabular}{l|c|c|c|c}
			\hline
			Method & mIoU & WCov & PolySim & BoundF \\
			\hline\hline
			FCN-8s \cite{long2015fully} & {-} & 45.6 & \textbf{32.3} & {-}  \\
			ResNet50 \cite{he2016deep} & {-} & 40.1 & 29.2 & {-} \\
			DWT \cite{bai2017deep} & {-} & 52.0 & 24.0 & {-} \\
			DSAC \cite{marcos2018learning} & 60.0 & \textbf{58.0} & 27.2 & 25.4 \\
			Ours, single init & 57.9 & 52.2 & 23.9 & 29.0 \\
			Ours, multi init & \textbf{60.1} & 57.5 & 26.8 & \textbf{29.6} \\
			\hline
		\end{tabular}
		}
	\end{center}
	\caption{Results on TorontoCity.}
	\label{tab:results_tcity}
	\vspace{-0.3cm}
\end{table}

\begin{figure}[t]
	\begin{center}
		\includegraphics[width=0.99\linewidth]{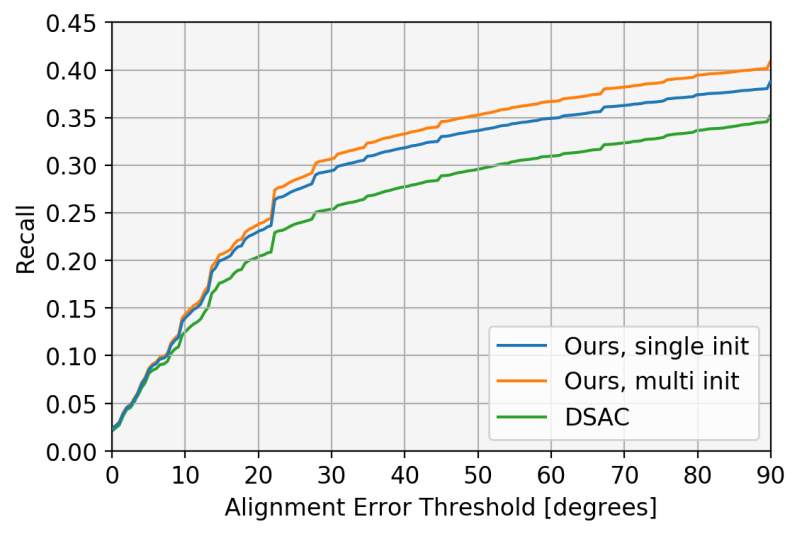}
	\end{center}
	\vspace{-0.3cm}
	\caption{Recall-alignment curve. The recall for all matching boundary pixels (within a 5 pixel threshold of ground truth) is plotted at a given alignment error. Our method exhibits increased recall overall.}
	\label{fig:alignment_curve}
	\vspace{-0.3cm}
\end{figure}

\textbf{Vaihingen and Bing Huts:} 
We compare against the baseline methods and DSAC~\cite{marcos2018learning}. 
In particular, the baseline exploits a fully convolutional network (FCN)~\cite{long2015fully} as the backbone to perform semantic segmentation of buildings (interior, boundary, and exterior) and then the post-processing following~\cite{marcos2018learning} to obtain the predicted contour.
For fair comparison, we also replace the backbone of DSAC and the baseline with ours.
We summarize results in Table~\ref{tab:results_vaihingen_bing}. Compared to the strong FCN baselines, our method exhibits improved performance across the majority of metrics. In particular, the significant improvement on PolySim suggest our segmentations are more geometrically similar. Furthermore, our method significantly outperforms DSAC on all metrics. Even in instances where DSAC exhibits good coverage-based metrics, our method is significantly better at capturing edges. It is interesting to note that substituting our backbone in DSAC does not increase performance, while DSAC's results generally exhibits higher variance, regardless of backbone.

\textbf{TorontoCity:} 
We compare against the semantic segmentation based methods that utilize FCN-8s~\cite{long2015fully} or ResNet50~\cite{he2016deep}, along with the instance segmentation method that relies on DWT. 
Additionally, because the commercial and industrial buildings contained in TorontoCity tend to possess complex boundaries, we examine the effects of using one versus multiple initializations described in Section~\ref{ssec:contour_repr}. 
We summarize results in Table~\ref{tab:results_tcity}. 
Our method shows improved performance compared to DSAC on the mIOU. 
For weighted coverage, which accounts for the area of the building, we achieve comparable performance to DSAC. 
Our performance on this metric is influenced primarily by larger buildings, as moving from a single initialization to multiple initializations significantly improves the performance. 
We find that large areas presented by commercial and industrial lots present many ambiguous features such as mechanical penthouses and visible facades due to parallax error. 
Our single initialization method tends to produce smaller segmentations as it focuses on the boundaries offered by these features. 
This is alleviated by using multiple initializations. 
The weighting from larger buildings is also reflected in polygon similarity. 
The average BoundF metric demonstrates our method is significantly better at capturing building boundaries than DSAC. 
It is important to note that even our segmentations generated with a single initialization showed improved performance. 
To delve deeper into the quality of these segmentations, we examine their alignment with respect to the ground truth in Figure~\ref{fig:alignment_curve}. 
We see that, for a given threshold for alignment error, our method exhibits superior recall. 
Overall, our multiple initialization scheme performs the best, although for lower error thresholds our single initialization is also very competitive.

\begin{figure}[t]
	\vspace{-0.35cm}
	\begin{center}
		\subfloat[]{\includegraphics[width=0.245\linewidth]{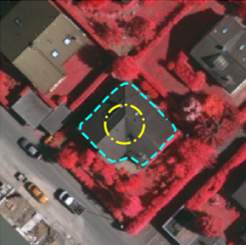}} \hfill
		\subfloat[]{\includegraphics[width=0.245\linewidth]{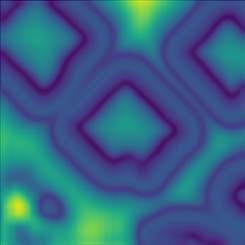}} \hfill
		\subfloat[]{\includegraphics[width=0.245\linewidth]{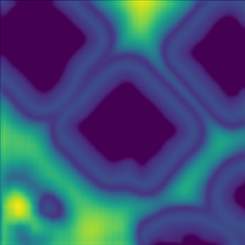}} \hfill
		\subfloat[]{\includegraphics[width=0.245\linewidth]{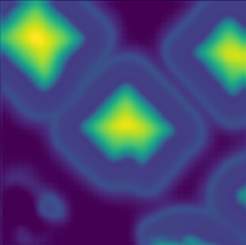}}
	\end{center}
	\caption{Energies predicted by our CNN on an example from the Vaihingen test set. (a) Input image, initial contour (yellow), and final contour (cyan). (b) Data term. (c) Curvature term. (d) Balloon term.}
	\label{fig:energies}
	\vspace{-0.3cm}
\end{figure}

\textbf{Number of Initializations:} 
In TorontoCity, we found that $25.7\%$ of examples required multiple initializations. 
For these examples, on average $3.71$ initializations were required.

\textbf{Qualitative Discussion:} 
We visualize some energies predicted by our CNN in Figure~\ref{fig:energies}. 
We see the CNN opts to predict a $D$ term that has deep valleys at the building contours. 
The $\beta$ term adopts small values along the edges to encourage straightness at the boundaries, while the $\kappa$ term acts to propel the contour points from inside. 
We show additional segmentations in Figure~\ref{fig:qual_results}.
The segmentations produced by our method are generally more adherent to the edges of the buildings.
This is especially helpful when buildings are densely packed together, as seen in the TorontoCity results (columns e-f).
Additionally, in comparison to DSAC, our parameterization successfully prevents self-intersecting contours (column b, second last row). 

\textbf{Failure Modes:} 
The last two rows in Figure~\ref{fig:qual_results} demonstrate some weaknesses of our model. 
In cases where the model is unsure about the extent of one building, it will expand the contour until it meets the edge of another building (column (b), last row). 
Also, on large commercial lots (column f, last two rows), our method becomes confused by the shapes and features of the buildings.

{
\captionsetup[subfigure]{labelformat=empty}
\begin{figure*}
	\centering
	\subfloat{\includegraphics[width=0.16\linewidth]{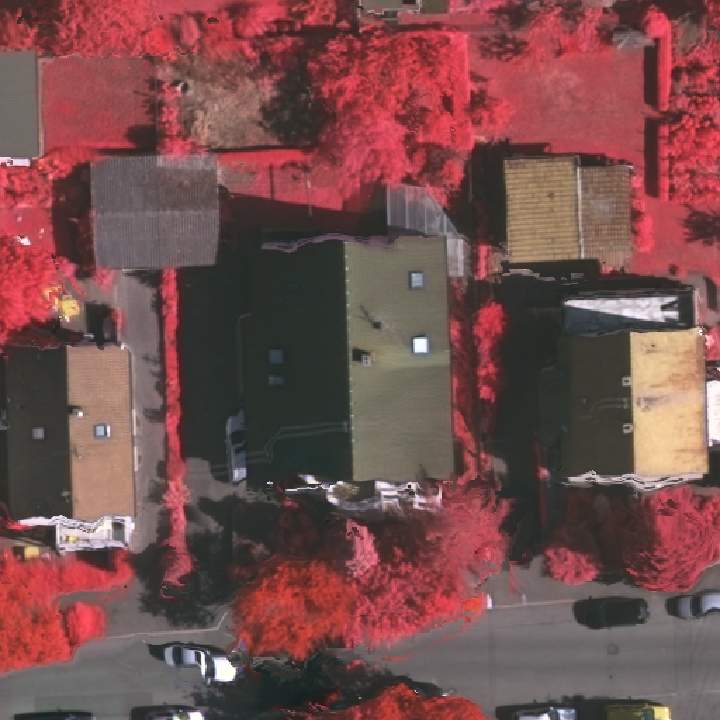}} \hfill
	\subfloat{\includegraphics[width=0.16\linewidth]{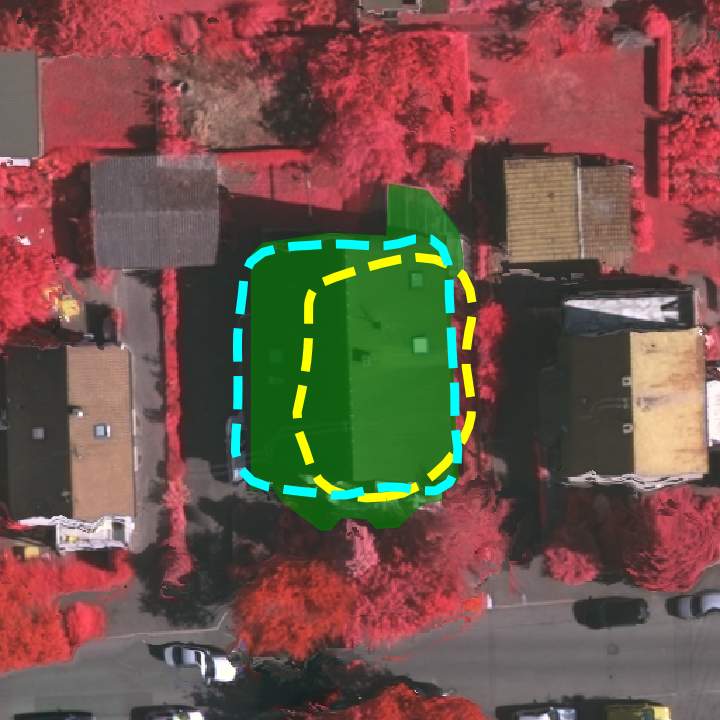}} \hfill
	\subfloat{\includegraphics[width=0.16\linewidth]{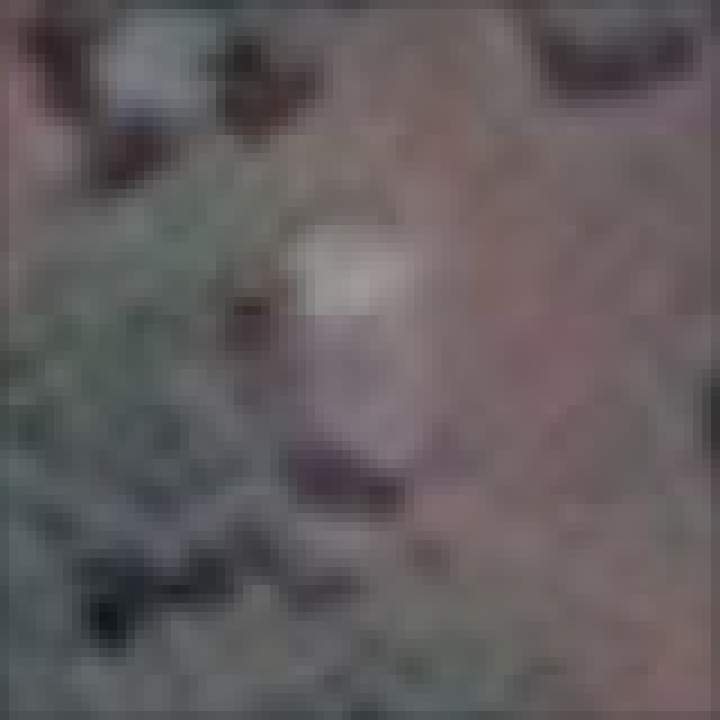}} \hfill
	\subfloat{\includegraphics[width=0.16\linewidth]{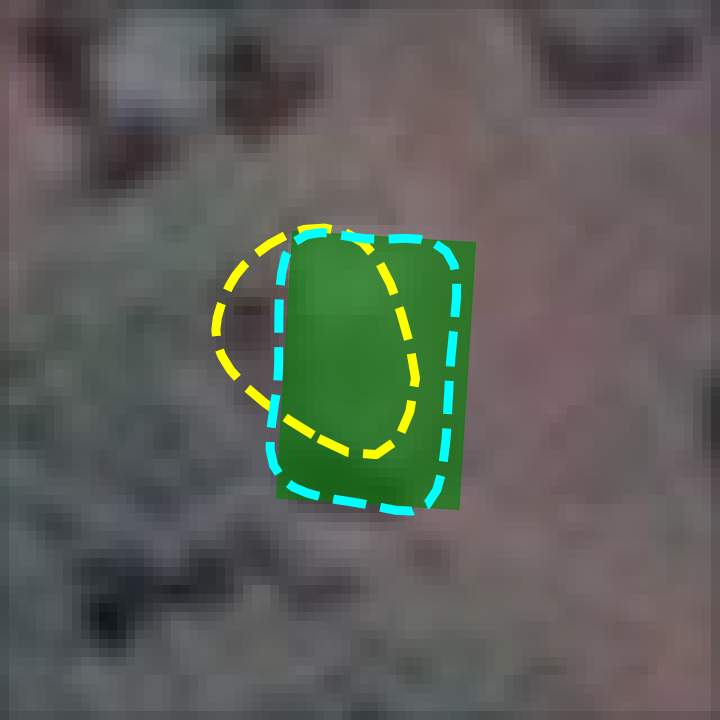} } \hfill
	\subfloat{\includegraphics[width=0.16\linewidth]{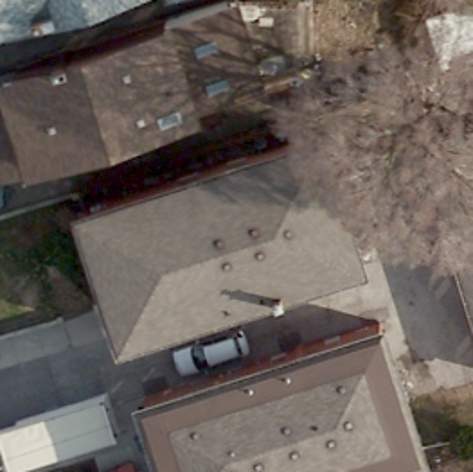}} \hfill
	\subfloat{\includegraphics[width=0.16\linewidth]{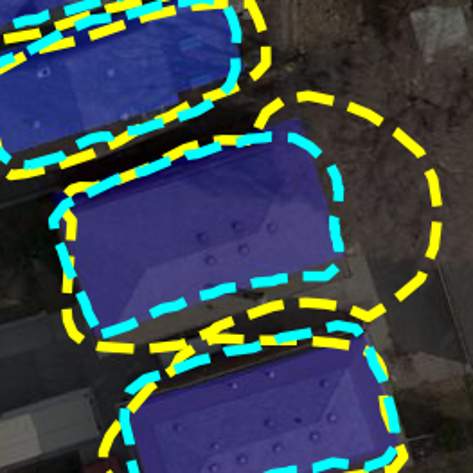} } \\ \vspace{-0.3cm}
	\subfloat{\includegraphics[width=0.16\linewidth]{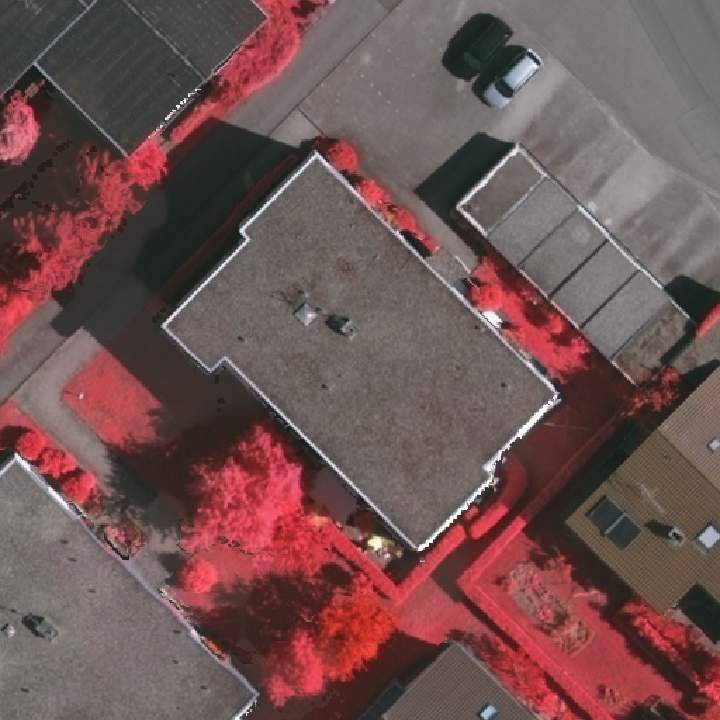}} \hfill
	\subfloat{\includegraphics[width=0.16\linewidth]{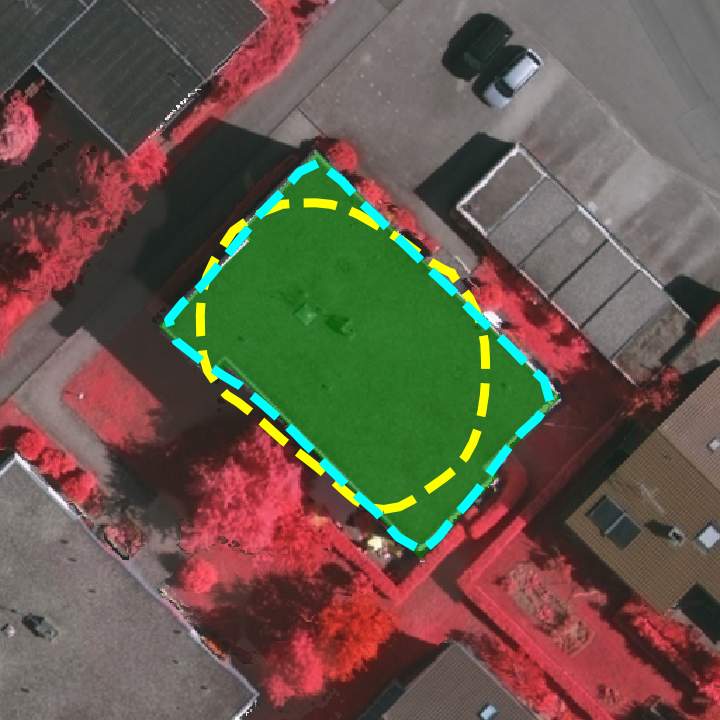}} \hfill
	\subfloat{\includegraphics[width=0.16\linewidth]{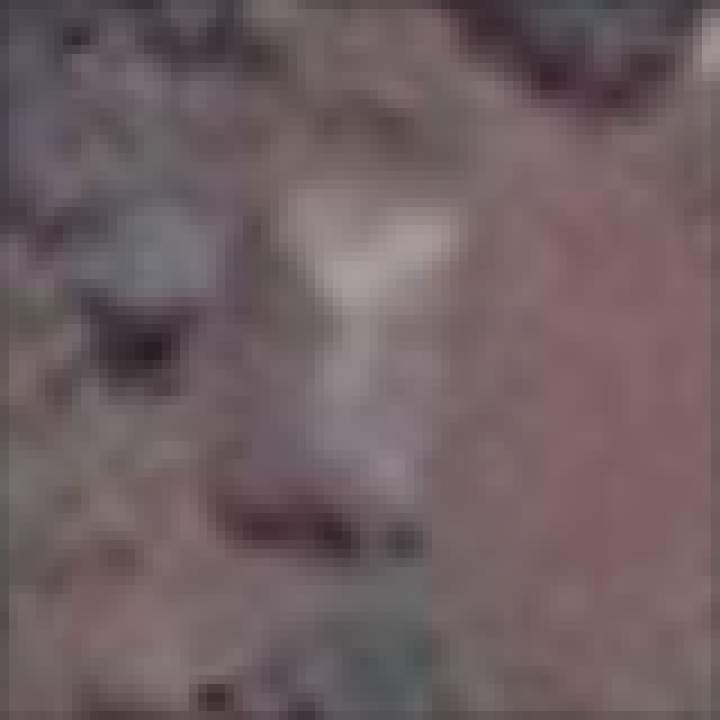}} \hfill
	\subfloat{\includegraphics[width=0.16\linewidth]{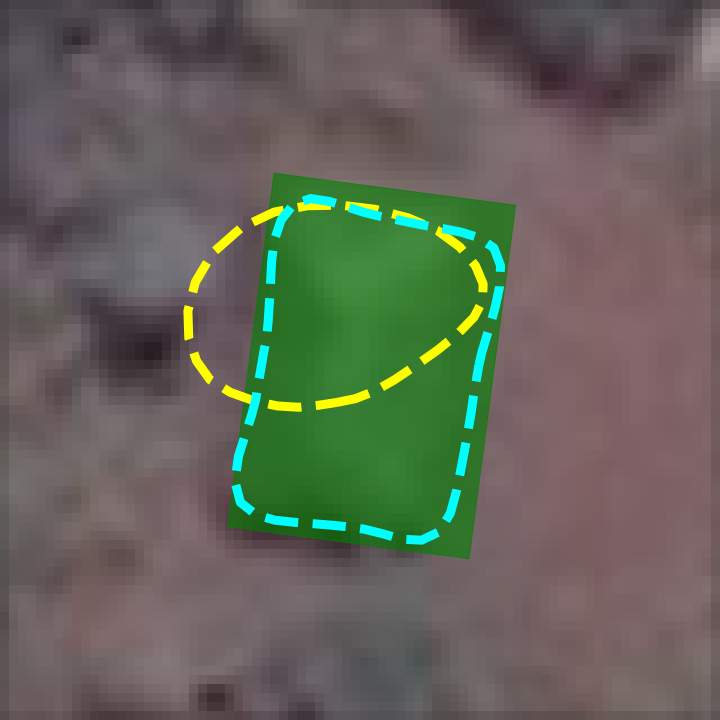} } \hfill
	\subfloat{\includegraphics[width=0.16\linewidth]{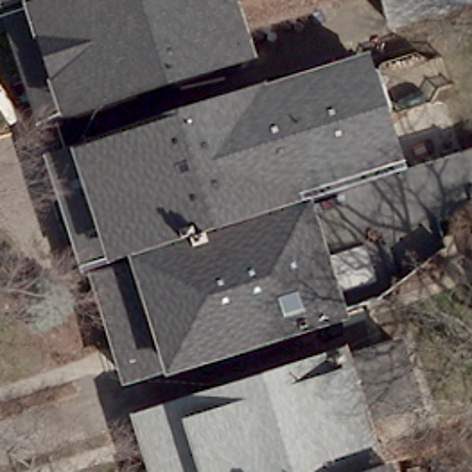}} \hfill
	\subfloat{\includegraphics[width=0.16\linewidth]{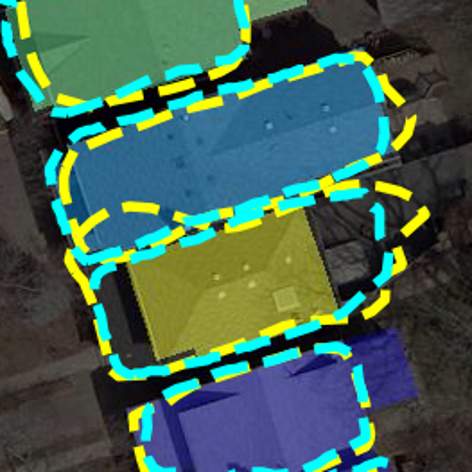} } \\ \vspace{-0.3cm}
	\subfloat{\includegraphics[width=0.16\linewidth]{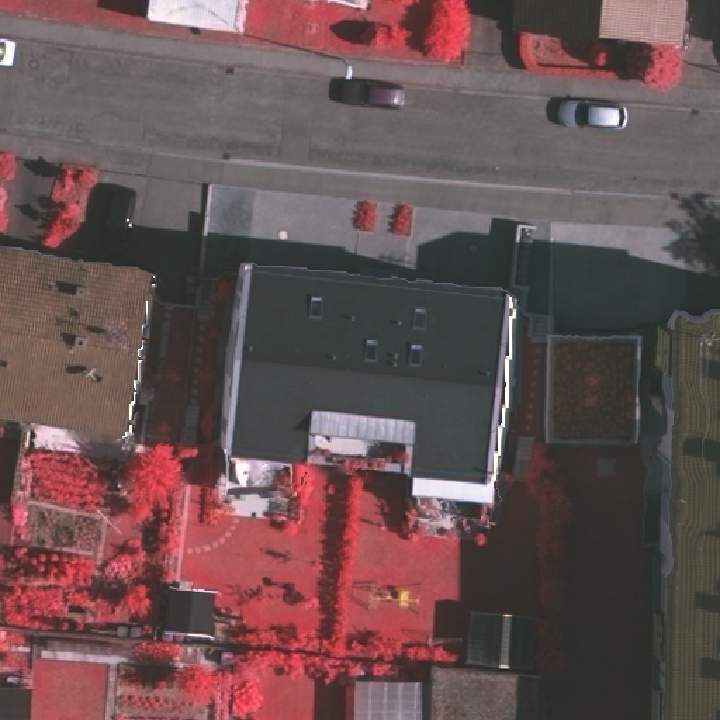}} \hfill
	\subfloat{\includegraphics[width=0.16\linewidth]{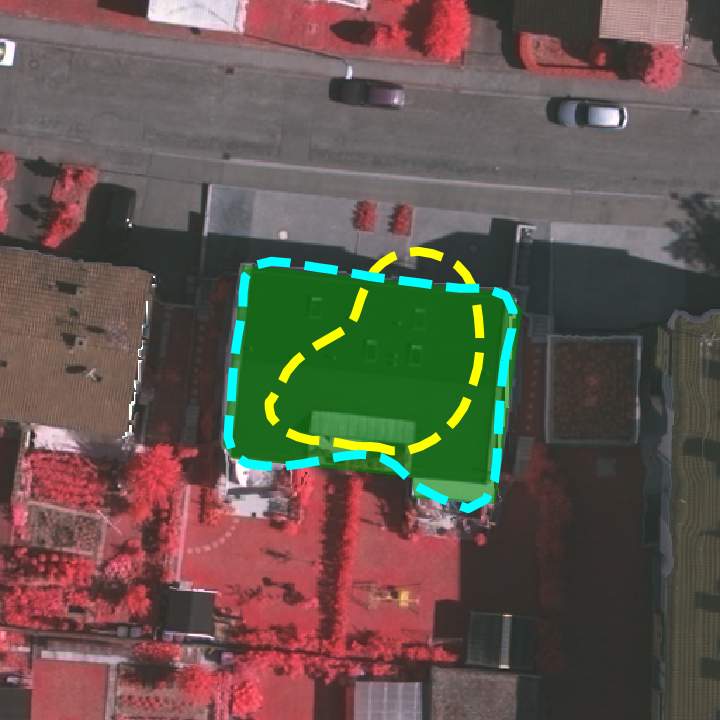}} \hfill
	\subfloat{\includegraphics[width=0.16\linewidth]{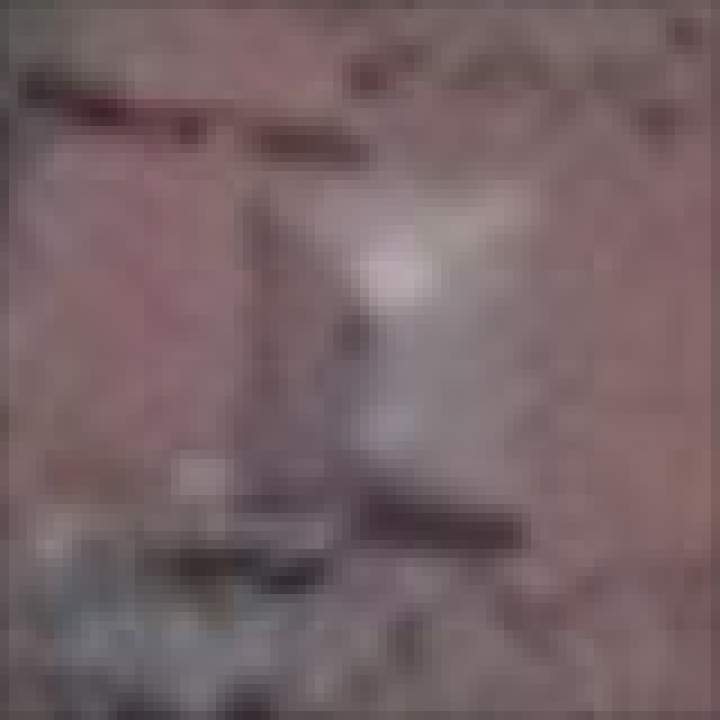}} \hfill
	\subfloat{\includegraphics[width=0.16\linewidth]{figs/qual_results/bing_good/bing_353_compare.jpg} } \hfill
	\subfloat{\includegraphics[width=0.16\linewidth]{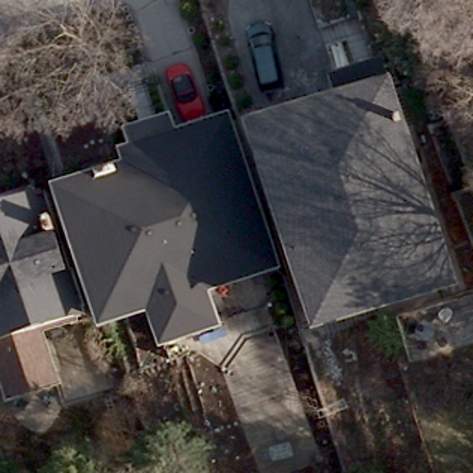}} \hfill
	\subfloat{\includegraphics[width=0.16\linewidth]{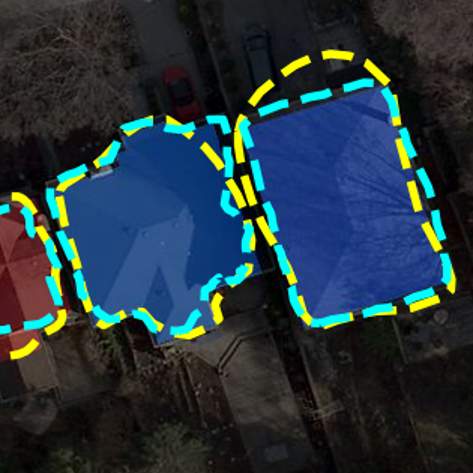} } \\ \vspace{-0.3cm}
	\subfloat{\includegraphics[width=0.16\linewidth]{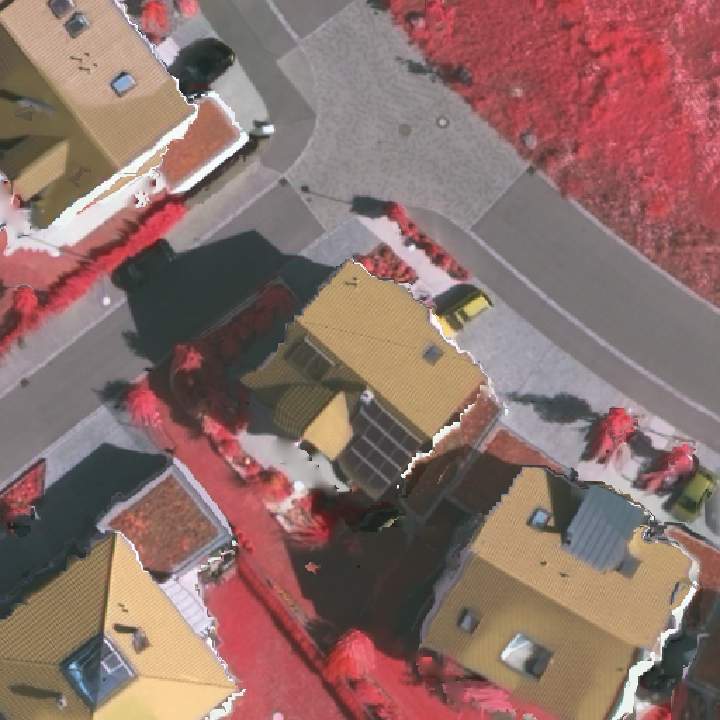}} \hfill
	\subfloat{\includegraphics[width=0.16\linewidth]{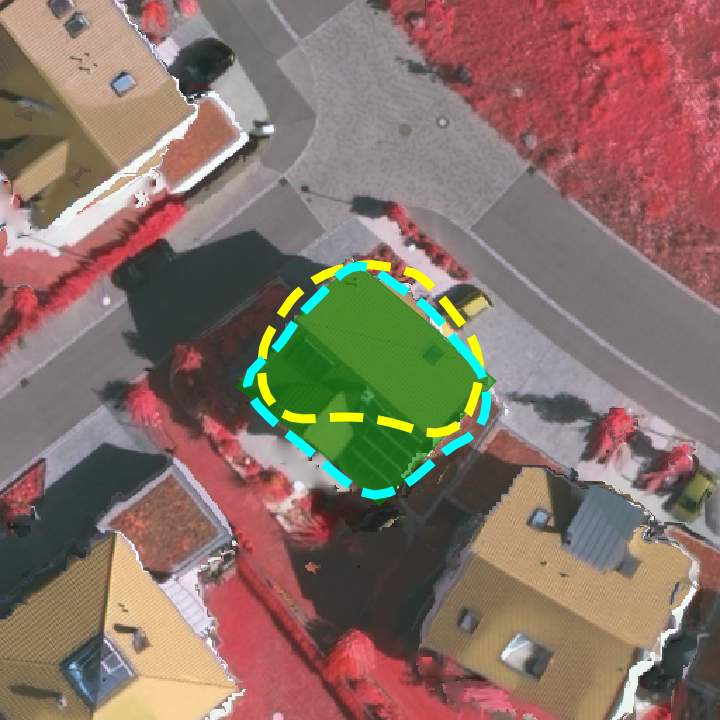}} \hfill
	\subfloat{\includegraphics[width=0.16\linewidth]{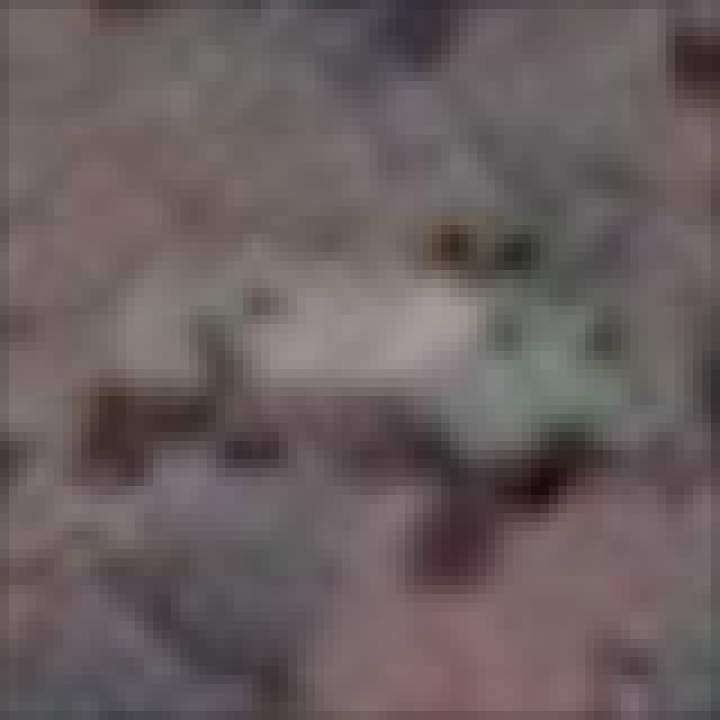}} \hfill
	\subfloat{\includegraphics[width=0.16\linewidth]{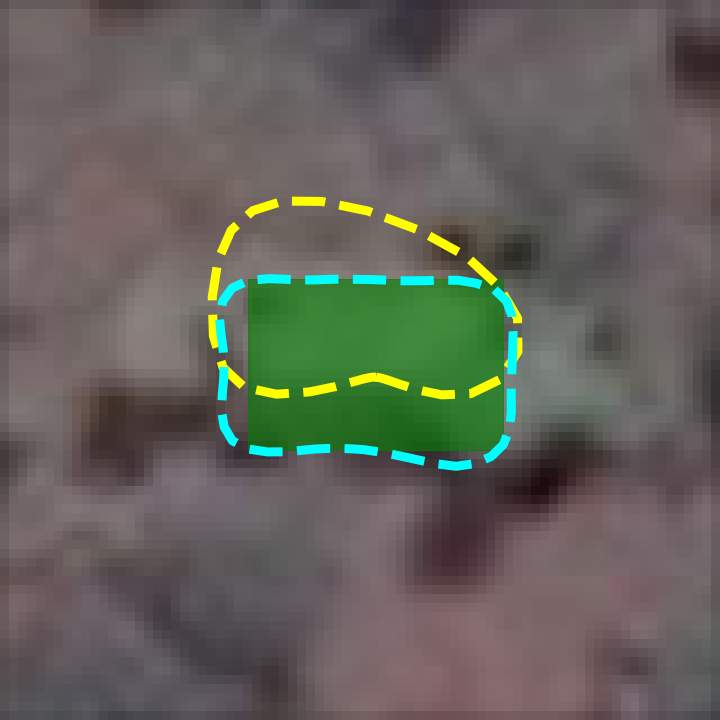} } \hfill
	\subfloat{\includegraphics[width=0.16\linewidth]{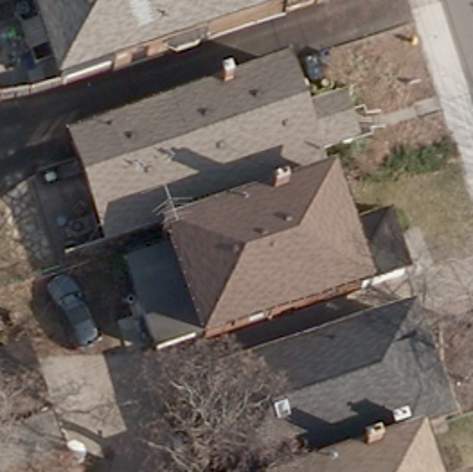}} \hfill
	\subfloat{\includegraphics[width=0.16\linewidth]{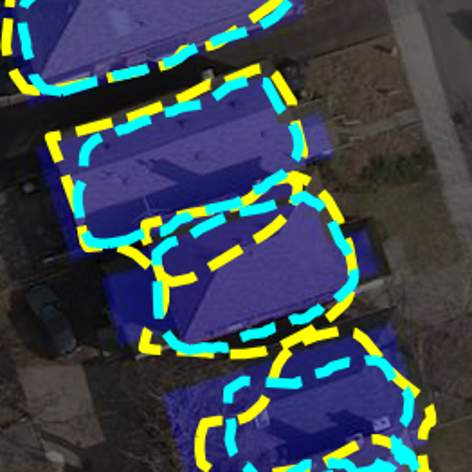} } \\ \vspace{-0.3cm}
	\subfloat{\includegraphics[width=0.16\linewidth]{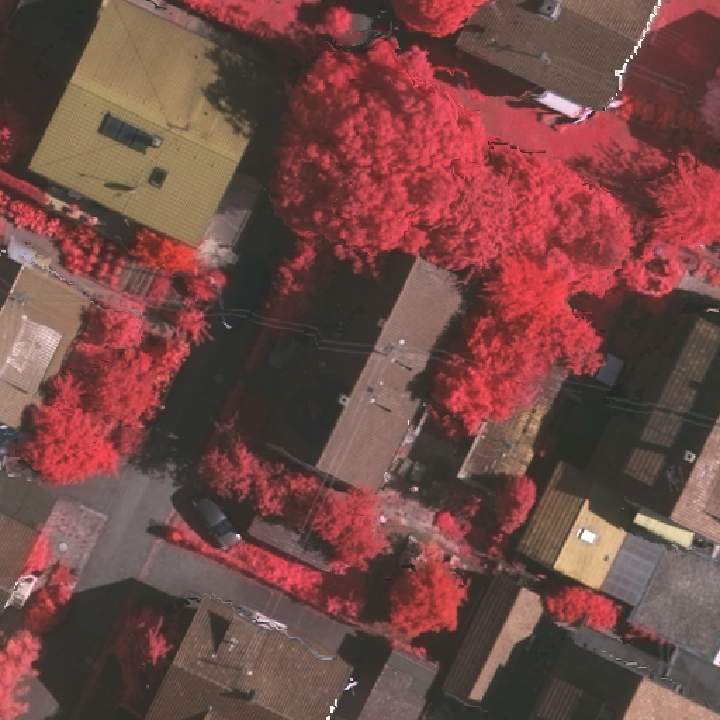}} \hfill
	\subfloat{\includegraphics[width=0.16\linewidth]{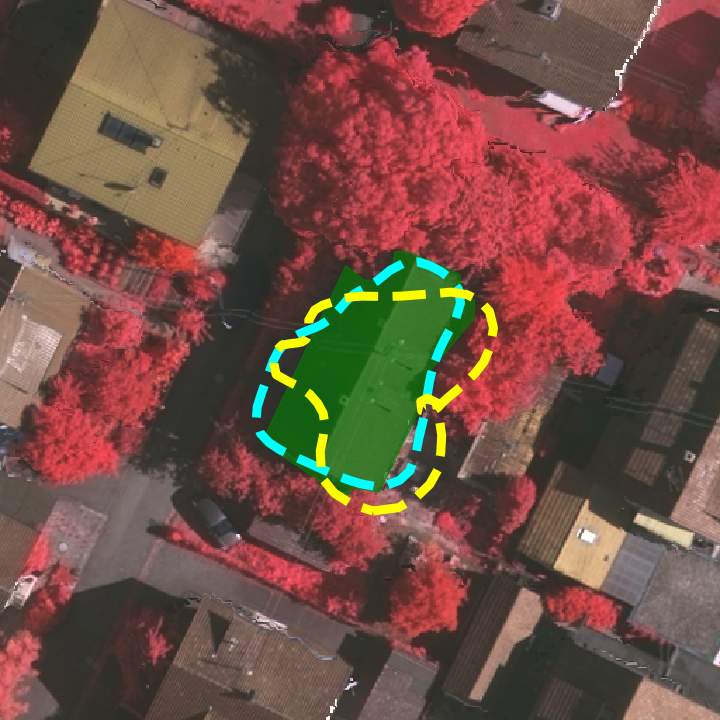}} \hfill
	\subfloat{\includegraphics[width=0.16\linewidth]{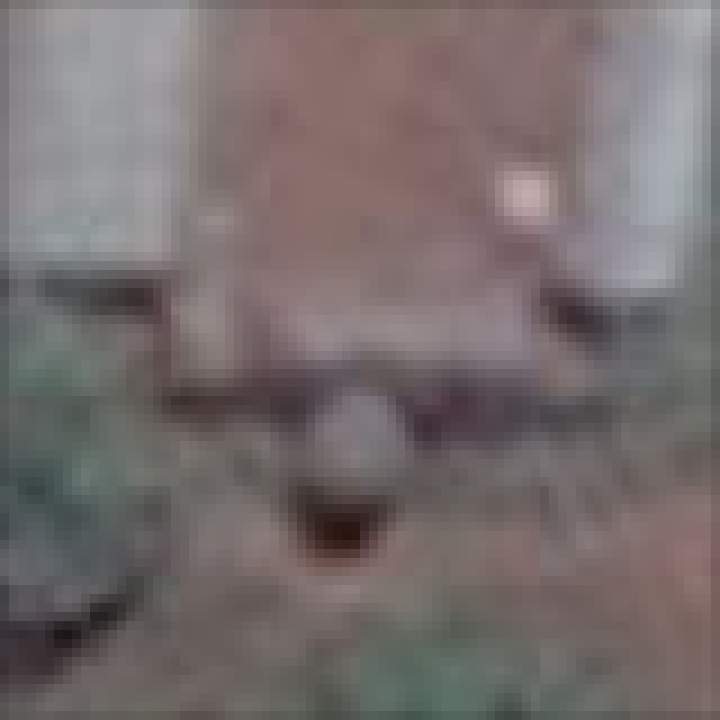}} \hfill
	\subfloat{\includegraphics[width=0.16\linewidth]{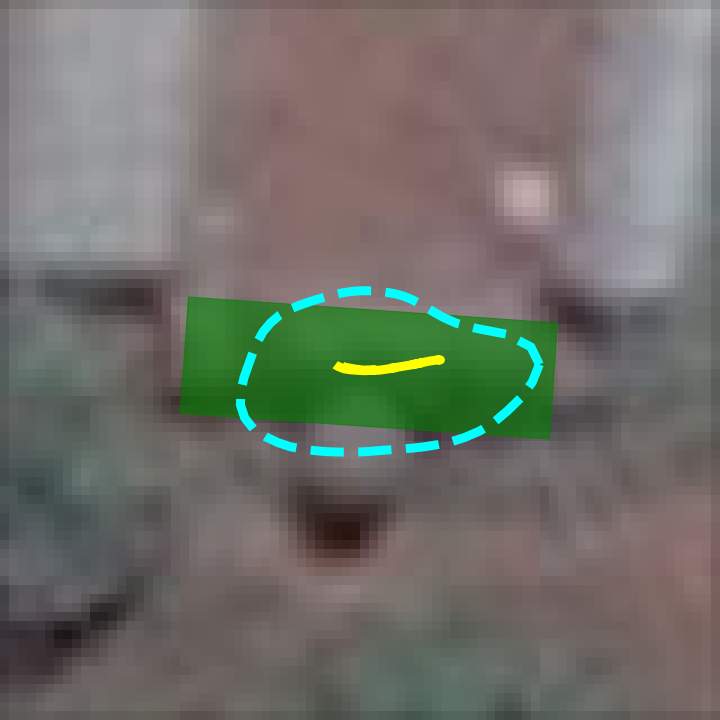} } \hfill
	\subfloat{\includegraphics[width=0.16\linewidth]{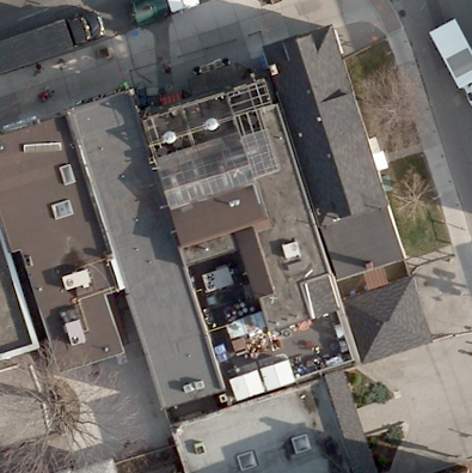}} \hfill
	\subfloat{\includegraphics[width=0.16\linewidth]{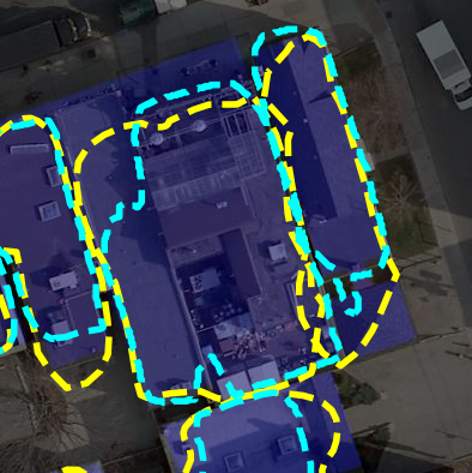} } \\ \vspace{-0.3cm}
	\subfloat[(a)]{\includegraphics[width=0.16\linewidth]{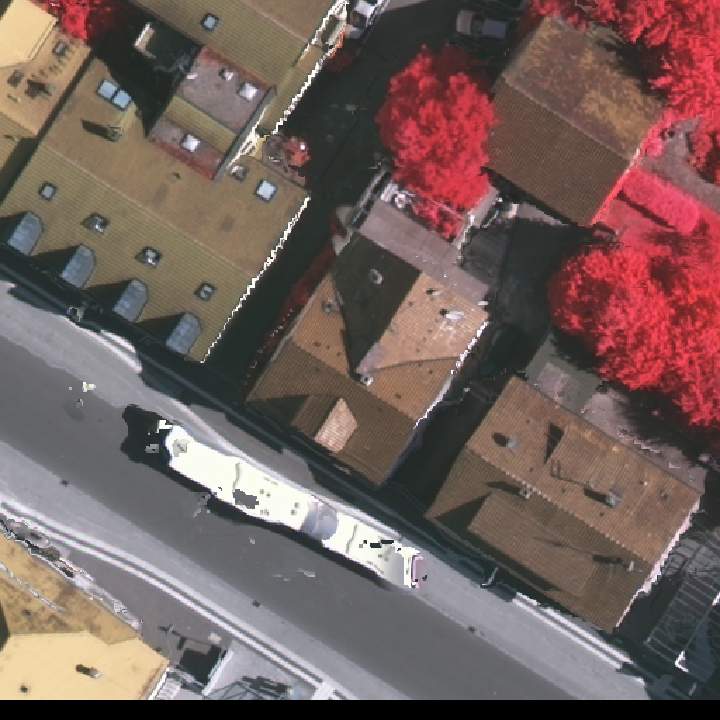}} \hfill
	\subfloat[(b)]{\includegraphics[width=0.16\linewidth]{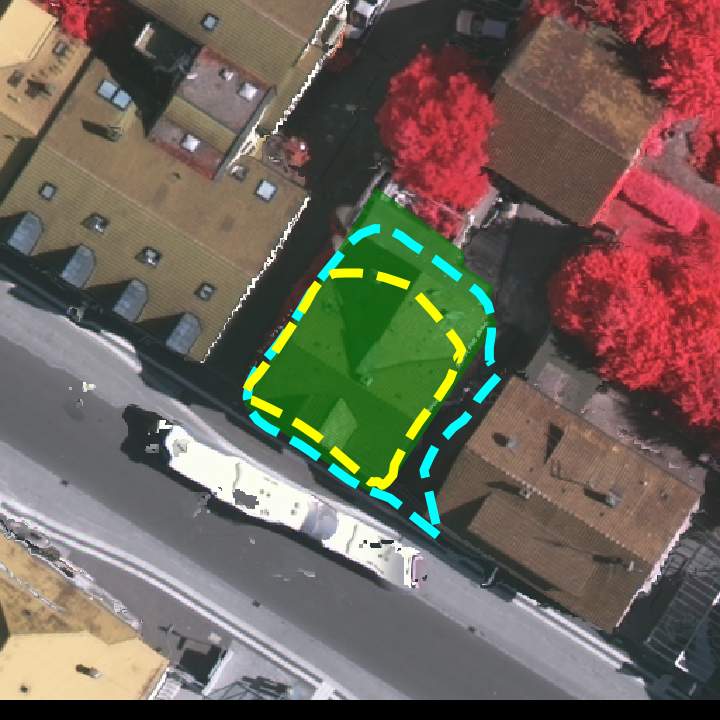}} \hfill
	\subfloat[(c)]{\includegraphics[width=0.16\linewidth]{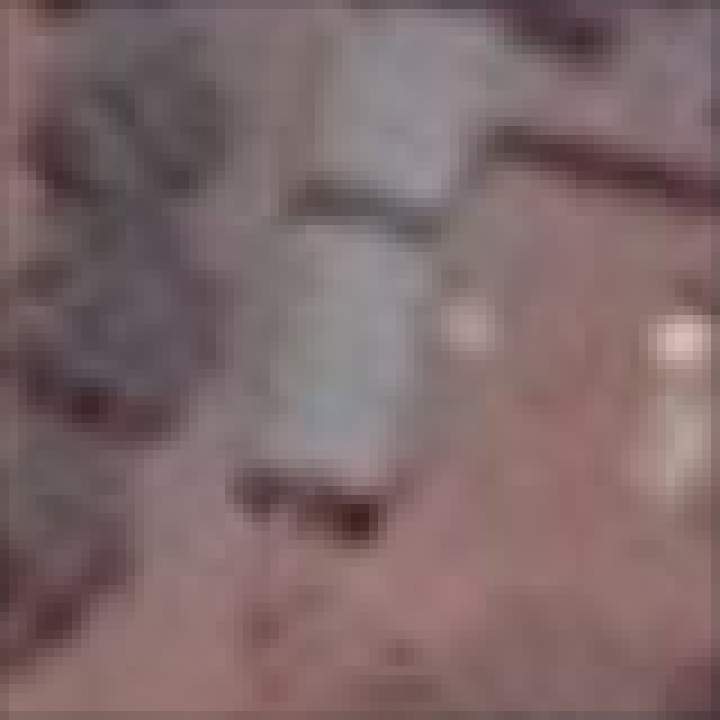}} \hfill
	\subfloat[(d)]{\includegraphics[width=0.16\linewidth]{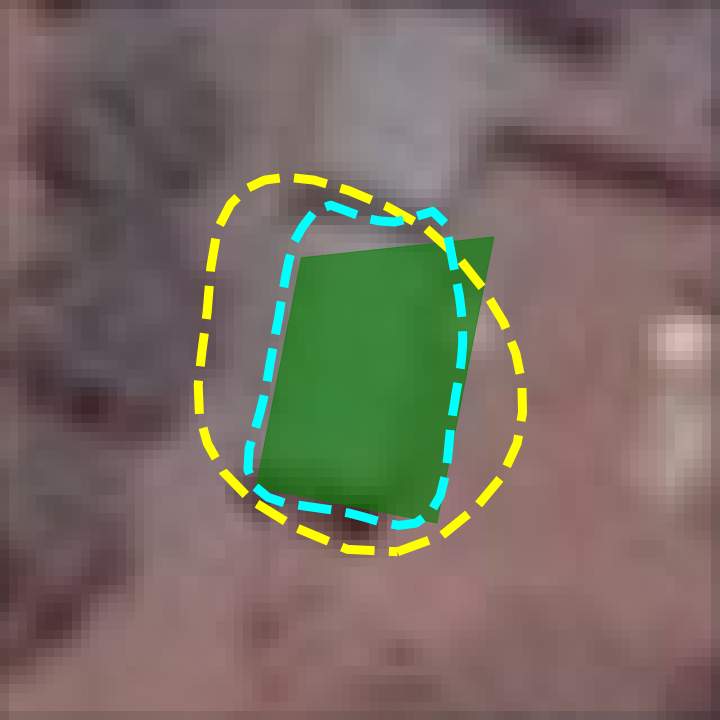} } \hfill
	\subfloat[(e)]{\includegraphics[width=0.16\linewidth]{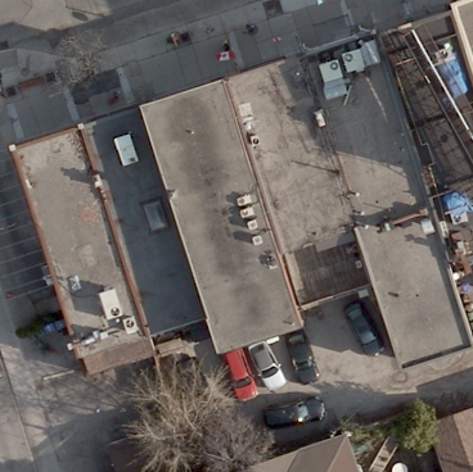}} \hfill
	\subfloat[(f)]{\includegraphics[width=0.16\linewidth]{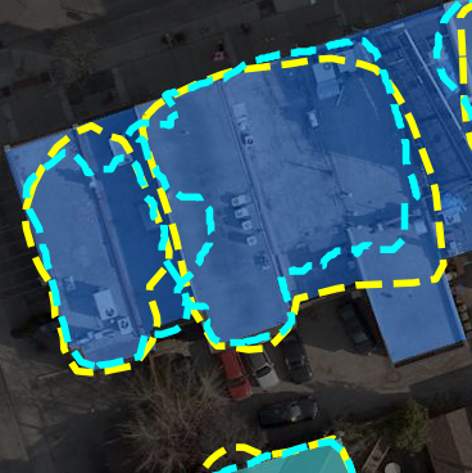} }
	\vspace{0.15cm}
	\caption{Results on (a-b) Vaihingen, (c-d) Bing Huts, (e-f) TorontoCity. Bottom two rows highlight failure cases. Original image shown in left. On right, our output is shown in cyan; DSAC output in yellow; ground truth is shaded.}
	\label{fig:qual_results}
	\vspace{-0.2cm}
\end{figure*}

}

\section{Conclusion}
\label{sec:conc}
In this paper, we presented an approach to building instance segmentation using a combination of active rays with energies predicted by a CNN. The use of a polar representation of contours enables us to predict contours that cannot self-intersect, and to employ a loss function that encourages our predicted contours to match building boundaries. Furthermore, we demonstrate a method to combine several predictions to generate more complex contours. Comparisons against other state-of-the-art methods on various buliding segmentation datasets demonstrate our method's power in generating segmentations that better capture, and are better aligned with, building boundaries.

\section*{Acknowledgments}
We gratefully acknowledge support from the Vector Institute, and NVIDIA for donation of GPUs. S.F. also acknowledges the Canada CIFAR AI Chair award at the Vector Institute.
We thank Relu Patrascu and Priyank Thatte for infrastructure support.

{\small
\bibliographystyle{ieee}
\bibliography{references}
}

\end{document}


\onecolumn

\title{Supplementary Material:\\DARNet: Deep Active Ray Network for Building Segmentation}

\author{Dominic Cheng\textsuperscript{1, 2} \qquad Renjie Liao\textsuperscript{1, 2, 3} \qquad Sanja Fidler\textsuperscript{1, 2, 4} \qquad Raquel Urtasun\textsuperscript{1, 2, 3}\\
	University of Toronto\textsuperscript{1} \quad Vector Institute\textsuperscript{2} \quad Uber ATG Toronto\textsuperscript{3} \quad NVIDIA\textsuperscript{4}\\
	{\tt\small \{dominic, rjliao, fidler\}@cs.toronto.edu} \quad {\tt\small urtasun@uber.com}
}

\maketitle
\thispagestyle{empty}

\section{Proof of Proposition}

\begin{prop}
Given a closed convex set $X$, a ray starting from any interior point of $X$ will intersect with the boundary of $X$ once.
\end{prop}

\begin{proof}
First, it is straightforward that a ray starting from any interior point of $X$ will intersect with its boundary since otherwise $X$ is not closed.
Then we prove the intersection can only happen once by contradiction.
We assume a ray starts from an interior point $a$ and intersects with the boundary of $X$ twice at $b$ and $c$.
Without loss of generality, we assume $b$ is in between $a$ and $c$.
Since $a$ is an interior point, we can find an open ball $A$ which centers at $a$ and $A \in X$.
Given any point $\tilde{a} \in A$ other than $a$, we can uniquely determine a line $l$ which crosses $\tilde{a}$ and $c$.
Then we can uniquely draw an open ball $B$ which centers at $b$ and has $l$ as the tangent line.
Since $b$ is a boundary point of $X$, we can always find a point $\tilde{b} \in B$ such that $\tilde{b} \not\in X$.
Connecting $c$ and $\tilde{b}$, we can uniquely determine a line $\tilde{l}$.
Since $\vert \angle \tilde{a}cb \vert \ge \vert \angle \tilde{b}cb \vert$, line $\tilde{l}$ will intersect with $A$ for at least once.
Denoting any intersection point as $\hat{a}$, we know $\hat{a} \in A \in X$.
Since $c \in X$, we know any point in between $c$ and $\hat{a}$, \ie, the convex combination of $c$ and $\hat{a}$, should be in $X$ due to the fact that $X$ is convex.
Therefore, $\tilde{b} \in X$ which contradicts.
We show the schematic of proof in Fig. \ref{fig:proof}.
\end{proof}

\begin{figure}[h]
	\begin{center}
		\includegraphics[height=0.4\linewidth]{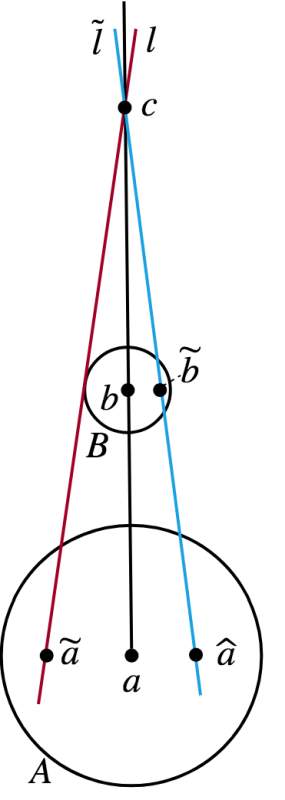}
	\end{center}
	\caption{Illustration of proof.}
	\label{fig:proof}
\end{figure}

\section{Contour Inference Details}
Our contour inference relies on the following equation,
\begin{equation}
	\rho^{(t+1)} = \rho^{(t)} - \Delta t \, (A\rho^{(t)} + f)
\end{equation}
where $\rho^{(t)}$ represents the contour at step $t$, $\Delta t$ is a time step hyper-parameter for solving the system, $A$ and $f$ consist of partial derivatives of the energy w.r.t. $\rho^{(t)}$. Here, we detail the construction of this equation.

\paragraph{Matrix equation}
As mentioned in our paper, the relevant partial derivatives are as follows. 
For the data term,
\begin{align}\label{eq:data_deriv}
\pder[E_{\text{data}}(c)]{\rho_i} = \pder[D(c_i)]{x} \cos(i\Delta\theta) + \pder[D(c_i)]{y} \sin(i\Delta\theta) 
\end{align}
where
\begin{equation}
	c_i = 
	\begin{bmatrix}
	x_c + \rho_i \cos (i \Delta\theta) \\
	y_c + \rho_i \sin (i \Delta\theta)
	\end{bmatrix}
\end{equation}
For the curvature term,
\begin{align}
\begin{split}
\pder[E_{\text{curve}}(c)]{\rho_i} \approx{}
& [2\beta(c_{i+1})\cos(2\Delta\theta)]\rho_{i+2} + \\ 
& [-4(\beta(c_{i+1})+\beta(c_{i-1}))\cos(\Delta\theta)]\rho_{i+1} + \\
& 2[\beta(c_{i+1}) + 4\beta(c_i) + \beta(c_{i-1})]\rho_i + \\ 
& [-4(\beta(c_i) + \beta(c_{i-1}))\cos(\Delta\theta)]\rho_{i-1} + \\
& [2\beta(c_{i-1})\cos(2\Delta\theta)]\rho_{i-2}
\end{split}
\end{align}
For the balloon term,
\begin{equation} \label{eq:kappa_deriv}
\pder[E_{\text{balloon}}(c)]{\rho_i} \approx - \frac{\kappa(c_i)}{\rho_{\text{max}}}
\end{equation}
Combining these into the overall energy, we obtain
\begin{align}
	\begin{split}
	\pder[E]{\rho_i} \approx{} 
	& [2\beta(c_{i+1})\cos(2\Delta\theta)]\rho_{i+2} + \\ 
	& [-4(\beta(c_{i+1})+\beta(c_{i-1}))\cos(\Delta\theta)]\rho_{i+1} + \\
	& 2[\beta(c_{i+1}) + 4\beta(c_i) + \beta(c_{i-1})]\rho_i + \\ 
	& [-4(\beta(c_i) + \beta(c_{i-1}))\cos(\Delta\theta)]\rho_{i-1} + \\
	& [2\beta(c_{i-1})\cos(2\Delta\theta)]\rho_{i-2} + \\
	& \pder[D(c_i)]{x} \cos(i\Delta\theta) + \pder[D(c_i)]{y} \sin(i\Delta\theta) - \frac{\kappa(c_i)}{\rho_{\text{max}}}
	\end{split}
\end{align}
We have $L$ such equations, one for each $\rho_i$. In each equation, there are dependencies on the four adjacent entries of $\rho_i$, and $\rho_i$ itself (\ie for entries on the borders, the indices wrap around). This summarizes into matrix form,
\begin{align}
\pder[E]{\rho} &\approx{}
	\begin{bmatrix}
	c_1     & b_1   & a_1   & 0         & \cdots    & 0         & e_1   & d_1   \\
	d_2     & c_2   & b_2   & a_2   & 0         & \cdots    & 0         & e_2   \\
	\vdots      & \vdots    & \vdots    & \vdots    & \vdots    & \ddots    & \vdots    & \vdots    \\
	a_{L-1}   & 0         & \cdots    & 0         & e_{L-1}  & d_{L-1} & c_{L-1} & b_{L-1} \\
	b_{L}     & a_{L}   & 0         & \cdots    & 0         & e_{L}   & d_{L}   & c_{L}
	\end{bmatrix}
	\begin{bmatrix}
	\rho_1     \\ \rho_2  \\ \vdots   \\ \rho_{L-1}      \\ \rho_{L}
	\end{bmatrix}
	+
	\begin{bmatrix}
	f_1 \\ f_2 \\ \vdots \\ f_{L-1} \\ f_L
	\end{bmatrix} \\
	&= A \rho + f
\end{align}
where
\begin{align}
a_i &= 2\beta(c_{i+1})\cos(2\Delta\theta) \\
b_i &= -4(\beta(c_{i+1})+\beta(c_{i-1}))\cos(\Delta\theta) \\
c_i &= 2 (\beta(c_{i+1}) + 4\beta(c_i) + \beta(c_{i-1})) \\
d_i &= -4(\beta(c_i) + \beta(c_{i-1}))\cos(\Delta\theta) \\
e_i &= 2 \beta(c_{i-1}) \cos(2\Delta\theta) \\
f_i &= \pder[D(c_i)]{x} \cos(i\Delta\theta) + \pder[D(c_i)]{y} \sin(i\Delta\theta) - \frac{\kappa(c_i)}{\rho_{\text{max}}}
\end{align}

\paragraph{Solving the system} As mentioned in \cite{kass1988snakes}, to iteratively minimize the energy, we can solve this system of equations by introducing a time variable such that the solution is found at equilibrium. That is,
\begin{equation}
	\frac{\rho^{(t+1)} - \rho^{(t)}}{\Delta t} = -(A \rho + f)
\end{equation}
where we purposefully leave out the time indices on the right hand side of the equation.
\cite{kass1988snakes} then adopts an implicit-explicit method by interpreting $A$ implicitly (so it is associated with $\rho^{(t+1)}$) and $f$ explicitly (so it is associated with $\rho^{(t)}$),
\begin{align}
\frac{\rho^{(t+1)} - \rho^{(t)}}{\Delta t} &= -(A \rho^{(t+1)} + f) \\
\rho^{(t+1)} &= \left(A + \frac{1}{\Delta t}I\right)^{-1} \left( \frac{1}{\Delta t}\rho^{(t)} - f\right)
\end{align}
To avoid the need to perform a batched matrix inverse, we instead adopt an explicit method,
\begin{align}
\frac{\rho^{(t+1)} - \rho^{(t)}}{\Delta t} &= -(A\rho^{(t)} + f) \\
\rho^{(t+1)} &= \rho^{(t)} - \Delta t (A\rho^{(t)} + f)
\end{align}

\section{CNN Architecture}
\paragraph{Our CNN backbone}
Our CNN backbone uses Dilated Residual Network \cite{yu2017dilated}, specifically DRN-D-22 with its last pooling and fully-connected layers removed, as the main method of feature extraction. We append additional upsampling layers to yield output maps that match the input image size. This is illustrated in Figure \ref{fig:cnn_arch}.
\begin{figure}[h]
	\begin{center}
		\includegraphics[height=0.95\linewidth]{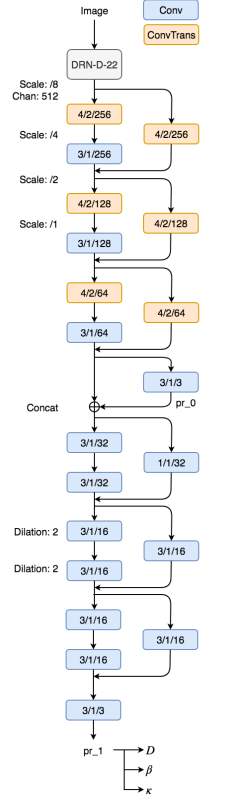}
	\end{center}
	\caption{CNN architecture. Convolutions shown in blue; transposed convolutions shown in orange. Layers are annotated as kernel\_size/stride/output\_channels. Dilation is set to 1 unless stated otherwise. Output sizes are listed as a scale factor from the original input image. All residual blocks are combinations of Convolution-BatchNorm-ReLU in the style of \cite{yu2017dilated}, except for pr\_0 and pr\_1, where BatchNorm \cite{ioffe2015batch} is not used.}
	\label{fig:cnn_arch}
\end{figure}

\paragraph{Adaptation to Deep Structured Active Contours (DSAC)}
To implement DSAC \cite{marcos2018learning}, we leverage the same architecture in Figure \ref{fig:cnn_arch}, except with $4$ outputs corresponding to the energy maps required in that framework. Following the implementation of DSAC, we add a gaussian smoothing layer with kernel size $9$ and $\sigma=2$ to the final output of the data term.

\section{Hyper-parameters}
We train our network using SGD with momentum. We choose a learning rate of $4\times 10^{-5}$, which halves every $E$ epochs, momentum of $0.3$, a weight decay of $1\times 10^{-5}$, and a batch size of $10$. We set $E = 30$ for Vaihingen and Bing Huts, and $E = 1$ for TorontoCity. We train for $100$ epochs on Vaihingen and Bing Huts, and $10$ epochs on TorontoCity.

To encourage stability in contour inference without using common techniques that are non-differentiable, we pretrain the maps to output values that cause the contour to converge, although not necessarily close to the ground truth rays. 
Specifically, we leverage the Euclidean distance transform because it possesses some desirable properties. 
Recall that we wish for $D$ to assign relatively lower values to the building boundaries, such that its gradient near these boundaries can attract contour points towards it. 
Both of these properties are reflected in the distance transform. 
For $\beta$, we adopt the distance transform with the building interiors masked out; we wish for contour points to evolve outwards without restriction, and straighten out as it approaches the boundaries. 
For $\kappa$, we adopt the distance transform with the building exteriors masked out. 
To pretrain, we take the predictions pr\_0 and pr\_1 (from Figure \ref{fig:cnn_arch}) and regress to these distance transforms with a smooth $L_1$ loss, using Adam \cite{kingma2014adam} with an initial learning rate of $1\times 10^{-3}$, which halves every $E$ epochs, and weight decay of $4\times 10^{-4}$. We set $E = 50$ for Vaihingen and Bing Huts, and $E = 3$ for TorontoCity.
We pretrain for $250$ epochs on Vaihingen and Bing Huts, and $10$ epochs on TorontoCity. 
After pretraining, we scale the $\beta$ and $\kappa$ maps by $0.005$ and $0.1$ respectively so that, during the initial stages of training, the contours do not move too far. 
We found this procedure increases the stability of training. 

\section{More Visual Examples}
We show more examples in Figures \ref{fig:qual_results_1}, \ref{fig:qual_results_2}, \ref{fig:qual_results_3}, and \ref{fig:area}.
{
	\captionsetup[subfigure]{labelformat=empty}
	\begin{figure*}
		\centering
		\subfloat{\includegraphics[width=0.16\linewidth]{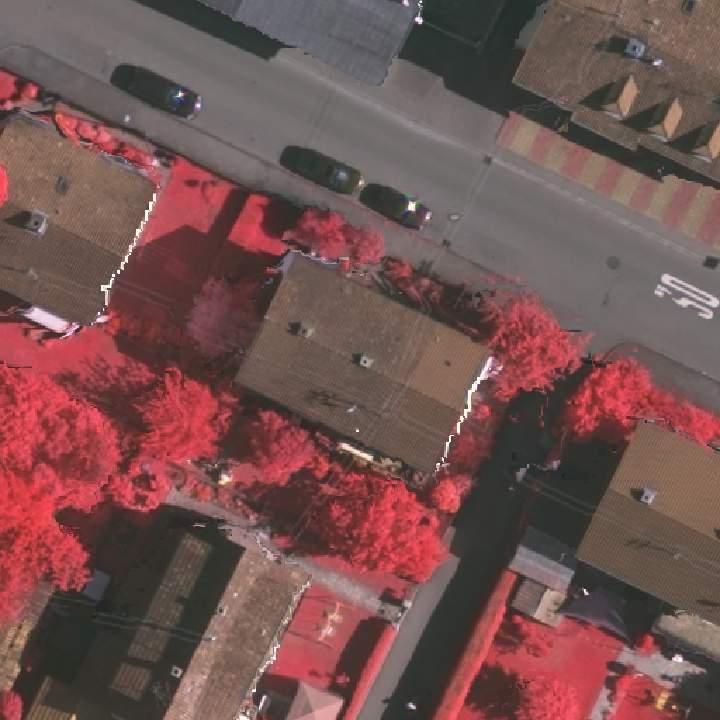}} \hfill
		\subfloat{\includegraphics[width=0.16\linewidth]{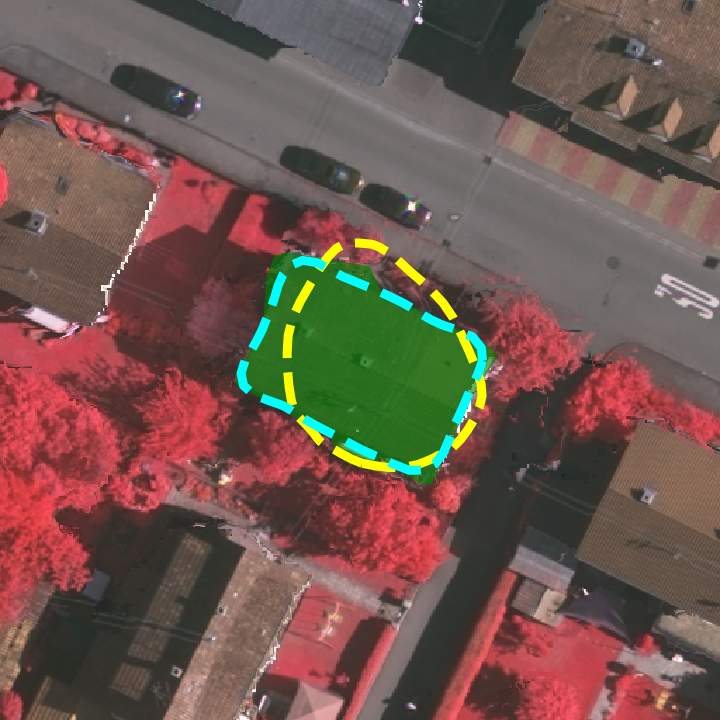}} \hfill
		\subfloat{\includegraphics[width=0.16\linewidth]{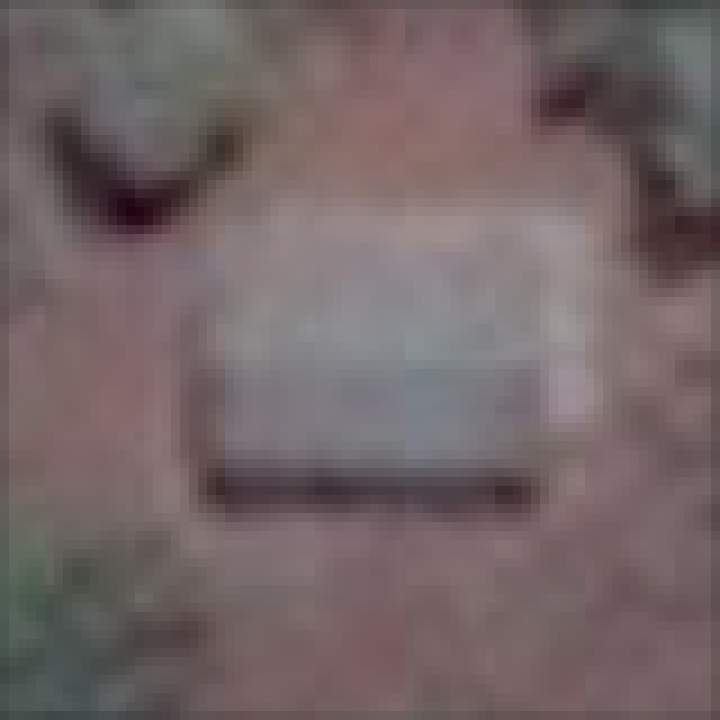}} \hfill
		\subfloat{\includegraphics[width=0.16\linewidth]{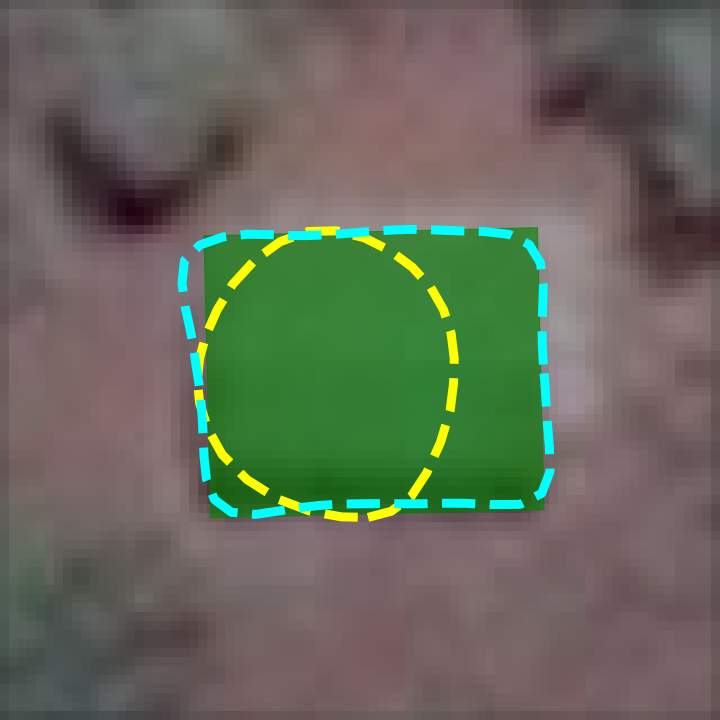} } \hfill
		\subfloat{\includegraphics[width=0.16\linewidth]{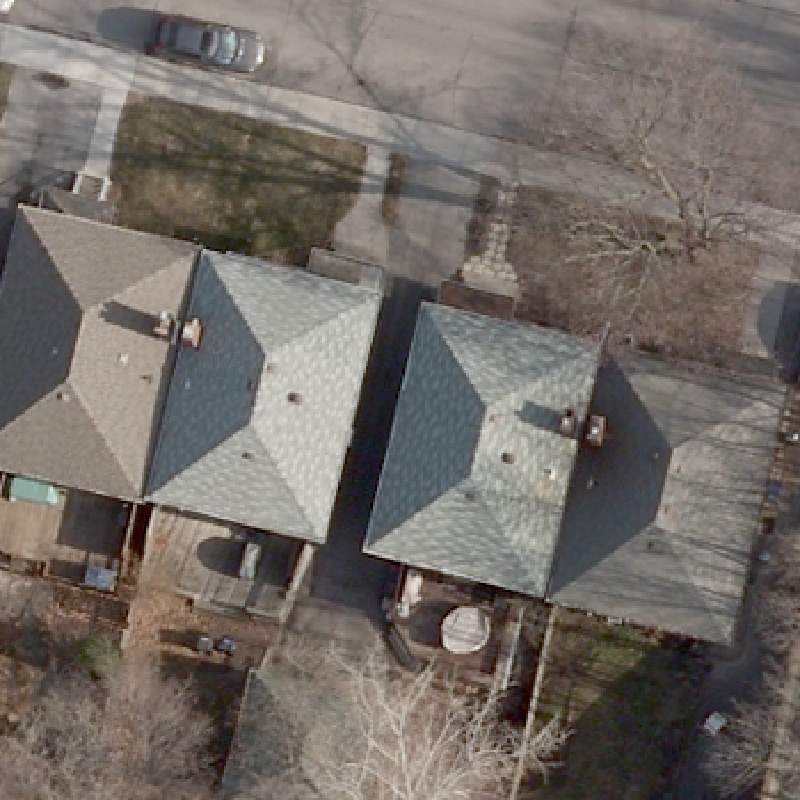}} \hfill
		\subfloat{\includegraphics[width=0.16\linewidth]{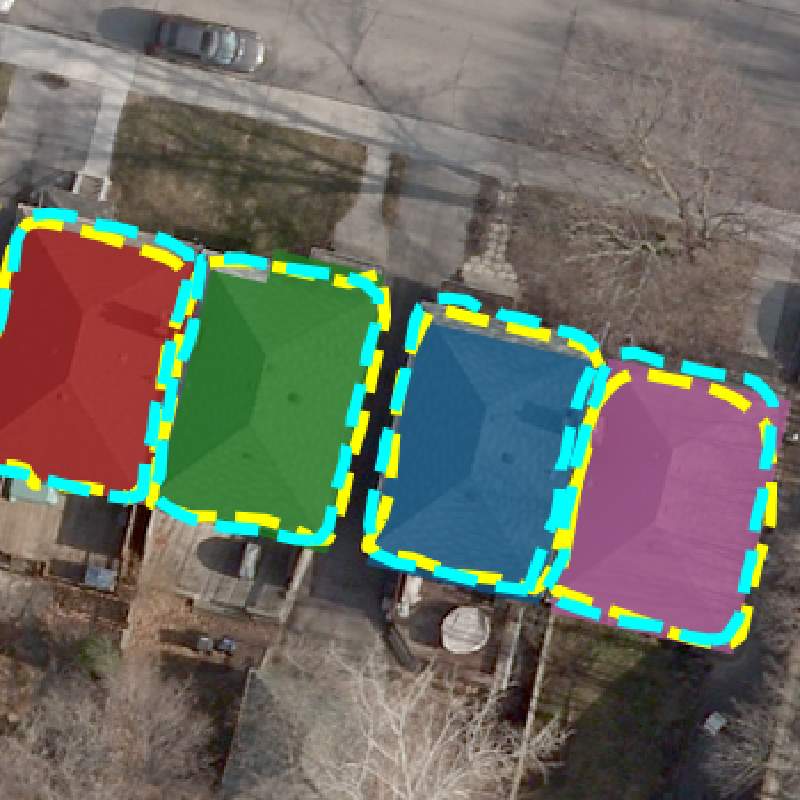} } \\ \vspace{-0.3cm}
		\subfloat{\includegraphics[width=0.16\linewidth]{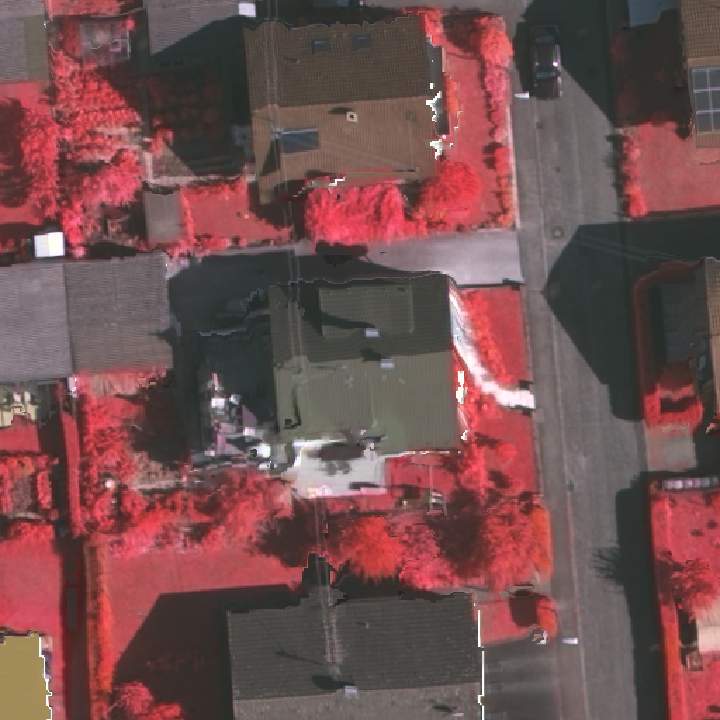}} \hfill
		\subfloat{\includegraphics[width=0.16\linewidth]{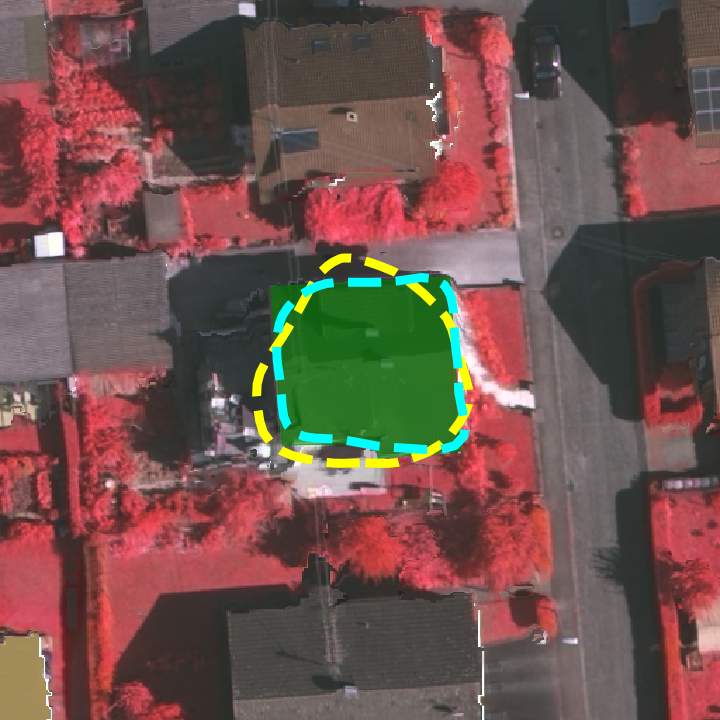}} \hfill
		\subfloat{\includegraphics[width=0.16\linewidth]{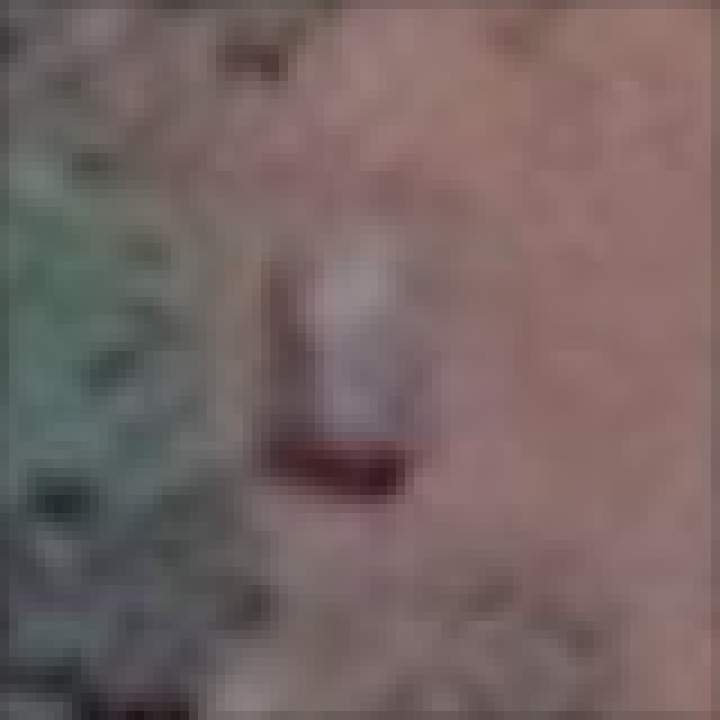}} \hfill
		\subfloat{\includegraphics[width=0.16\linewidth]{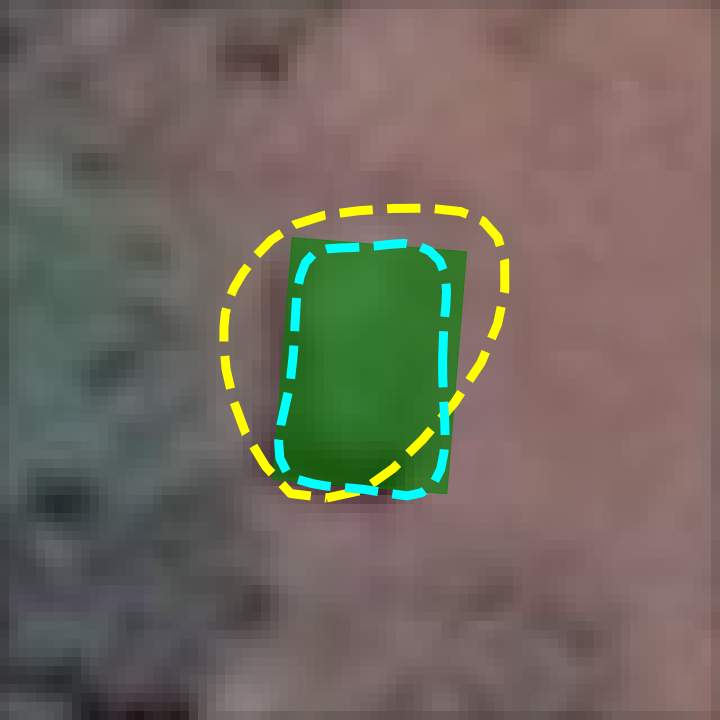} } \hfill
		\subfloat{\includegraphics[width=0.16\linewidth]{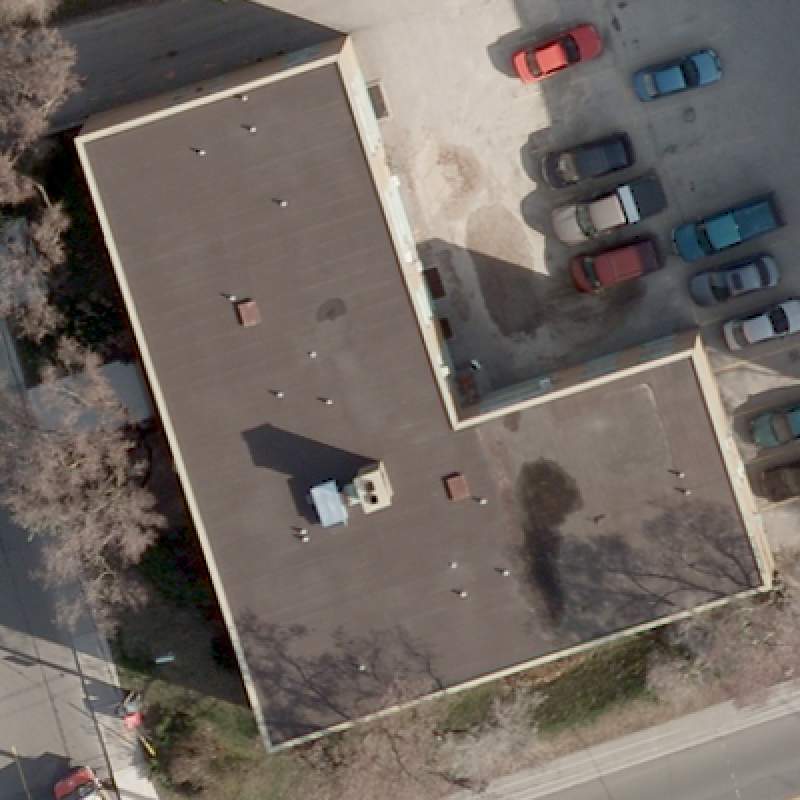}} \hfill
		\subfloat{\includegraphics[width=0.16\linewidth]{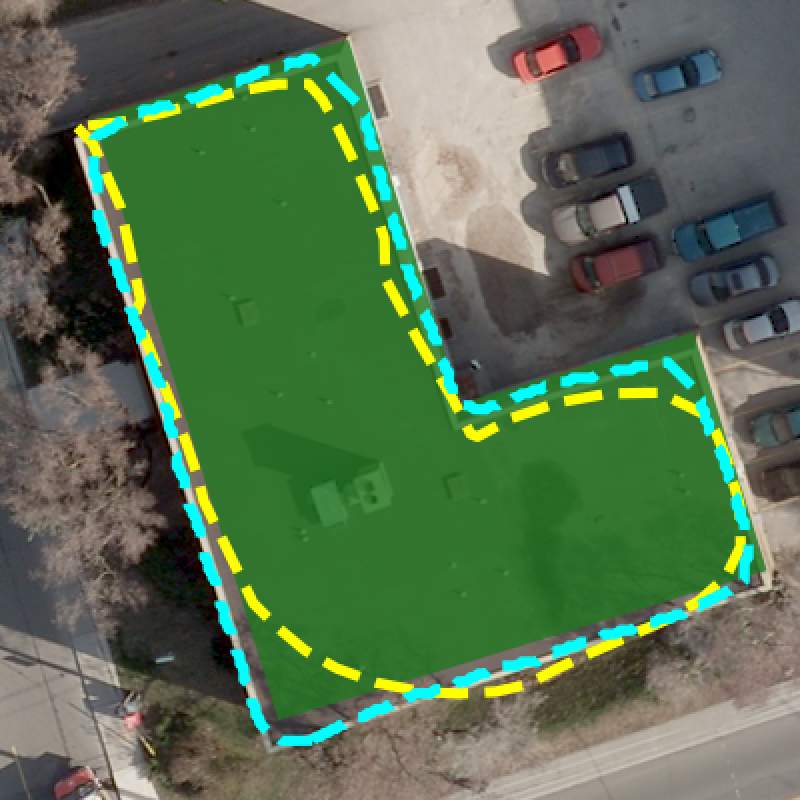} } \\ \vspace{-0.3cm}
		\subfloat{\includegraphics[width=0.16\linewidth]{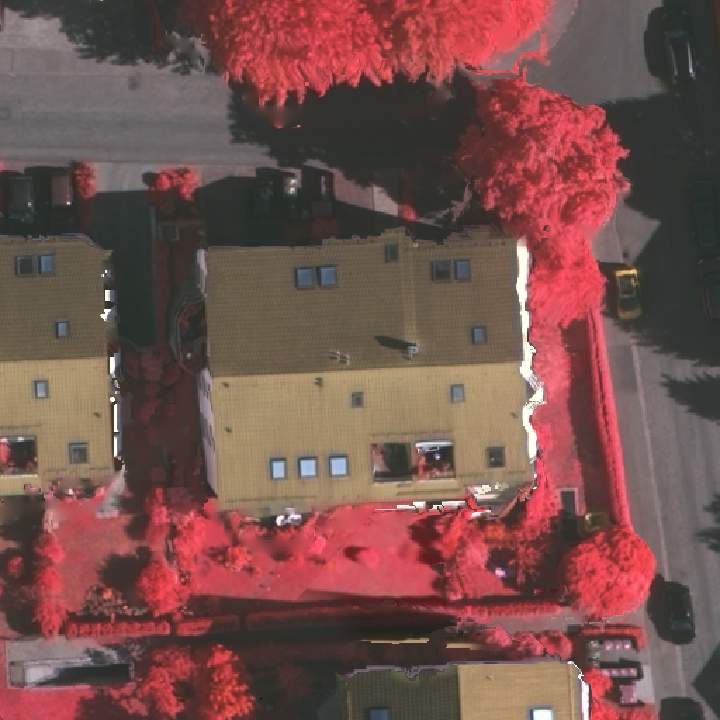}} \hfill
		\subfloat{\includegraphics[width=0.16\linewidth]{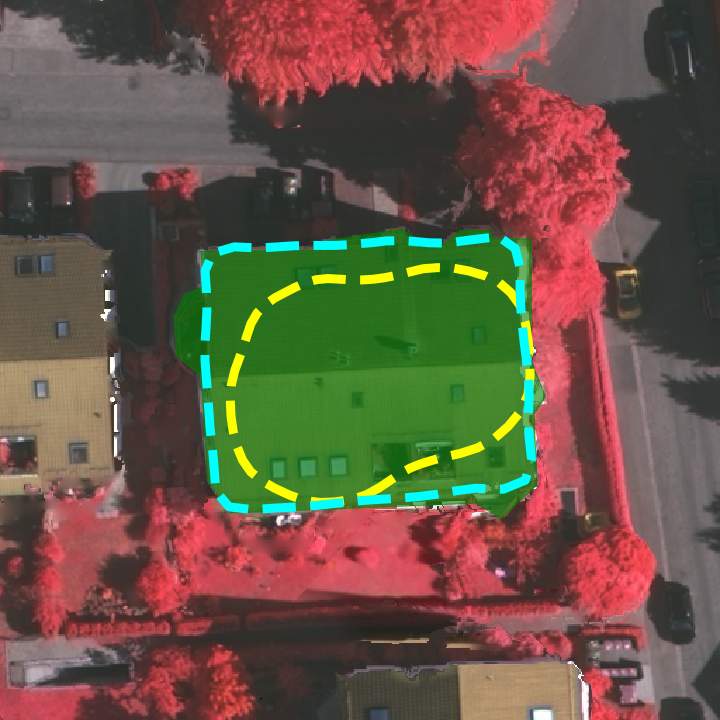}} \hfill
		\subfloat{\includegraphics[width=0.16\linewidth]{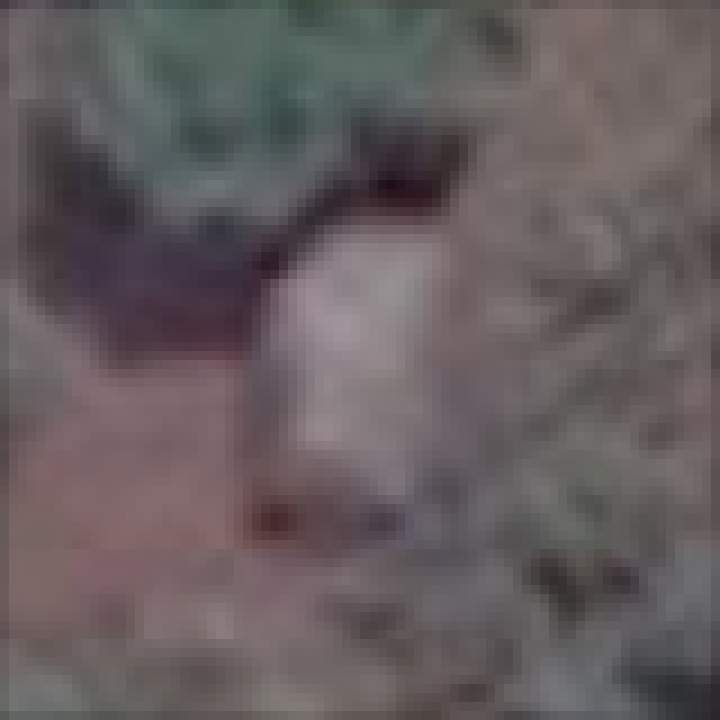}} \hfill
		\subfloat{\includegraphics[width=0.16\linewidth]{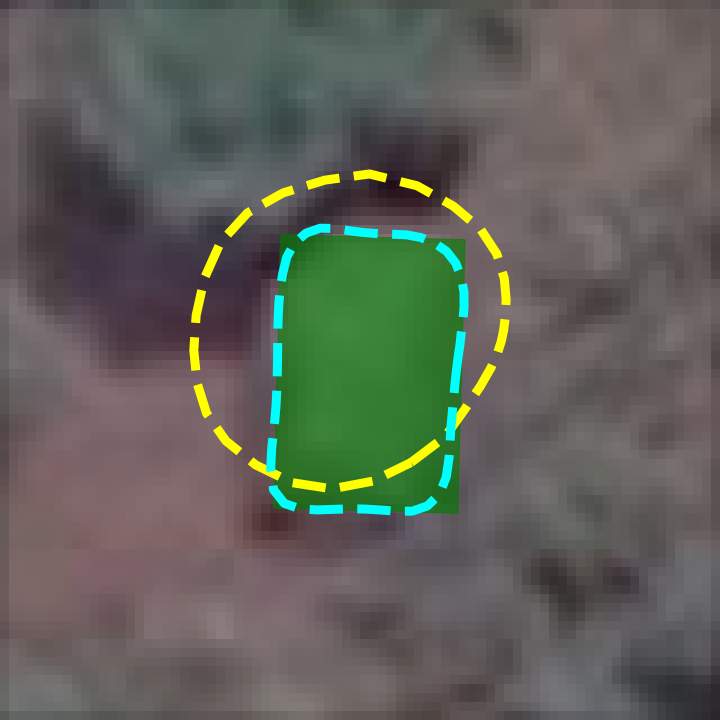} } \hfill
		\subfloat{\includegraphics[width=0.16\linewidth]{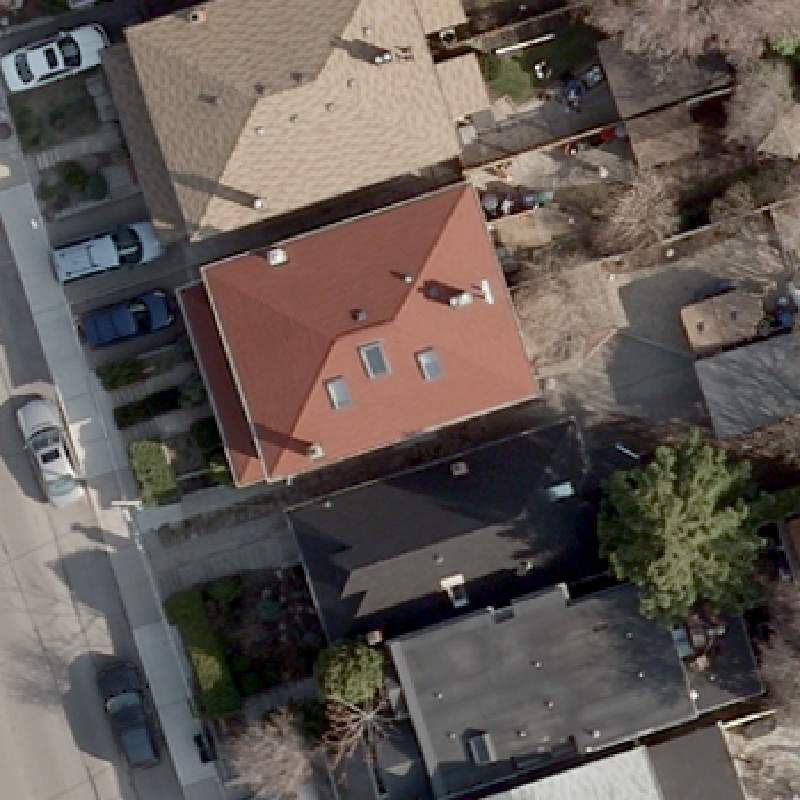}} \hfill
		\subfloat{\includegraphics[width=0.16\linewidth]{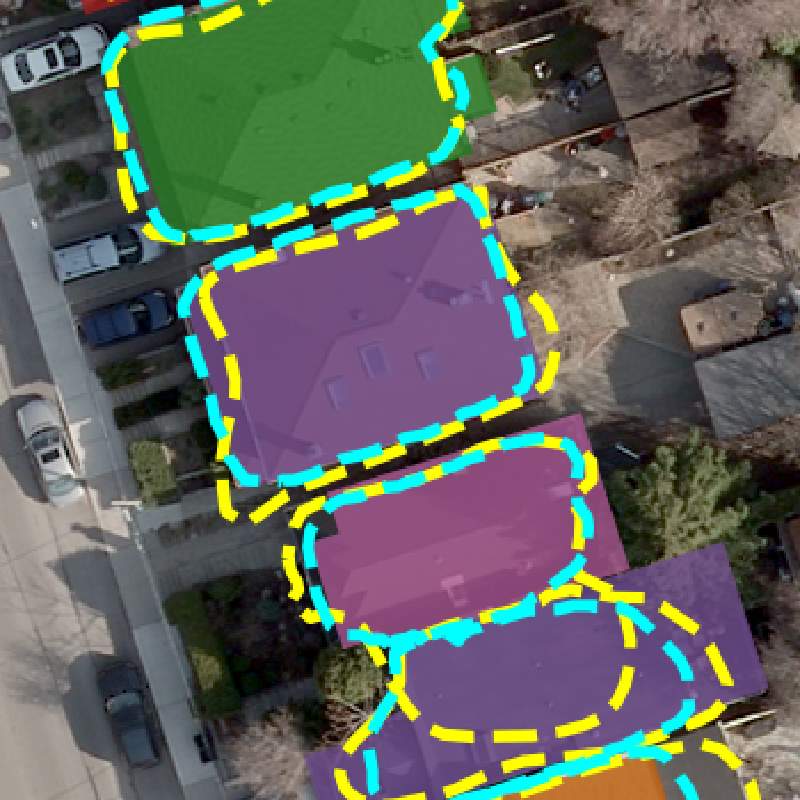} } \\ \vspace{-0.3cm}
		\subfloat{\includegraphics[width=0.16\linewidth]{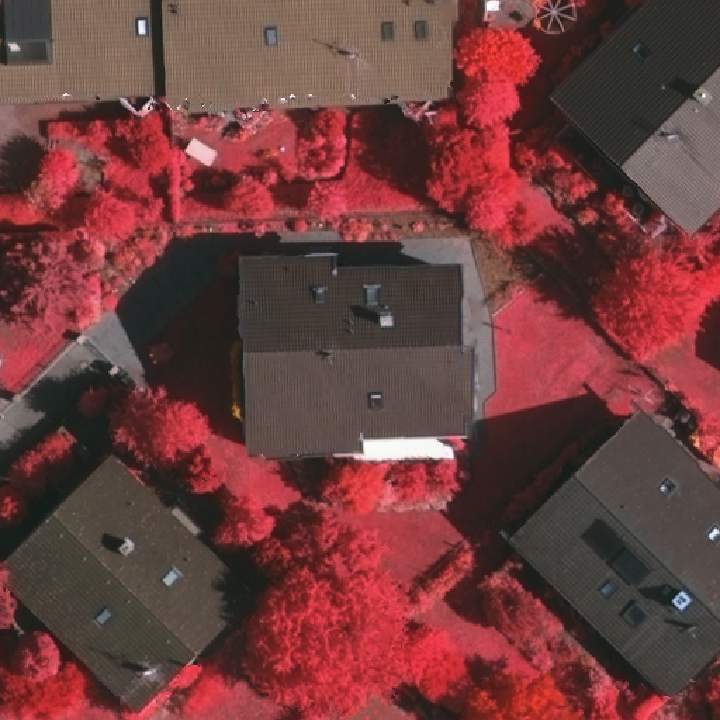}} \hfill
		\subfloat{\includegraphics[width=0.16\linewidth]{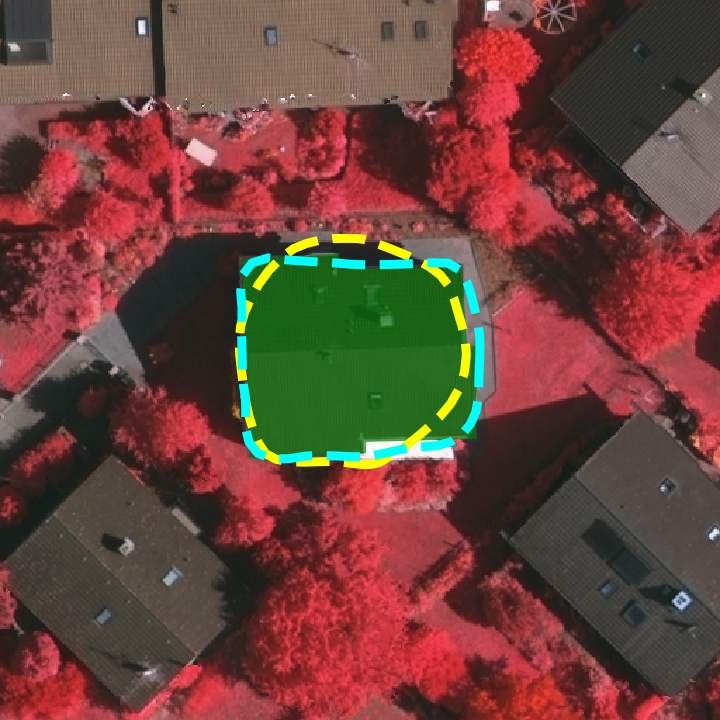}} \hfill
		\subfloat{\includegraphics[width=0.16\linewidth]{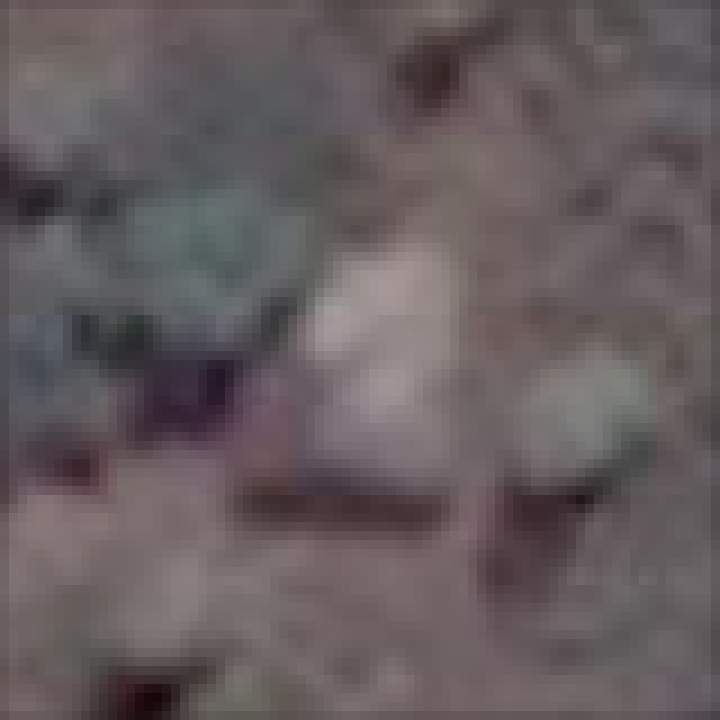}} \hfill
		\subfloat{\includegraphics[width=0.16\linewidth]{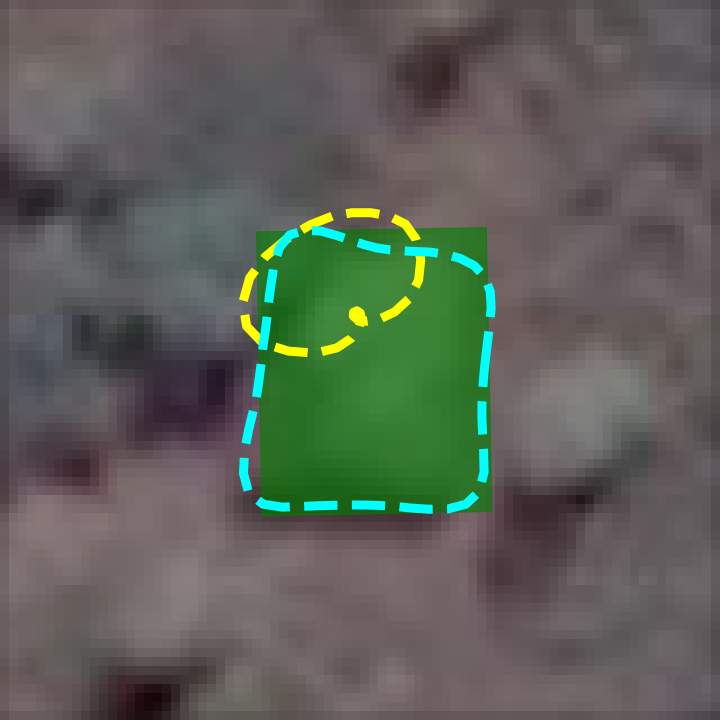} } \hfill
		\subfloat{\includegraphics[width=0.16\linewidth]{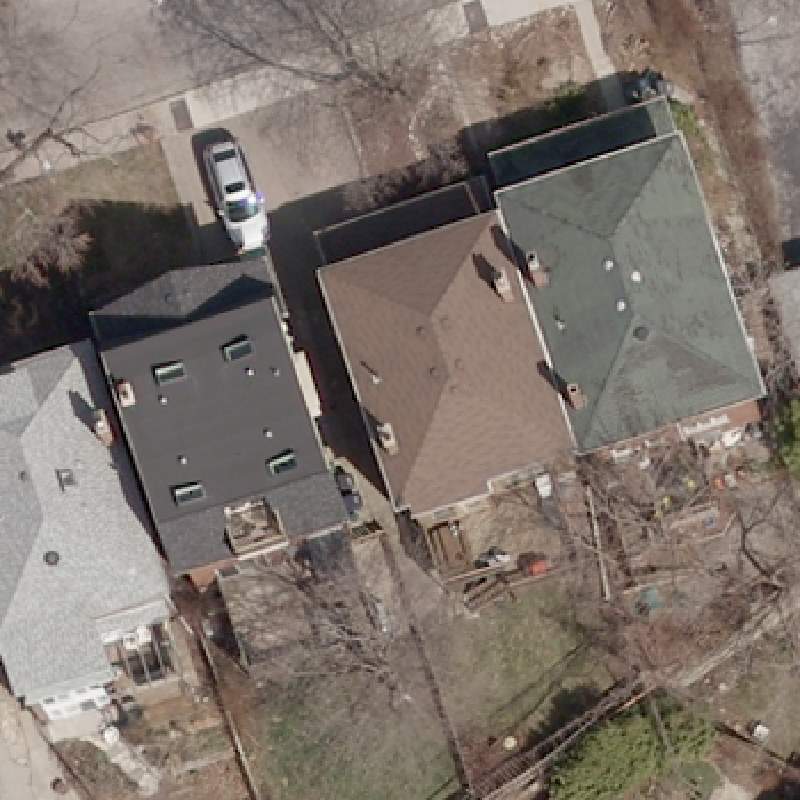}} \hfill
		\subfloat{\includegraphics[width=0.16\linewidth]{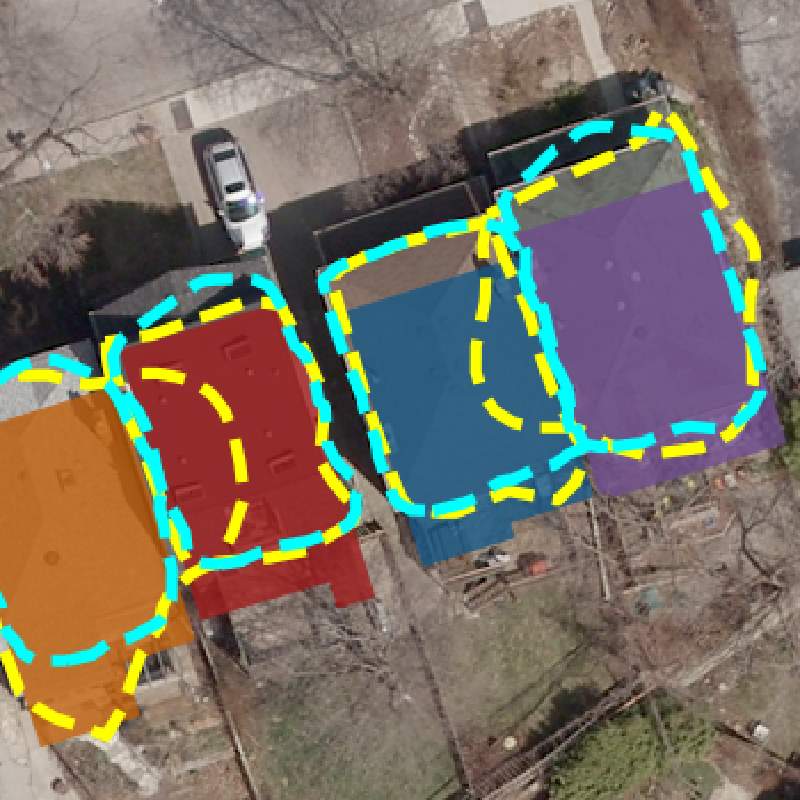} } \\ \vspace{-0.3cm}
		\subfloat{\includegraphics[width=0.16\linewidth]{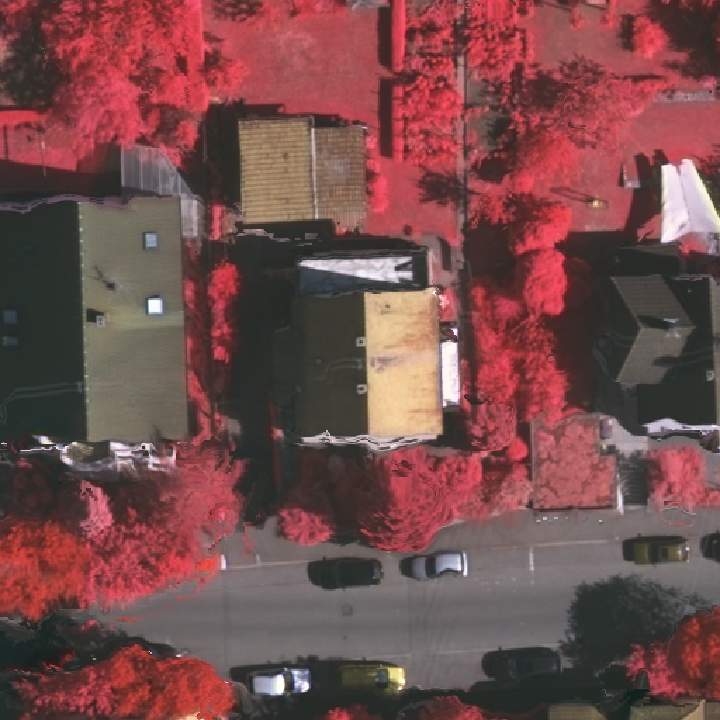}} \hfill
		\subfloat{\includegraphics[width=0.16\linewidth]{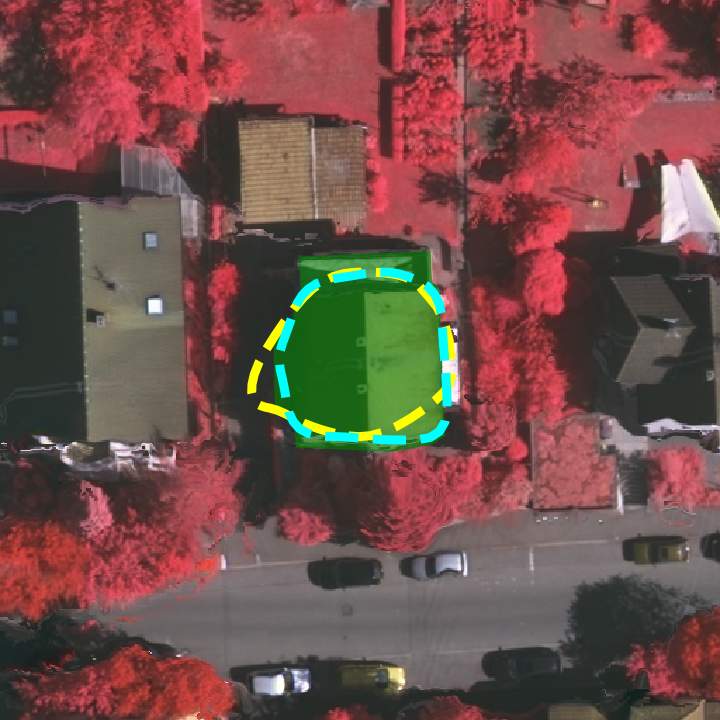}} \hfill
		\subfloat{\includegraphics[width=0.16\linewidth]{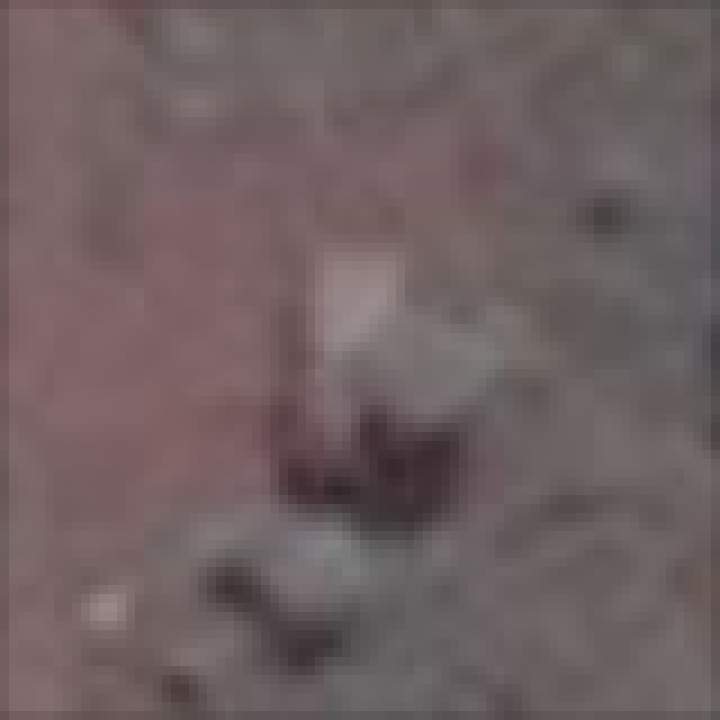}} \hfill
		\subfloat{\includegraphics[width=0.16\linewidth]{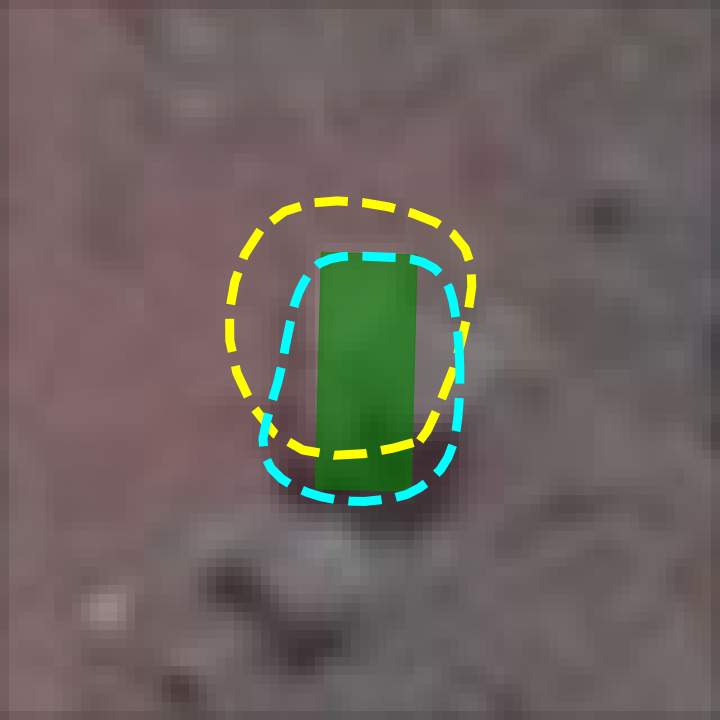} } \hfill
		\subfloat{\includegraphics[width=0.16\linewidth]{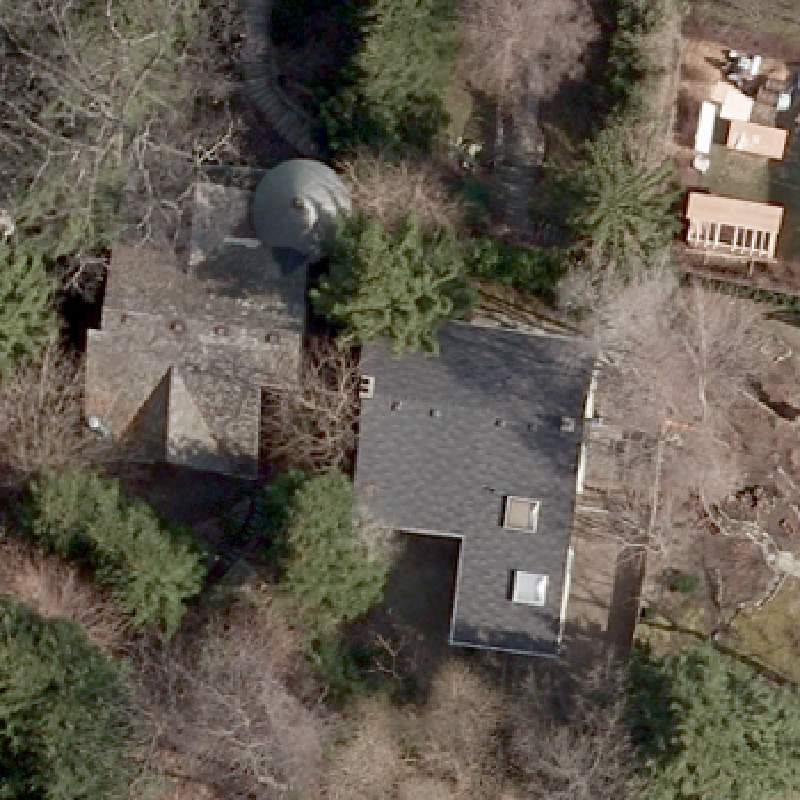}} \hfill
		\subfloat{\includegraphics[width=0.16\linewidth]{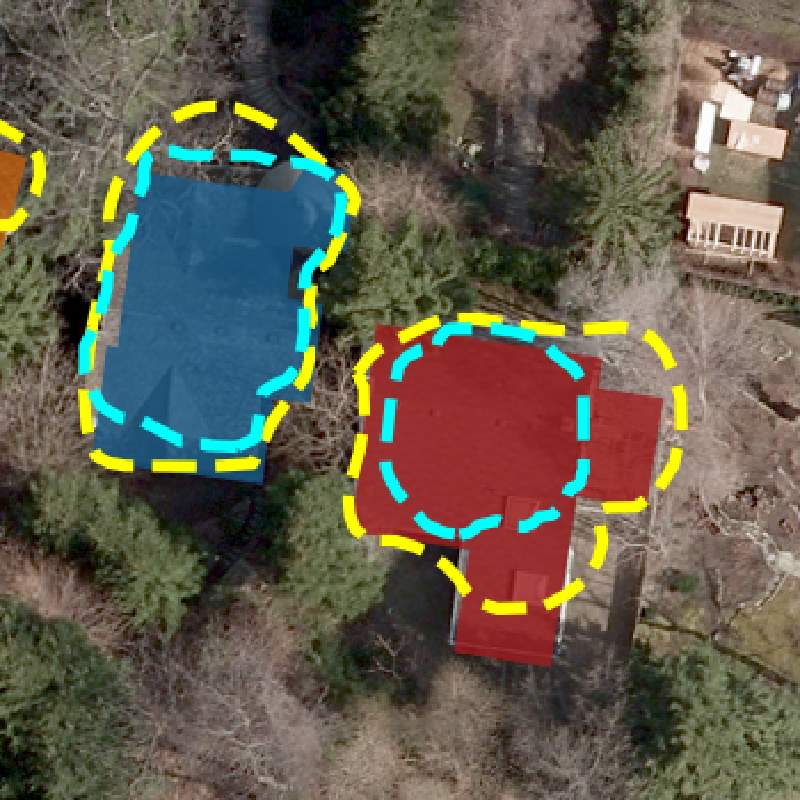} } \\ \vspace{-0.3cm}
		\subfloat{\includegraphics[width=0.16\linewidth]{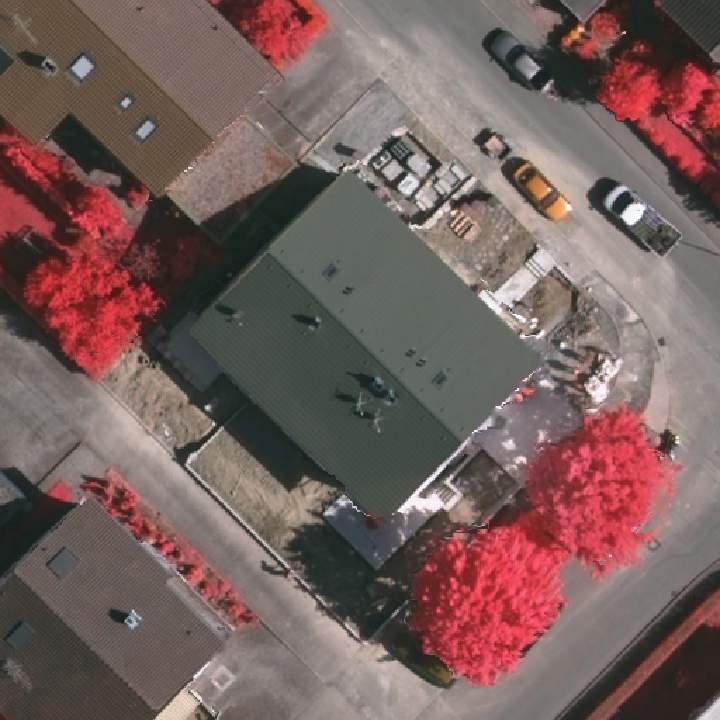}} \hfill
		\subfloat{\includegraphics[width=0.16\linewidth]{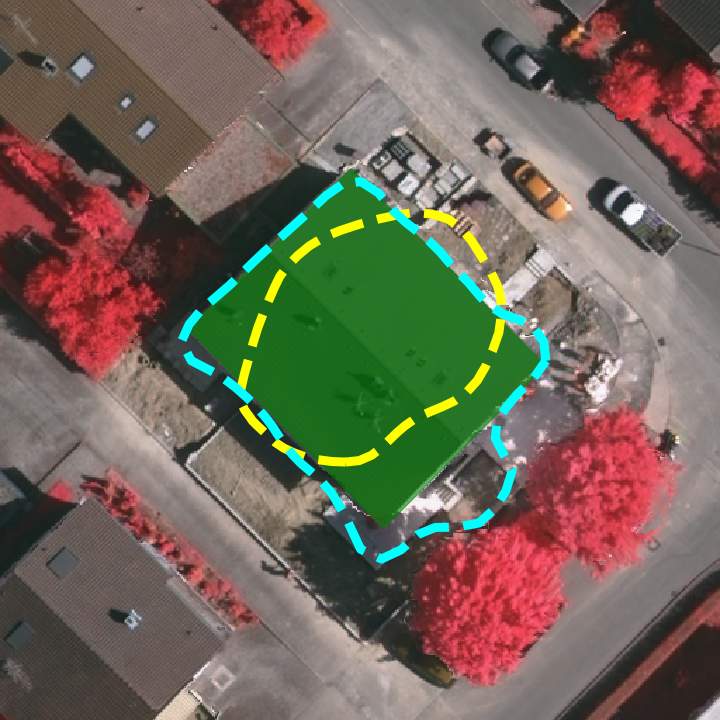}} \hfill
		\subfloat{\includegraphics[width=0.16\linewidth]{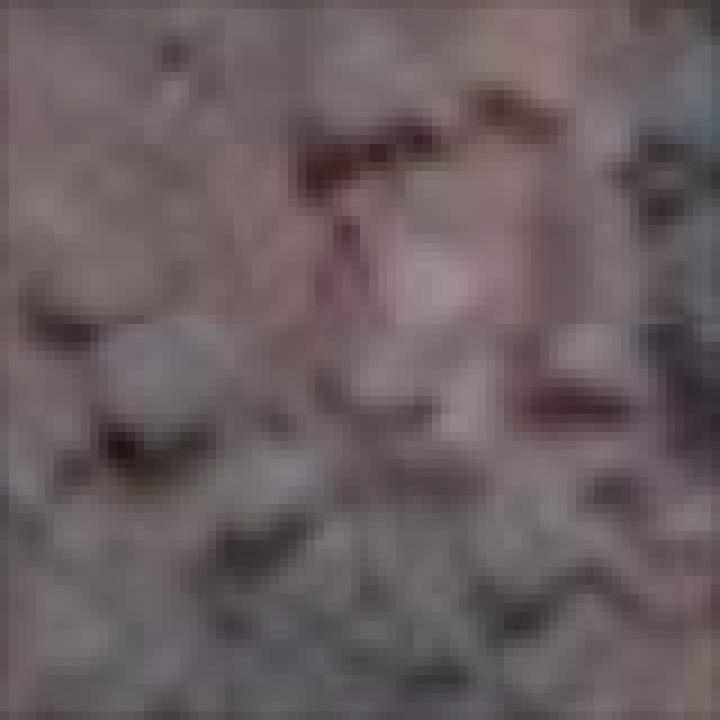}} \hfill
		\subfloat{\includegraphics[width=0.16\linewidth]{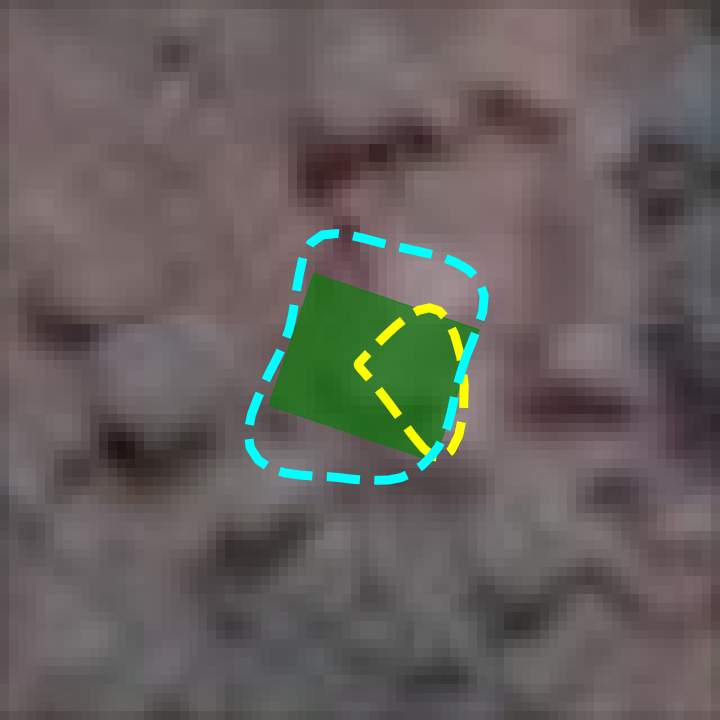} } \hfill
		\subfloat{\includegraphics[width=0.16\linewidth]{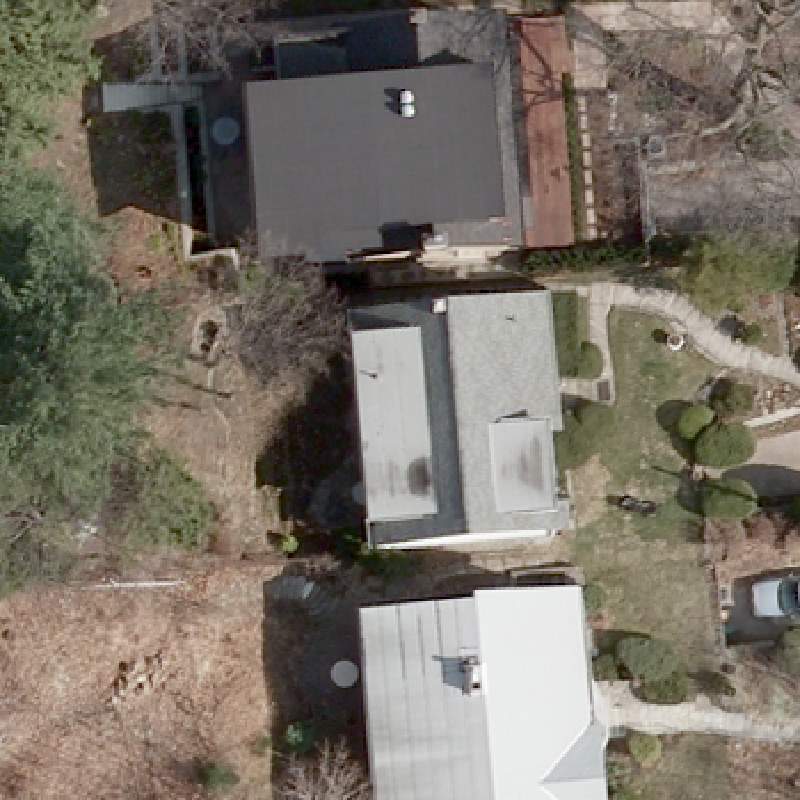}} \hfill
		\subfloat{\includegraphics[width=0.16\linewidth]{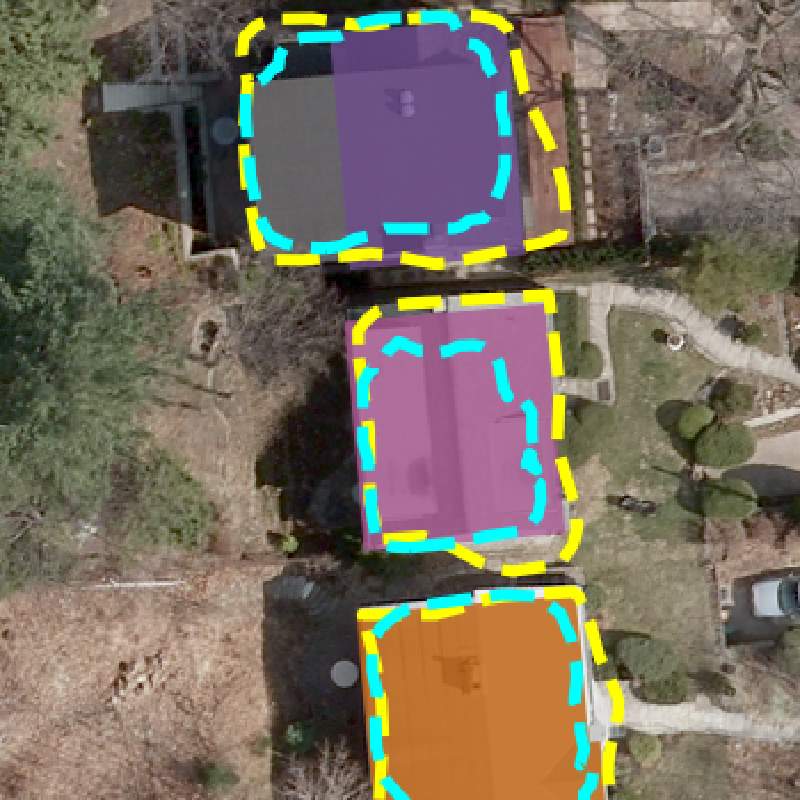} } \\ \vspace{-0.3cm}
		\subfloat[(a)]{\includegraphics[width=0.16\linewidth]{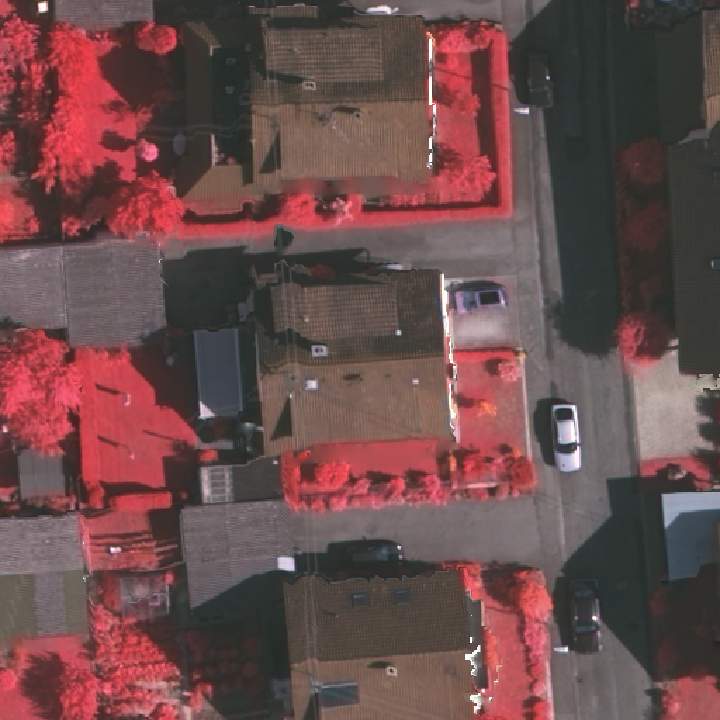}} \hfill
		\subfloat[(b)]{\includegraphics[width=0.16\linewidth]{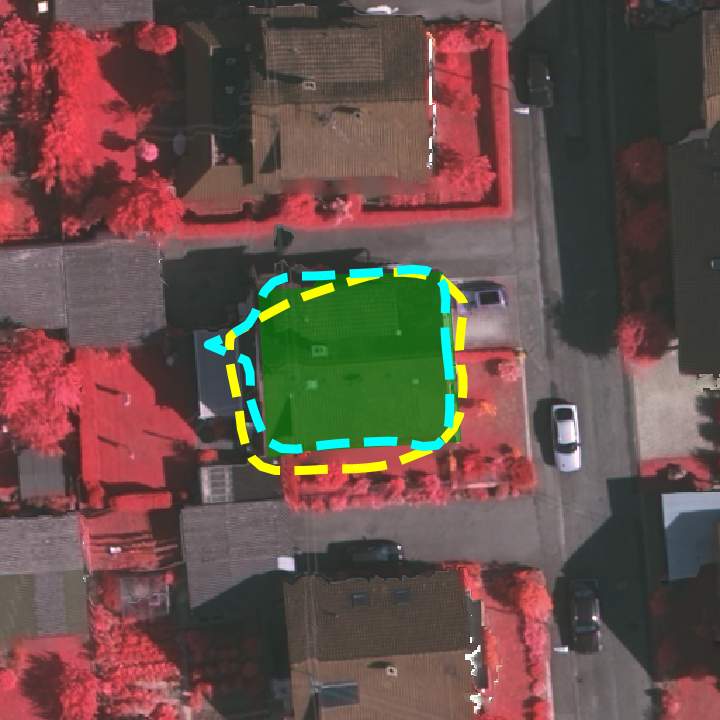}} \hfill
		\subfloat[(c)]{\includegraphics[width=0.16\linewidth]{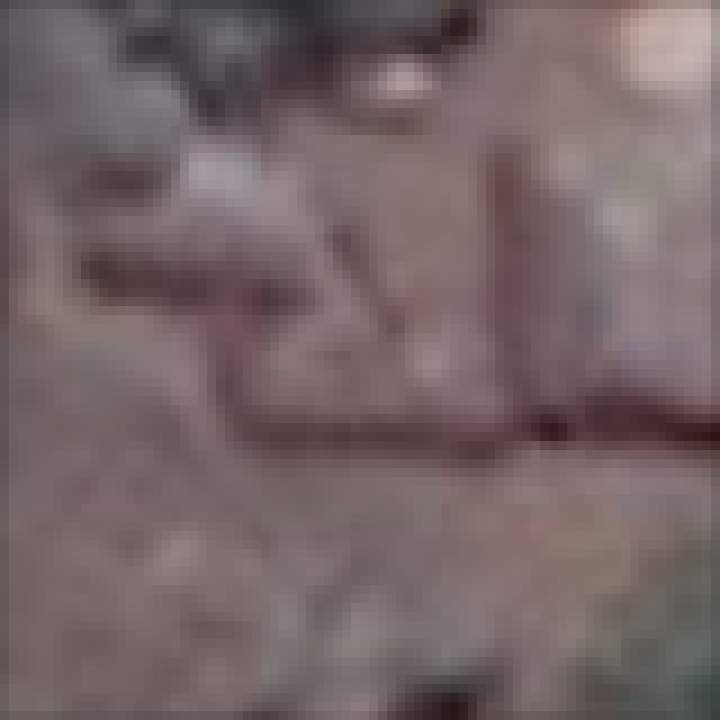}} \hfill
		\subfloat[(d)]{\includegraphics[width=0.16\linewidth]{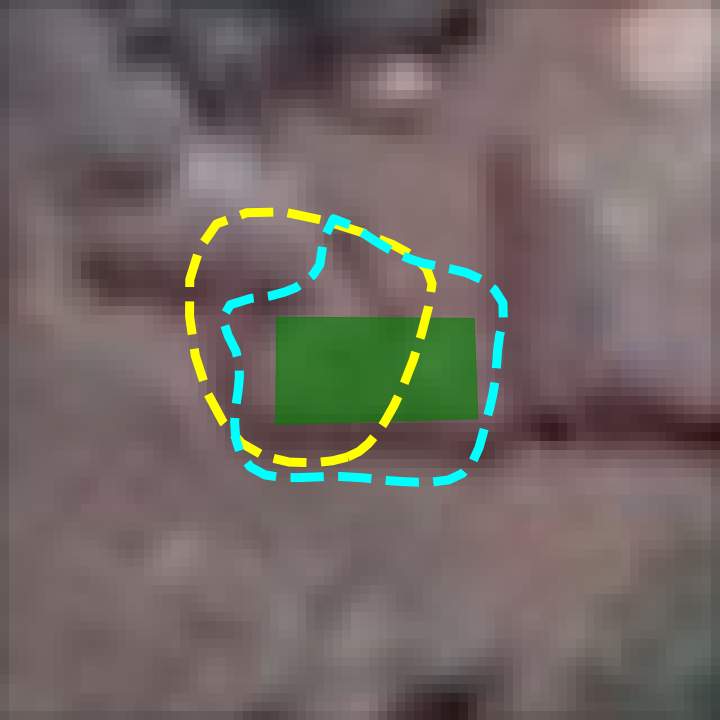} } \hfill
		\subfloat[(e)]{\includegraphics[width=0.16\linewidth]{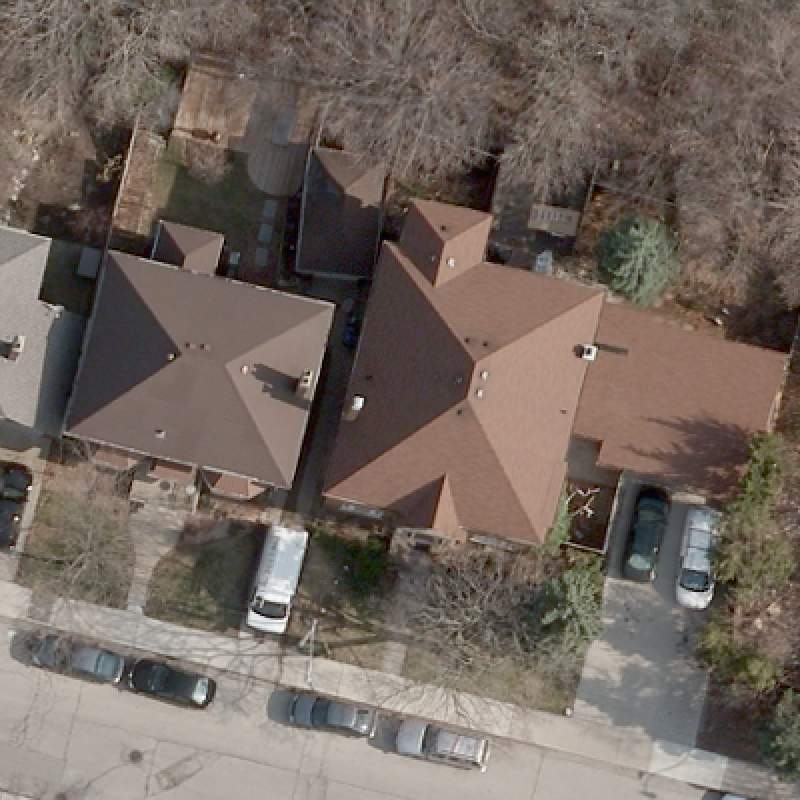}} \hfill
		\subfloat[(f)]{\includegraphics[width=0.16\linewidth]{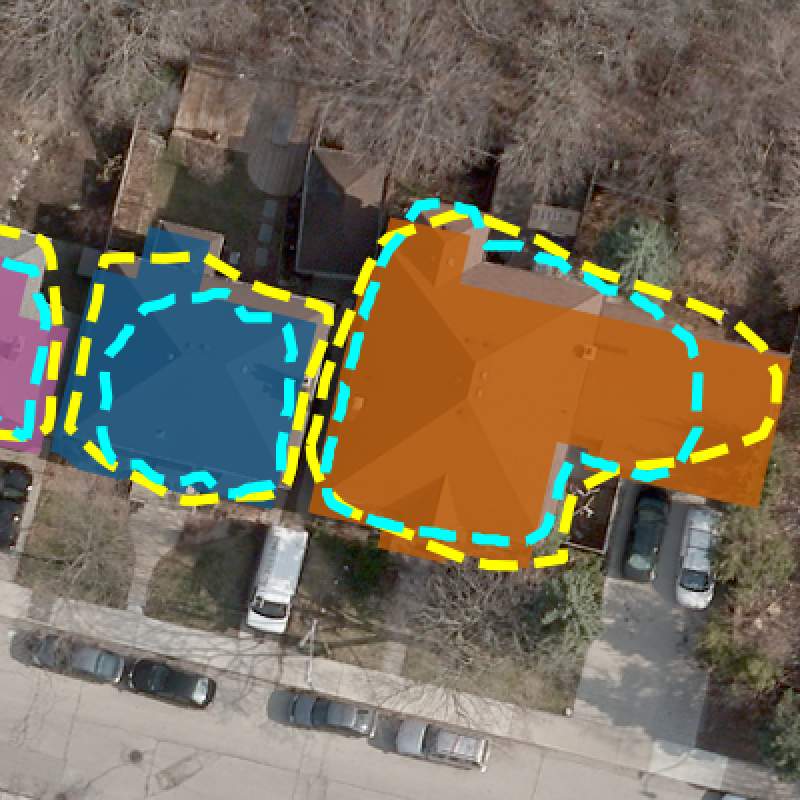}
		}
		\caption{Results on (a-b) Vaihingen, (c-d) Bing Huts, (e-f) TorontoCity. Bottom three rows highlight failure cases. Original image shown in left. On right, our output is shown in cyan; DSAC output in yellow; ground truth is shaded.}
		\label{fig:qual_results_1}
	\end{figure*}
}

{
	\captionsetup[subfigure]{labelformat=empty}
	\begin{figure*}
		\centering
		\subfloat{\includegraphics[width=0.16\linewidth]{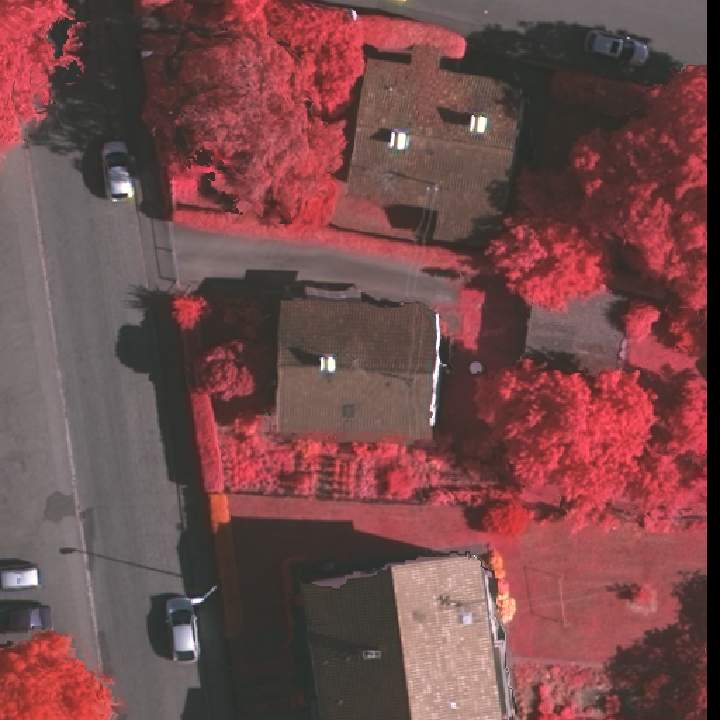}} \hfill
		\subfloat{\includegraphics[width=0.16\linewidth]{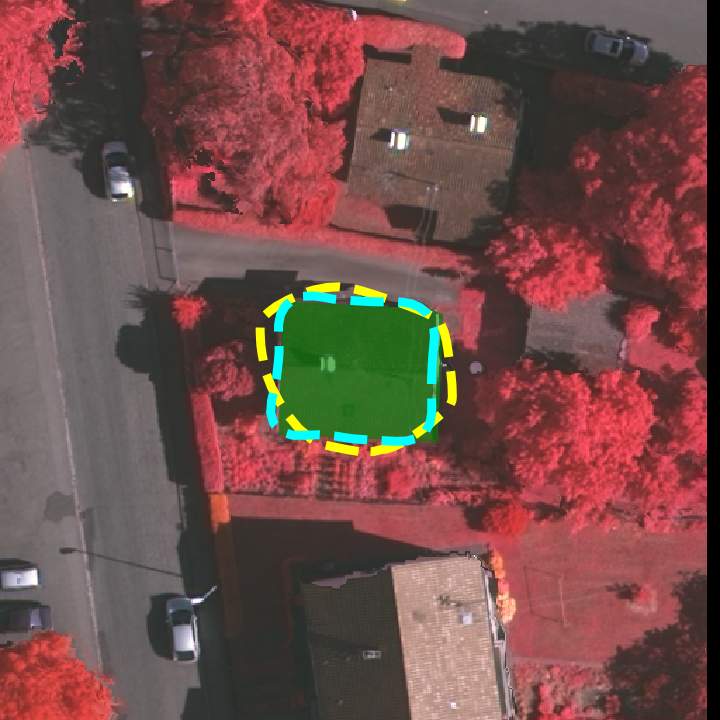}} \hfill
		\subfloat{\includegraphics[width=0.16\linewidth]{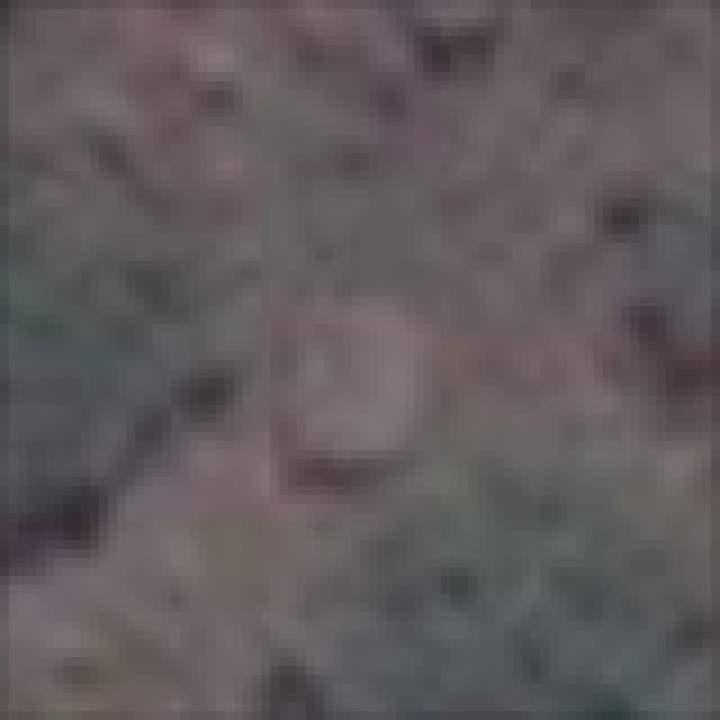}} \hfill
		\subfloat{\includegraphics[width=0.16\linewidth]{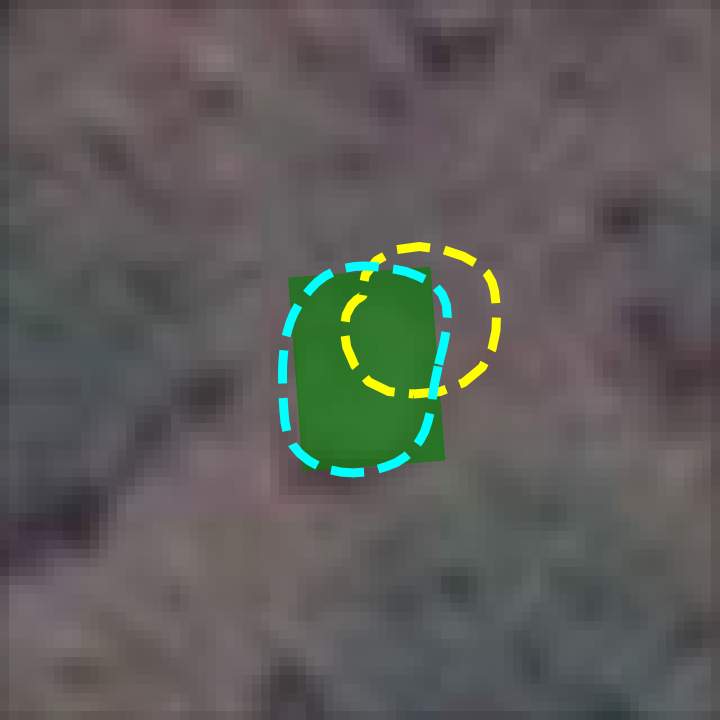} } \hfill
		\subfloat{\includegraphics[width=0.16\linewidth]{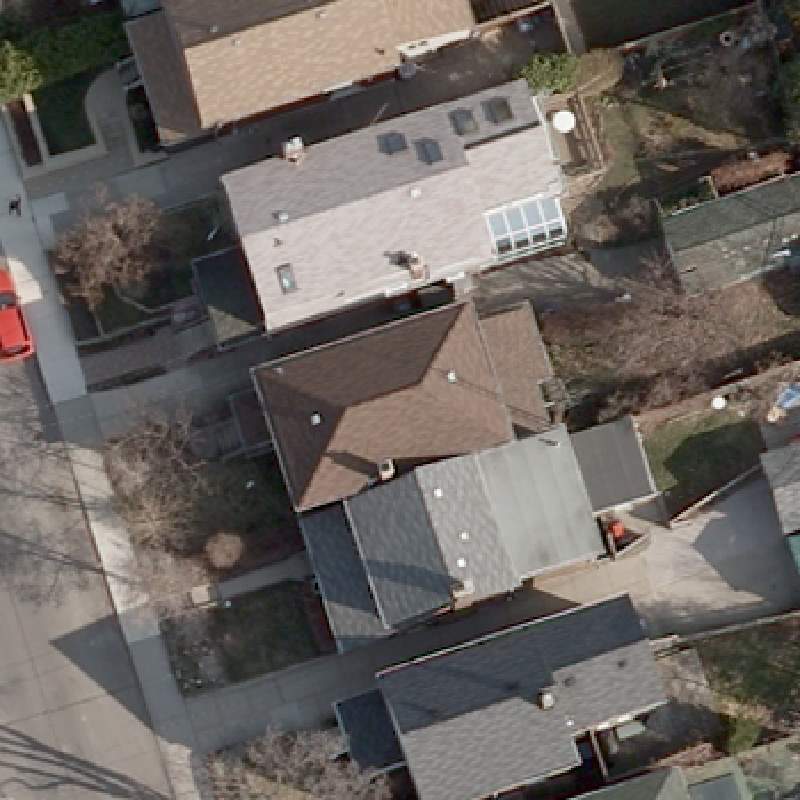}} \hfill
		\subfloat{\includegraphics[width=0.16\linewidth]{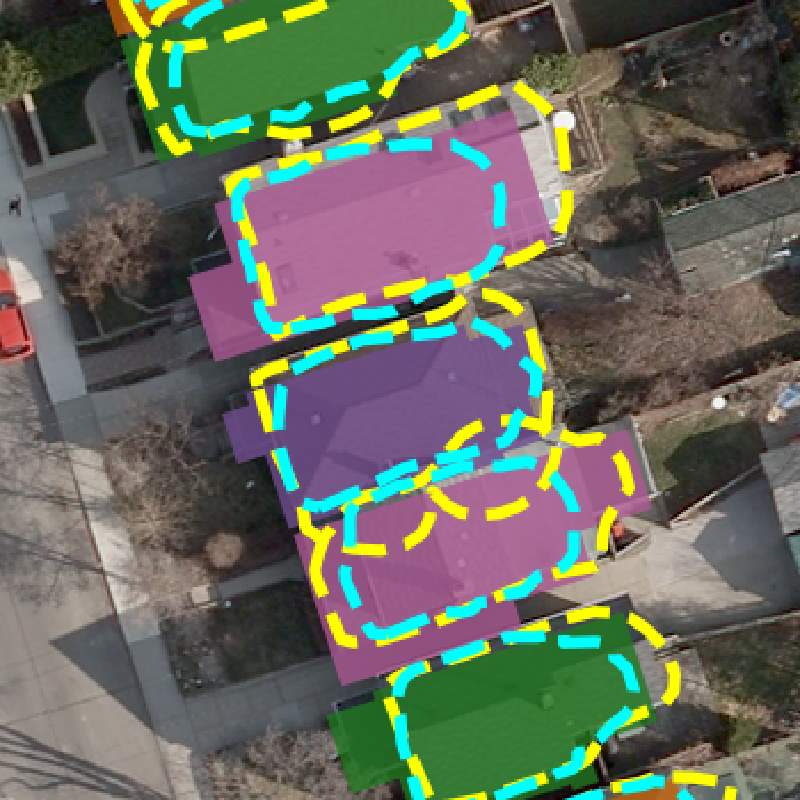} } \\ \vspace{-0.3cm}
		\subfloat{\includegraphics[width=0.16\linewidth]{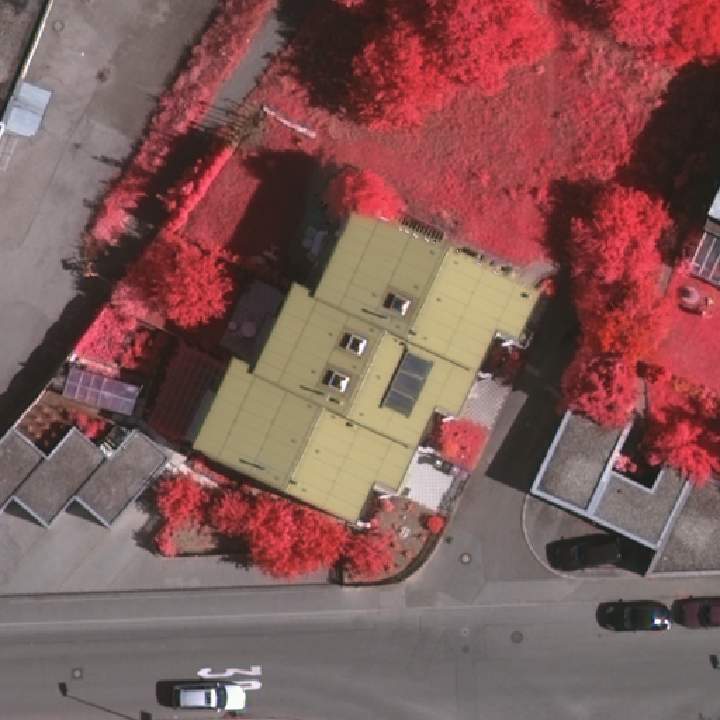}} \hfill
		\subfloat{\includegraphics[width=0.16\linewidth]{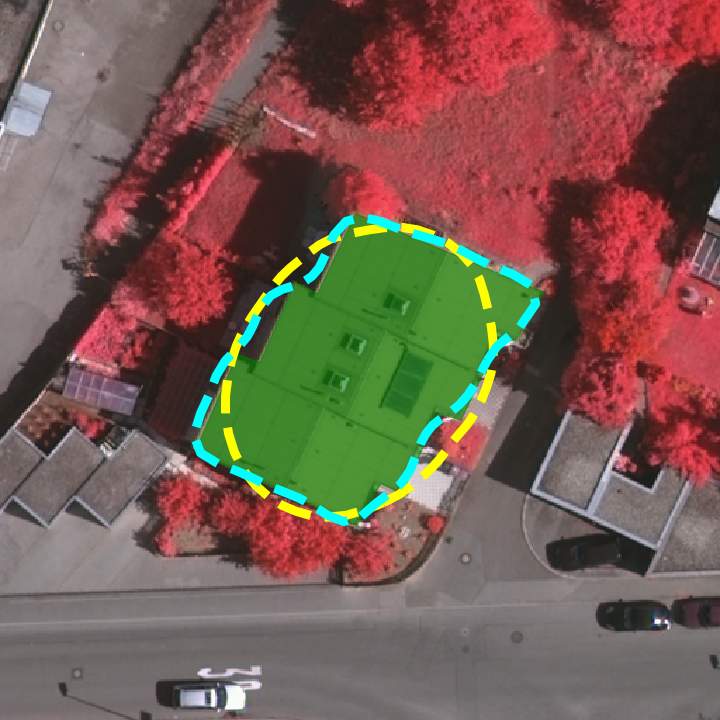}} \hfill
		\subfloat{\includegraphics[width=0.16\linewidth]{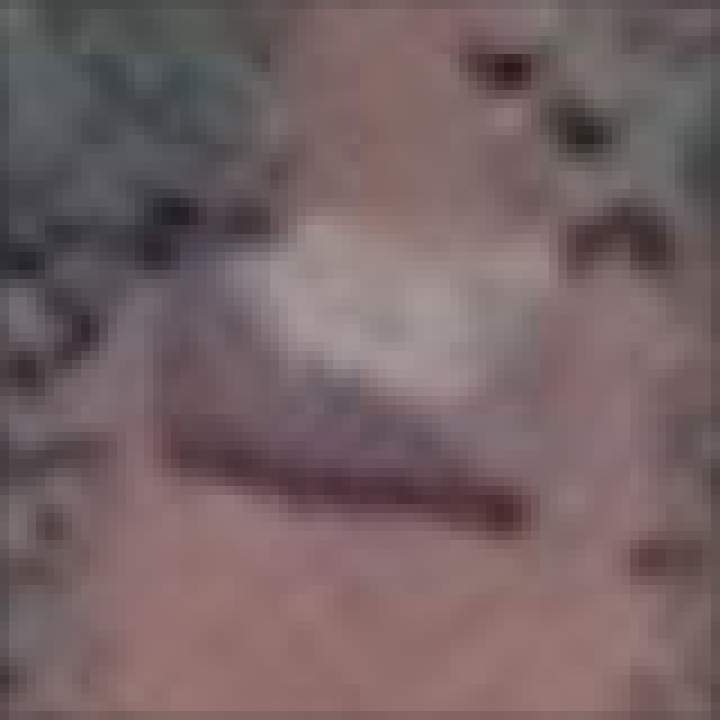}} \hfill
		\subfloat{\includegraphics[width=0.16\linewidth]{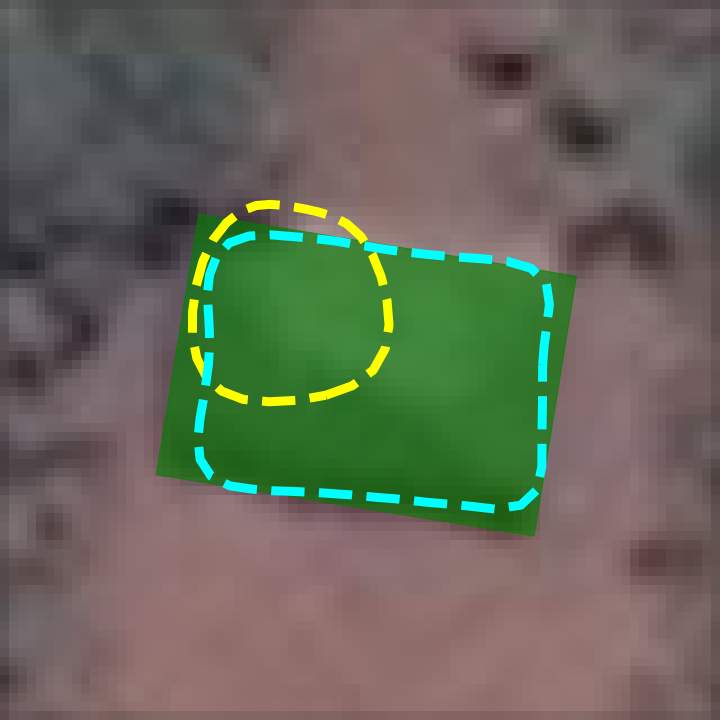} } \hfill
		\subfloat{\includegraphics[width=0.16\linewidth]{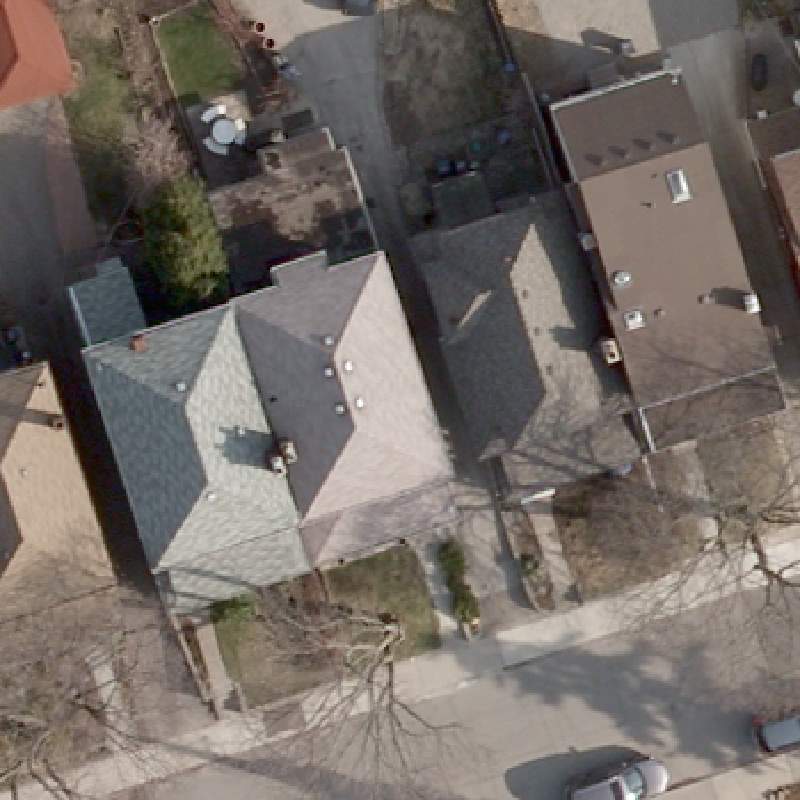}} \hfill
		\subfloat{\includegraphics[width=0.16\linewidth]{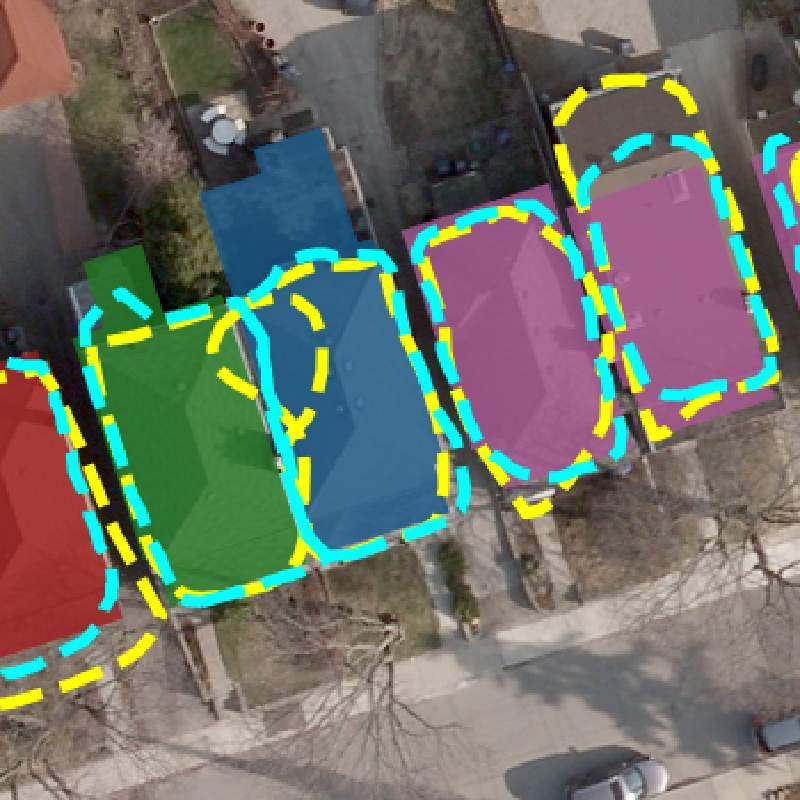} } \\ \vspace{-0.3cm}
		\subfloat{\includegraphics[width=0.16\linewidth]{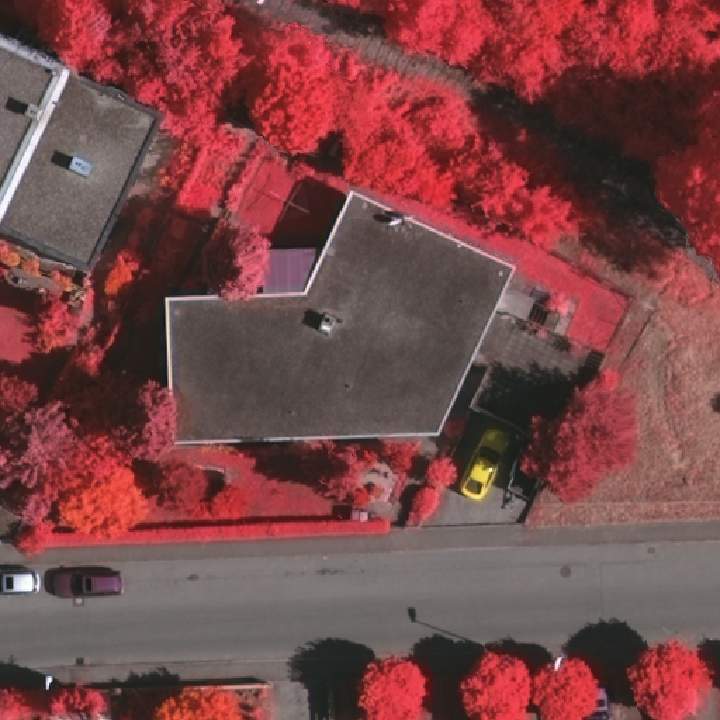}} \hfill
		\subfloat{\includegraphics[width=0.16\linewidth]{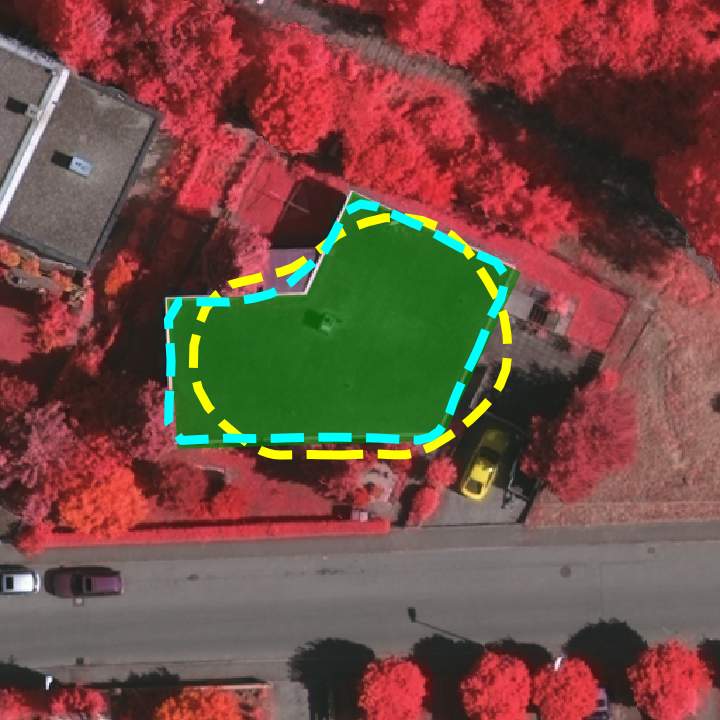}} \hfill
		\subfloat{\includegraphics[width=0.16\linewidth]{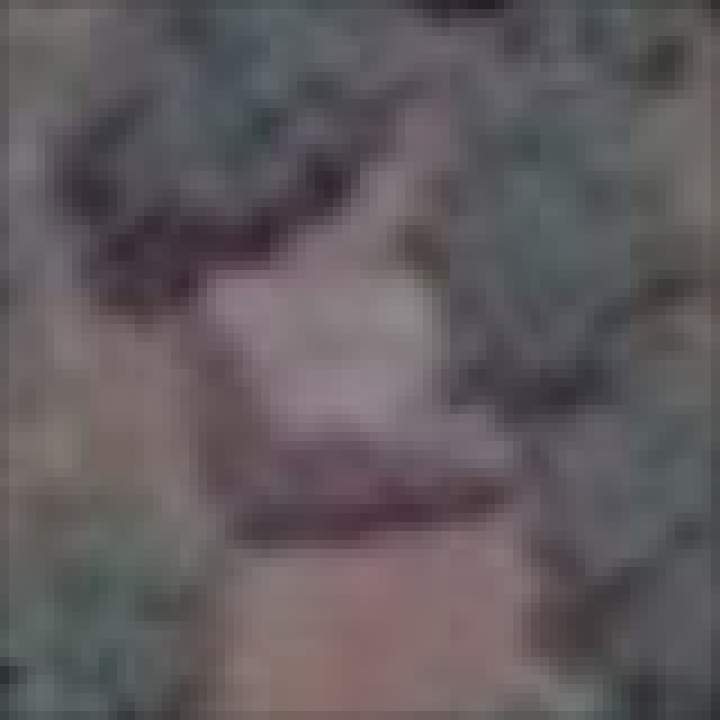}} \hfill
		\subfloat{\includegraphics[width=0.16\linewidth]{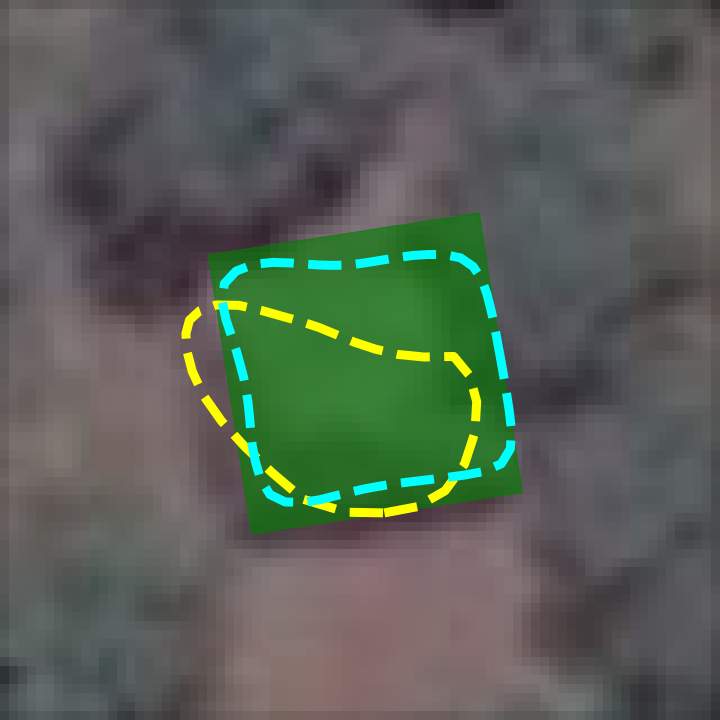} } \hfill
		\subfloat{\includegraphics[width=0.16\linewidth]{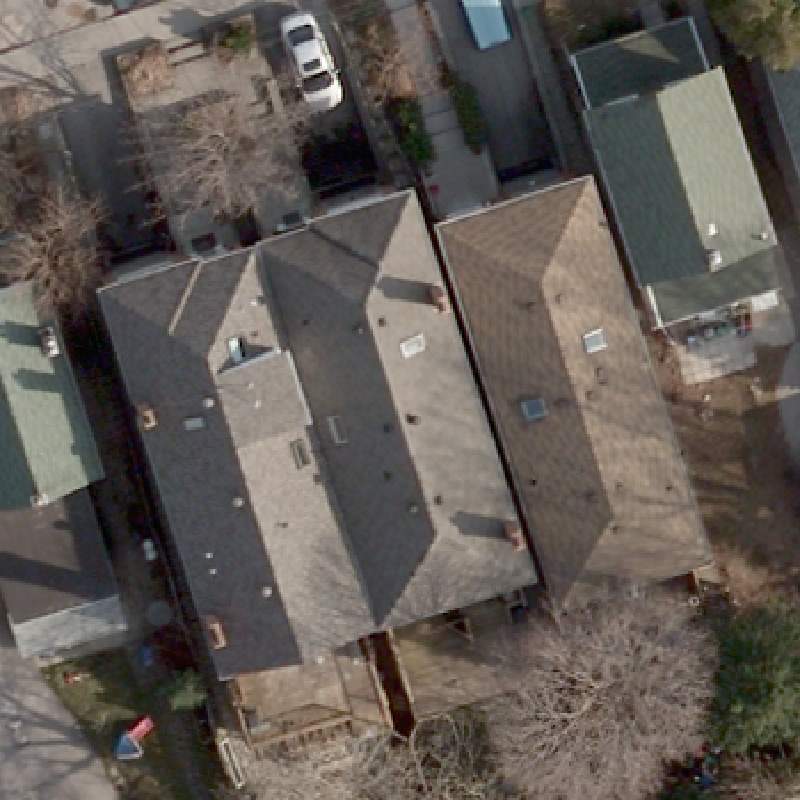}} \hfill
		\subfloat{\includegraphics[width=0.16\linewidth]{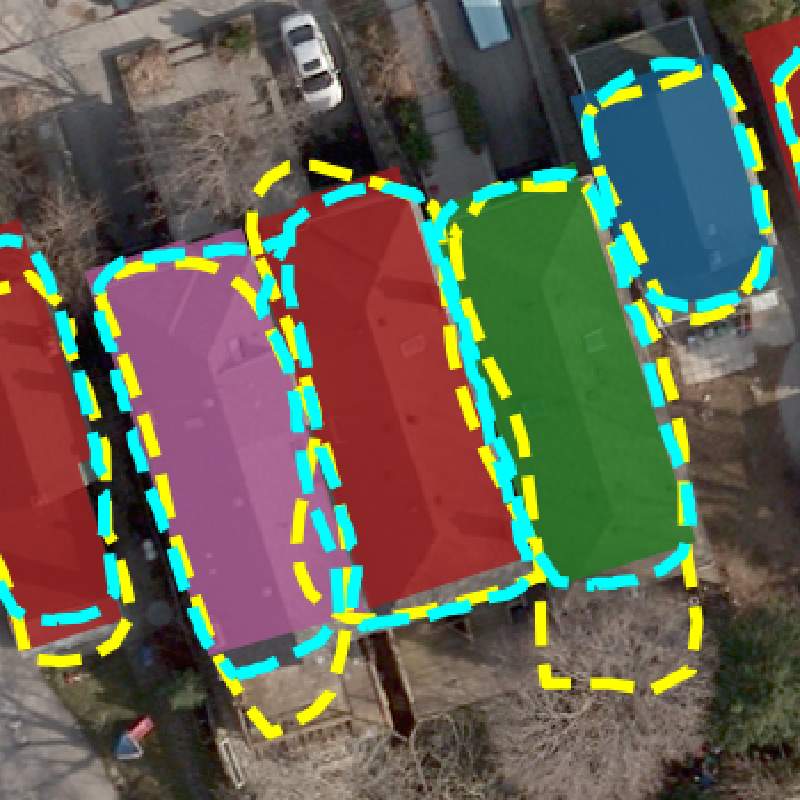} } \\ \vspace{-0.3cm}
		\subfloat{\includegraphics[width=0.16\linewidth]{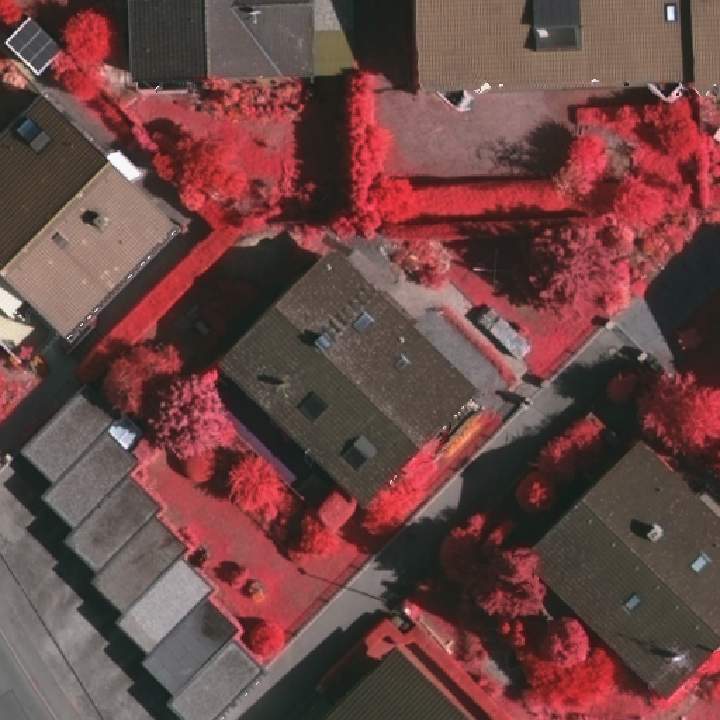}} \hfill
		\subfloat{\includegraphics[width=0.16\linewidth]{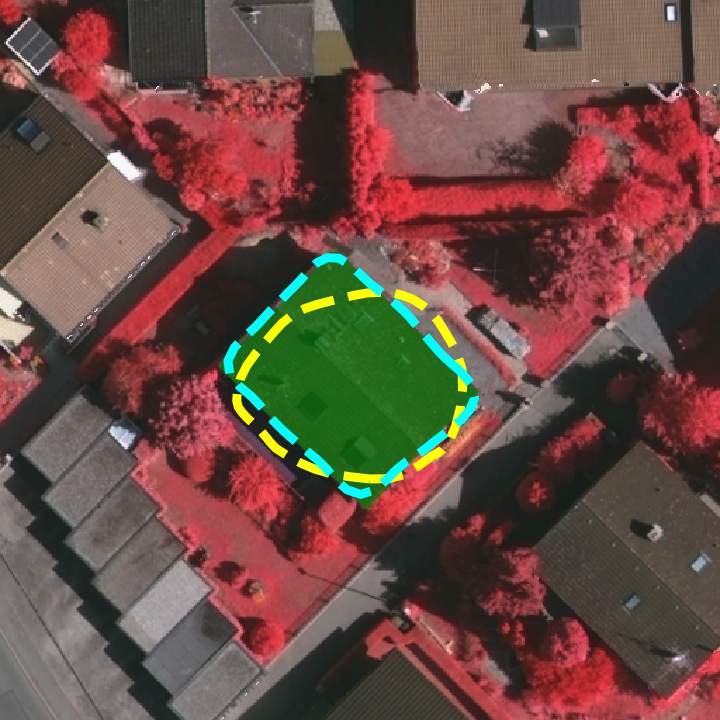}} \hfill
		\subfloat{\includegraphics[width=0.16\linewidth]{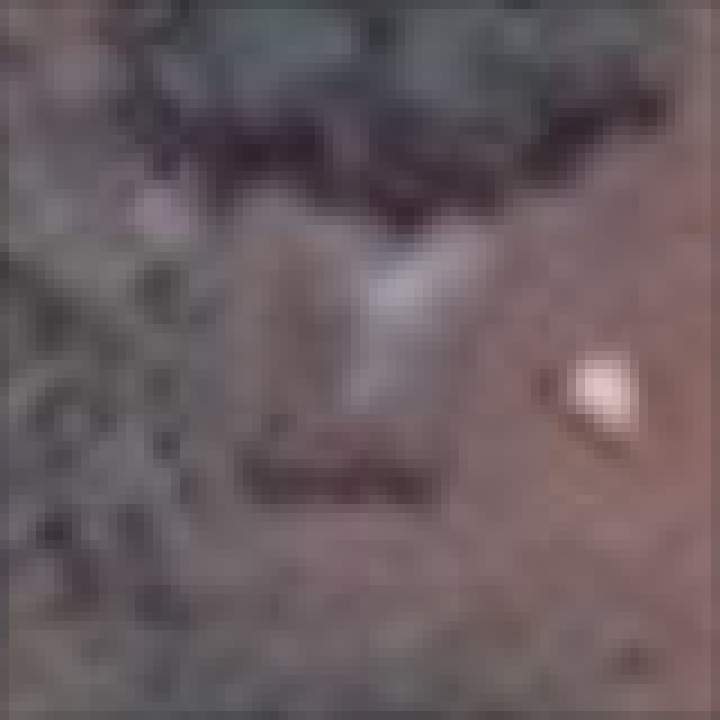}} \hfill
		\subfloat{\includegraphics[width=0.16\linewidth]{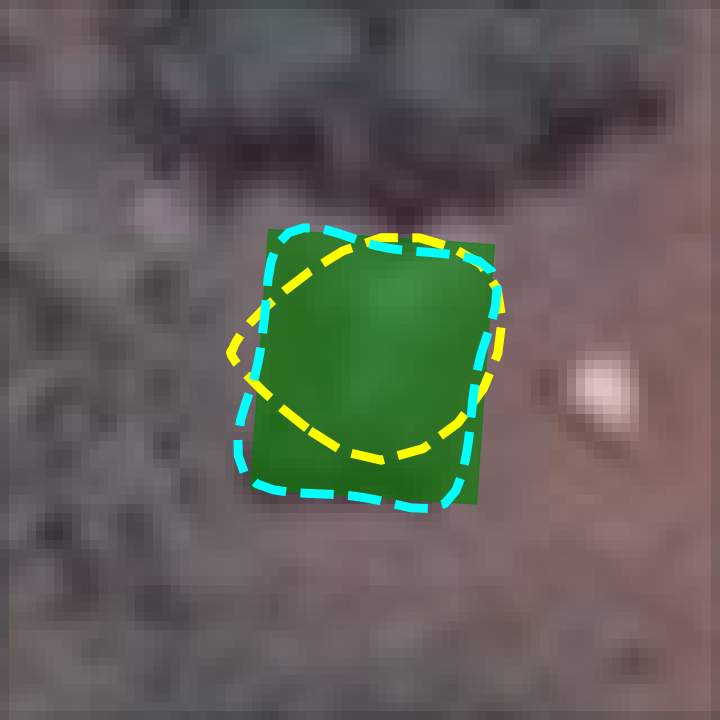} } \hfill
		\subfloat{\includegraphics[width=0.16\linewidth]{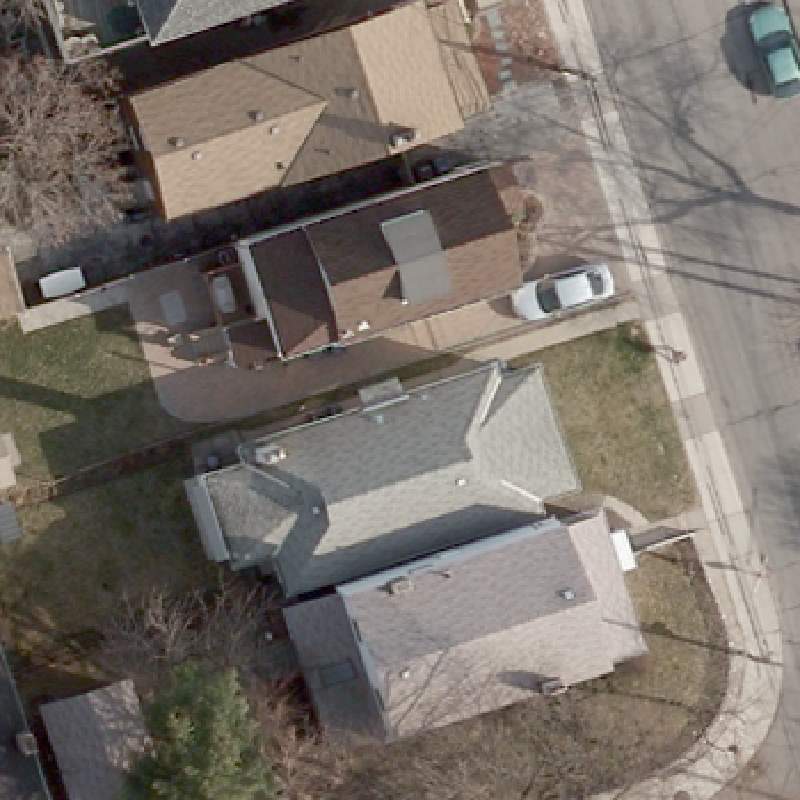}} \hfill
		\subfloat{\includegraphics[width=0.16\linewidth]{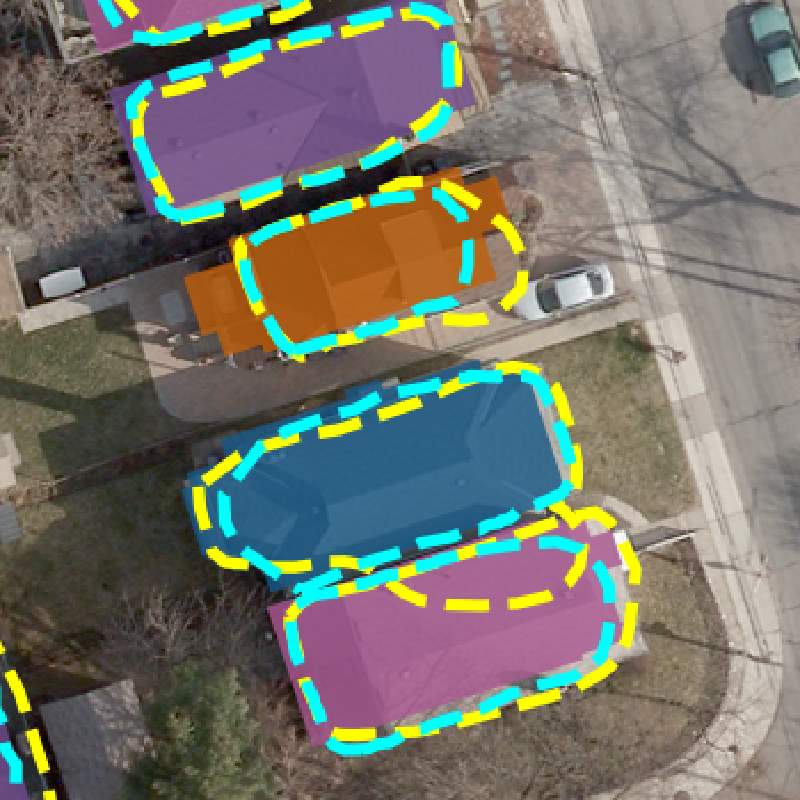} } \\ \vspace{-0.3cm}
		\subfloat{\includegraphics[width=0.16\linewidth]{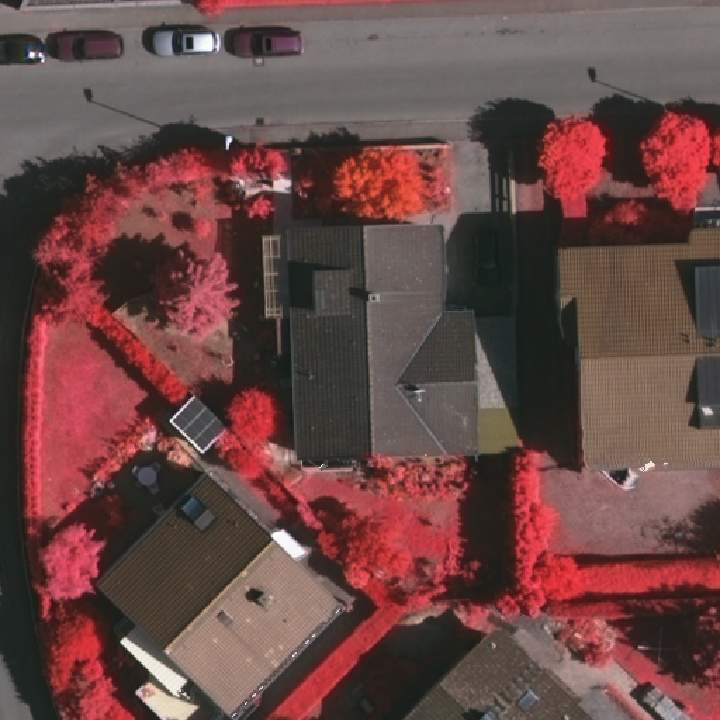}} \hfill
		\subfloat{\includegraphics[width=0.16\linewidth]{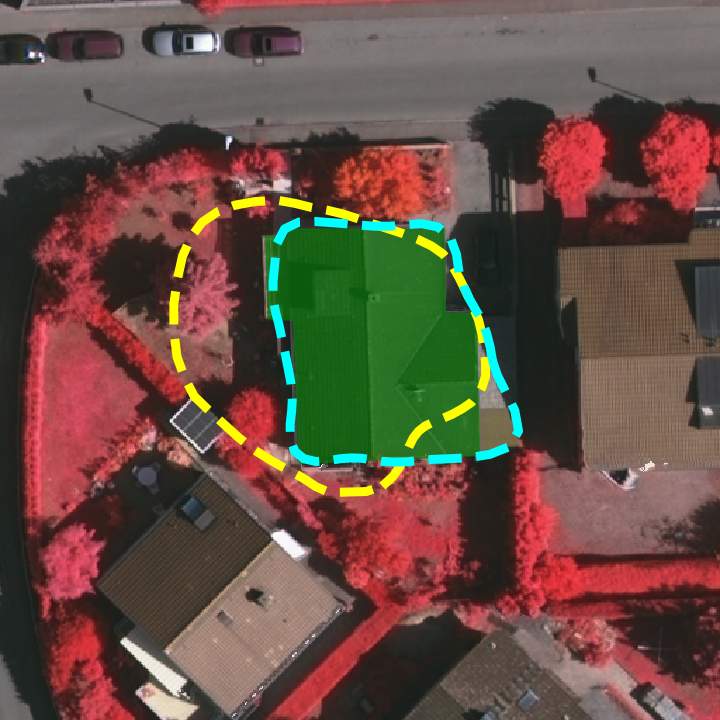}} \hfill
		\subfloat{\includegraphics[width=0.16\linewidth]{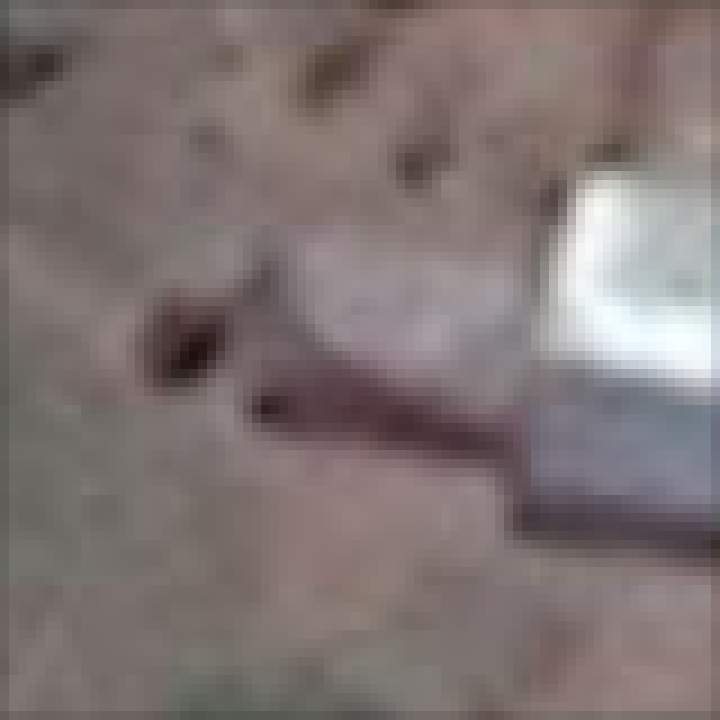}} \hfill
		\subfloat{\includegraphics[width=0.16\linewidth]{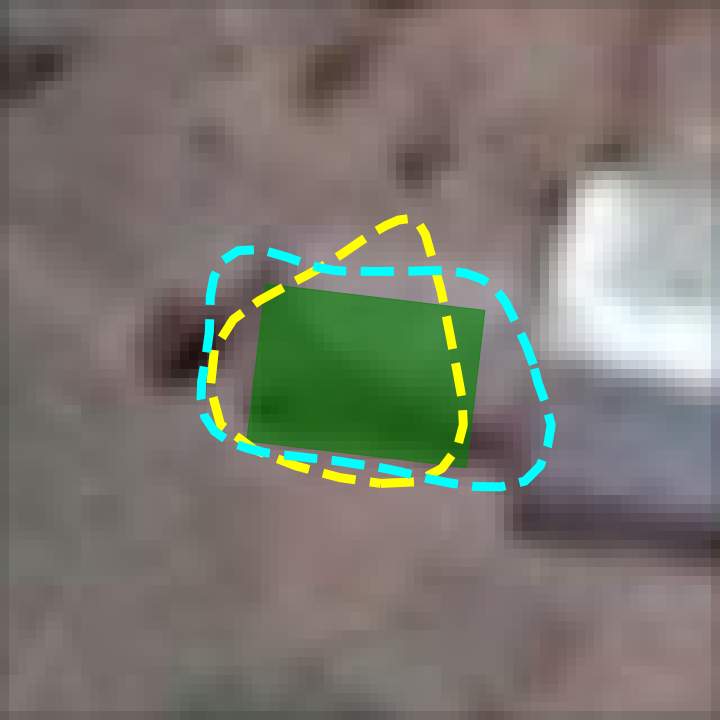} } \hfill
		\subfloat{\includegraphics[width=0.16\linewidth]{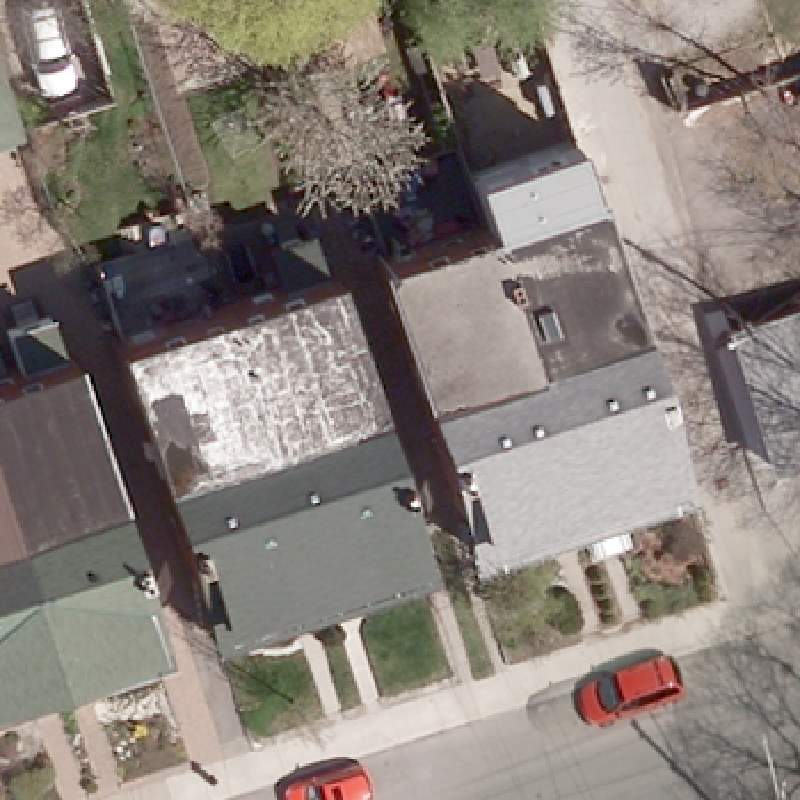}} \hfill
		\subfloat{\includegraphics[width=0.16\linewidth]{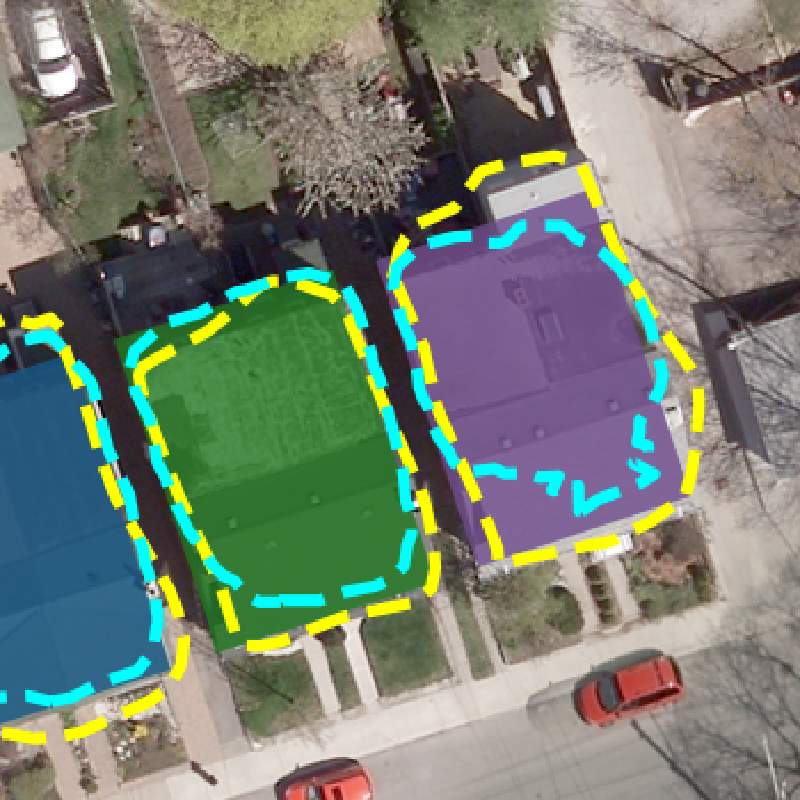} } \\ \vspace{-0.3cm}
		\subfloat{\includegraphics[width=0.16\linewidth]{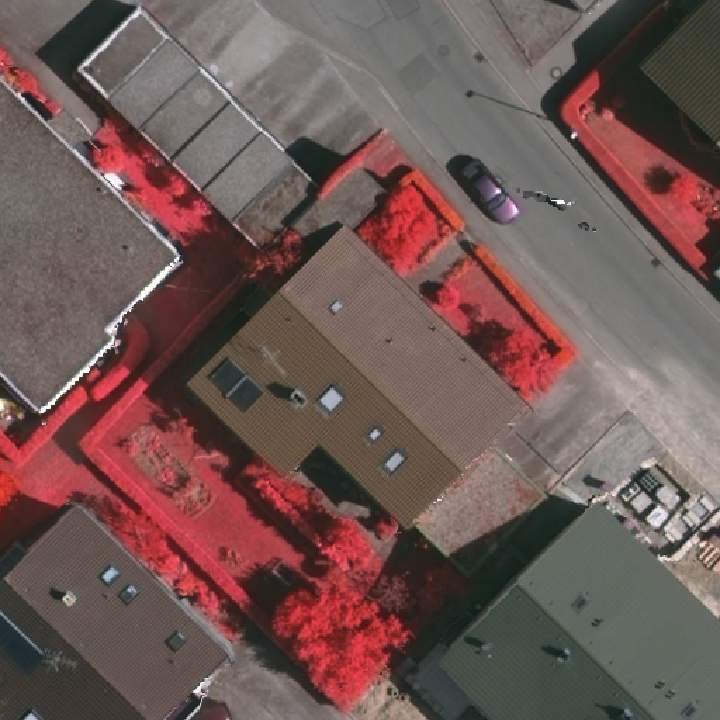}} \hfill
		\subfloat{\includegraphics[width=0.16\linewidth]{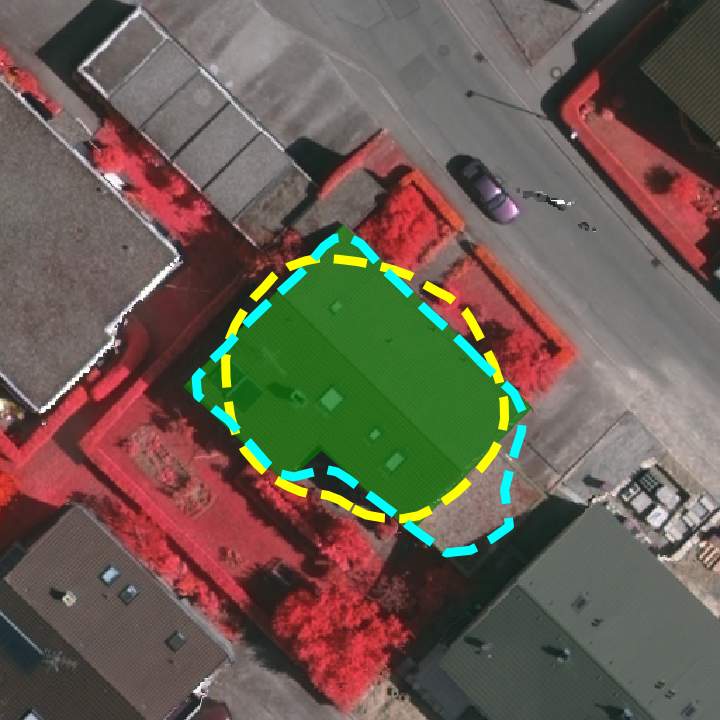}} \hfill
		\subfloat{\includegraphics[width=0.16\linewidth]{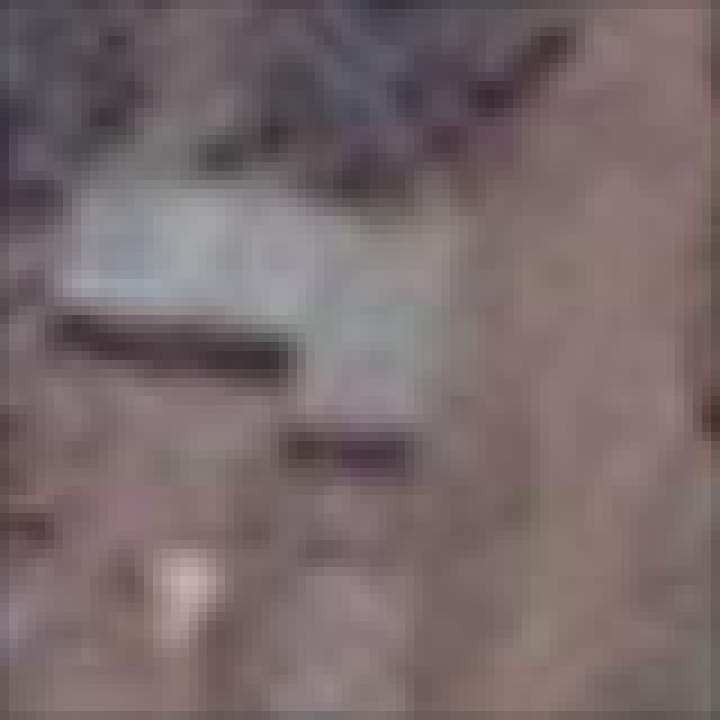}} \hfill
		\subfloat{\includegraphics[width=0.16\linewidth]{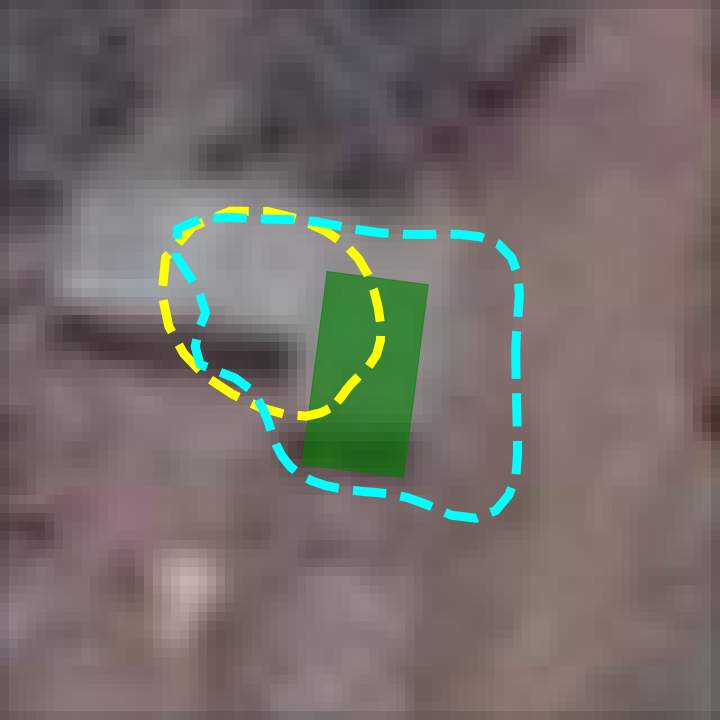} } \hfill
		\subfloat{\includegraphics[width=0.16\linewidth]{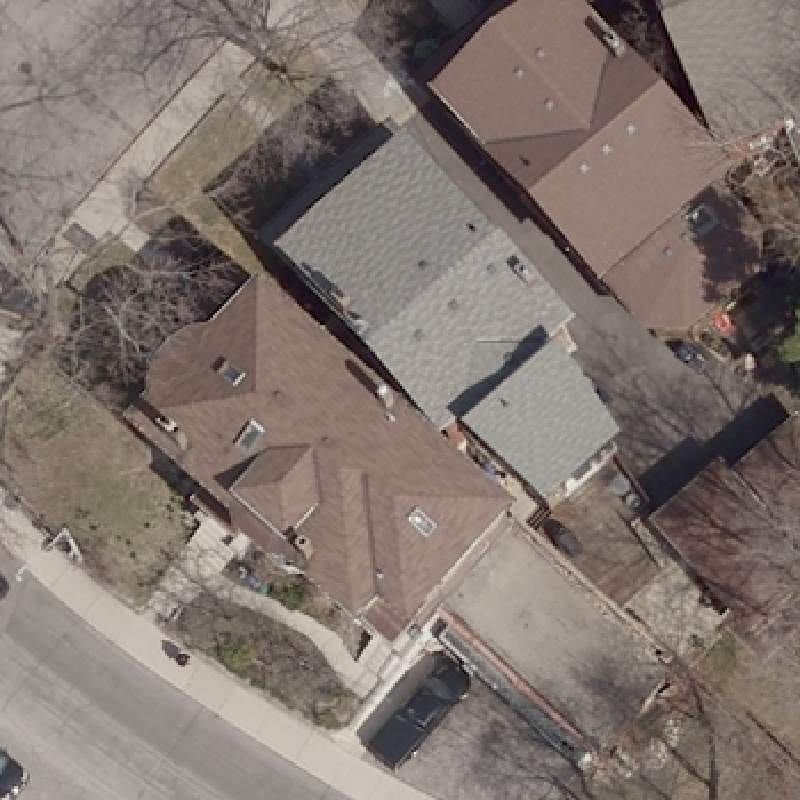}} \hfill
		\subfloat{\includegraphics[width=0.16\linewidth]{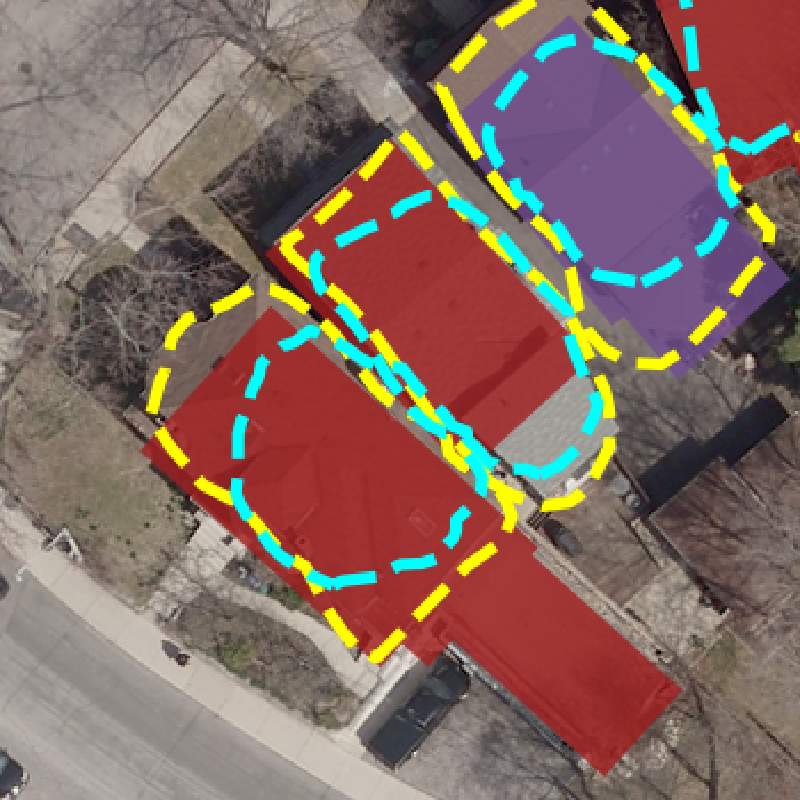} } \\ \vspace{-0.3cm}
		\subfloat[(a)]{\includegraphics[width=0.16\linewidth]{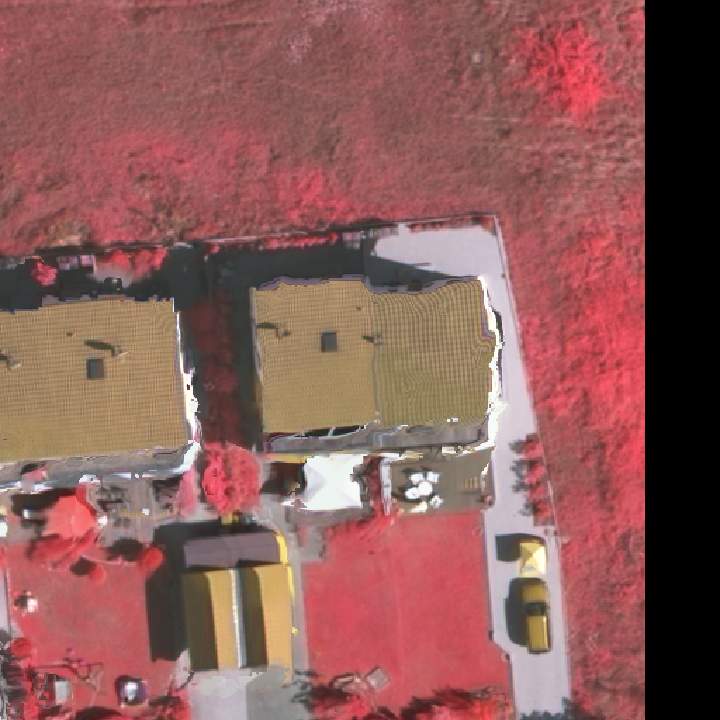}} \hfill
		\subfloat[(b)]{\includegraphics[width=0.16\linewidth]{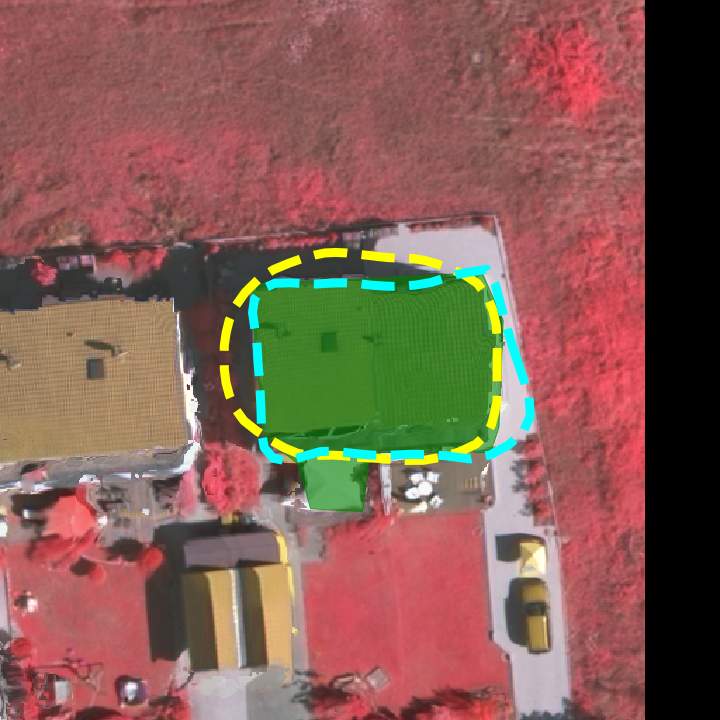}} \hfill
		\subfloat[(c)]{\includegraphics[width=0.16\linewidth]{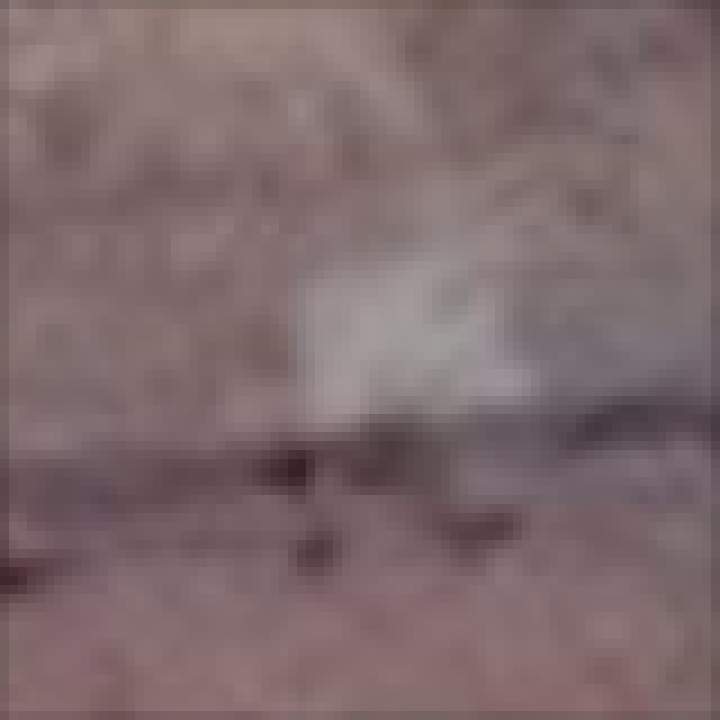}} \hfill
		\subfloat[(d)]{\includegraphics[width=0.16\linewidth]{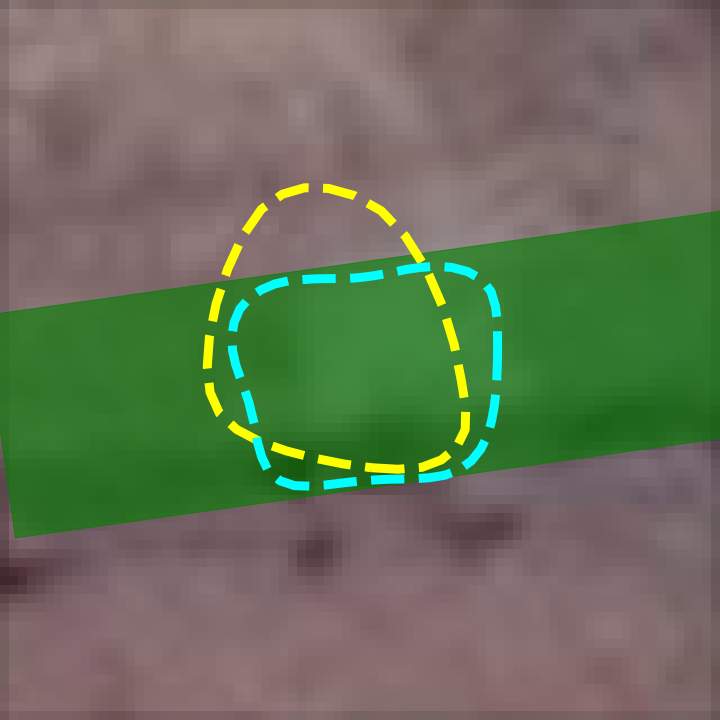} } \hfill
		\subfloat[(e)]{\includegraphics[width=0.16\linewidth]{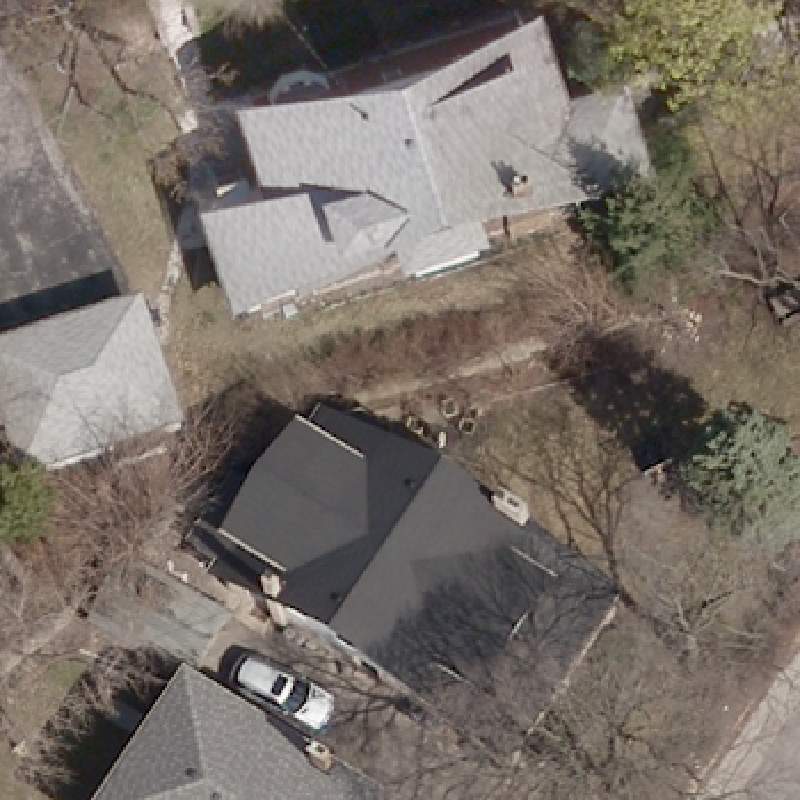}} \hfill
		\subfloat[(f)]{\includegraphics[width=0.16\linewidth]{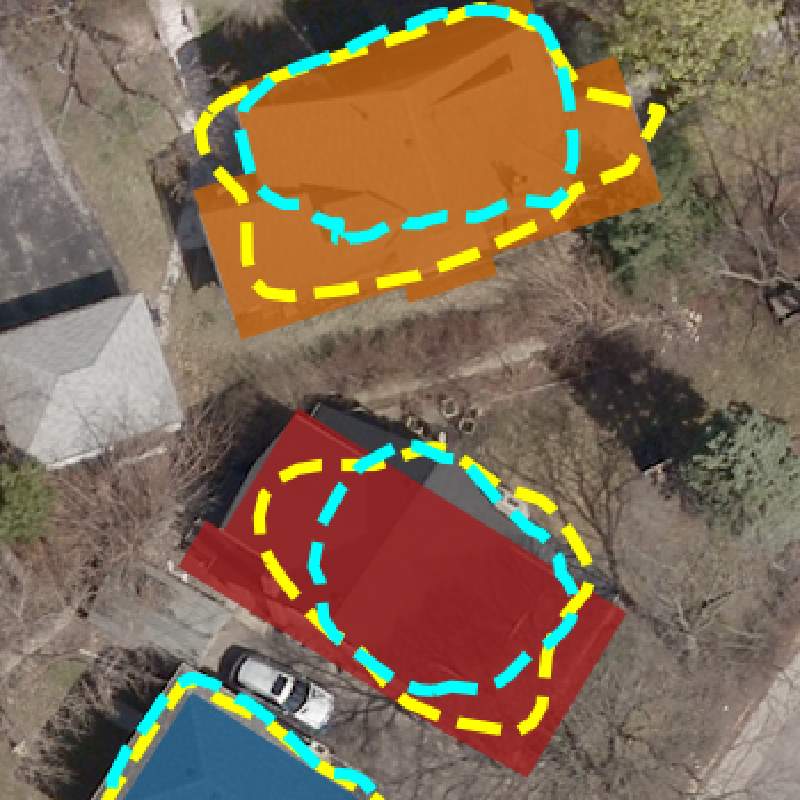}
		}
		\caption{Results on (a-b) Vaihingen, (c-d) Bing Huts, (e-f) TorontoCity. Bottom three rows highlight failure cases. Original image shown in left. On right, our output is shown in cyan; DSAC output in yellow; ground truth is shaded.}
		\label{fig:qual_results_2}
	\end{figure*}
}

{
	\captionsetup[subfigure]{labelformat=empty}
	\begin{figure*}
		\centering
		\subfloat{\includegraphics[width=0.16\linewidth]{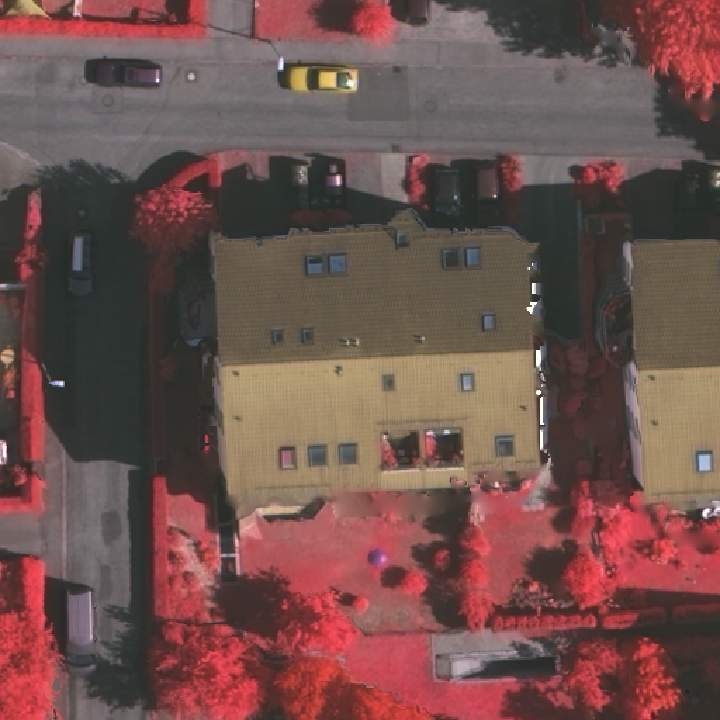}} \hfill
		\subfloat{\includegraphics[width=0.16\linewidth]{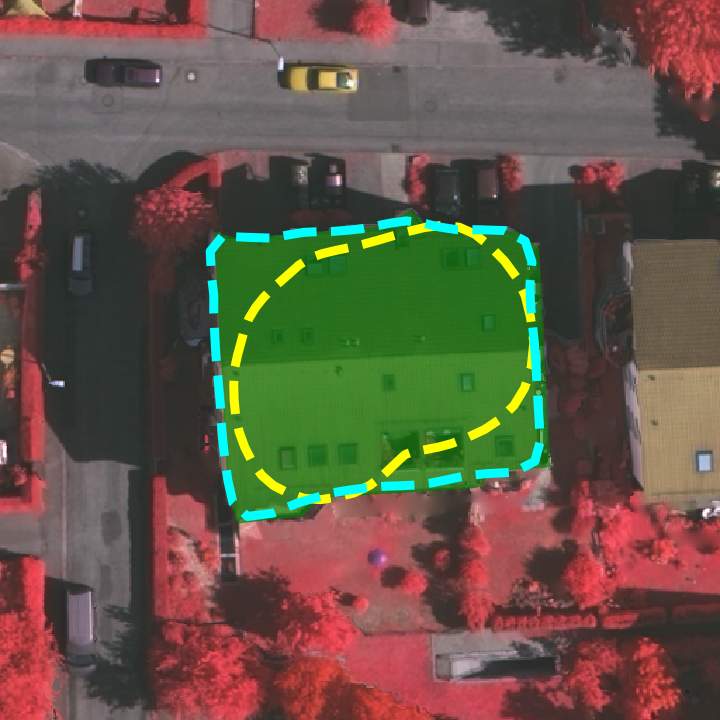}} \hfill
		\subfloat{\includegraphics[width=0.16\linewidth]{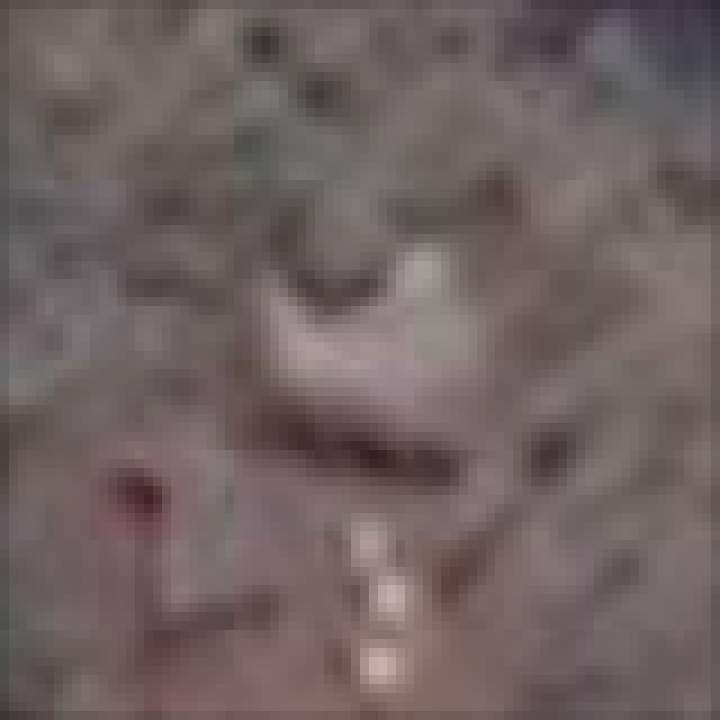}} \hfill
		\subfloat{\includegraphics[width=0.16\linewidth]{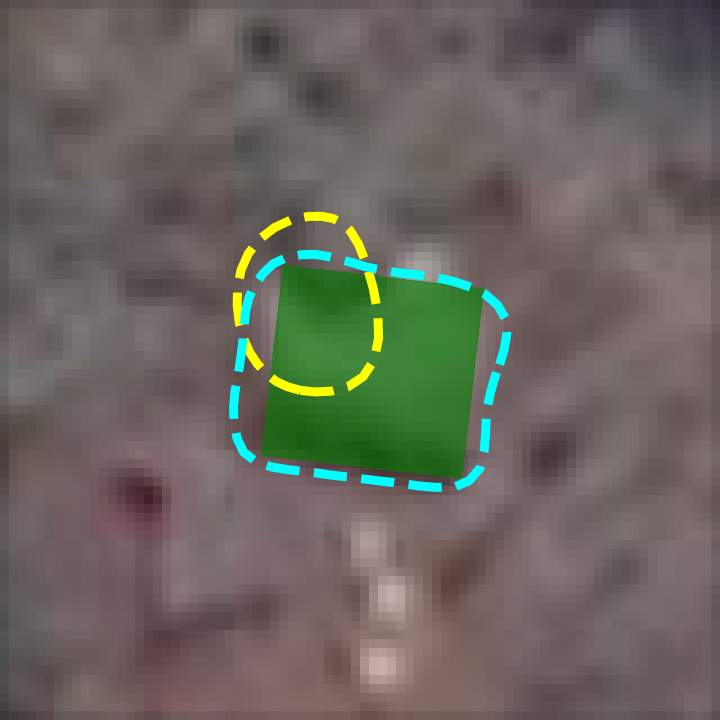} } \hfill
		\subfloat{\includegraphics[width=0.16\linewidth]{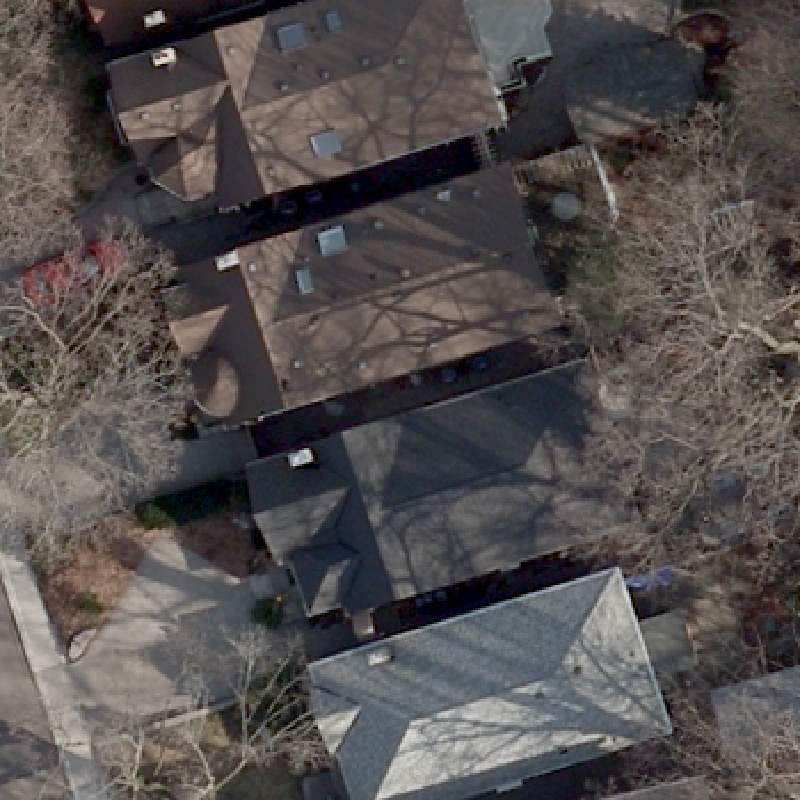}} \hfill
		\subfloat{\includegraphics[width=0.16\linewidth]{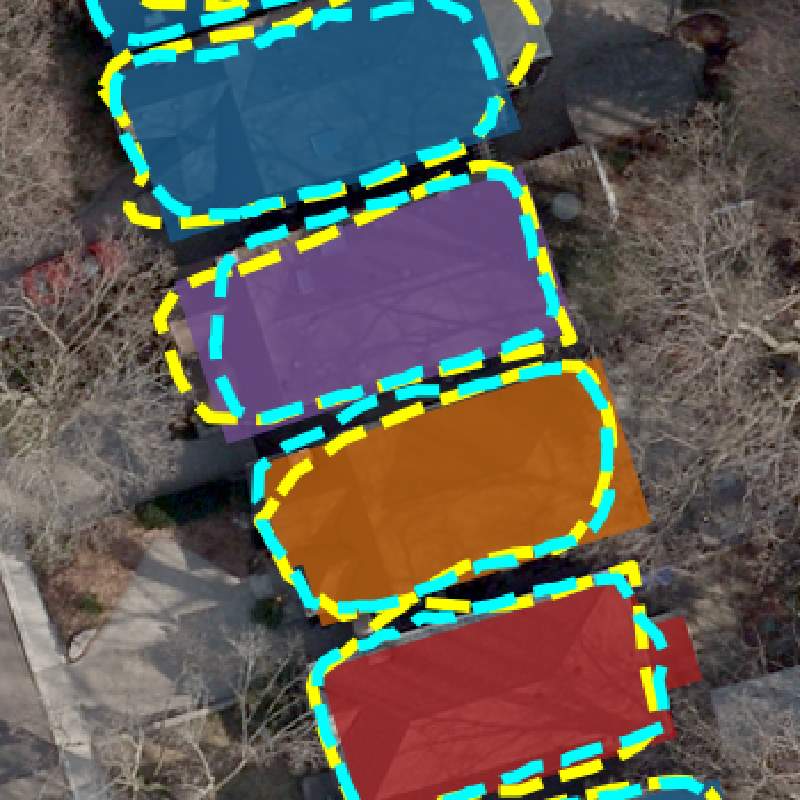} } \\ \vspace{-0.3cm}
		\subfloat{\includegraphics[width=0.16\linewidth]{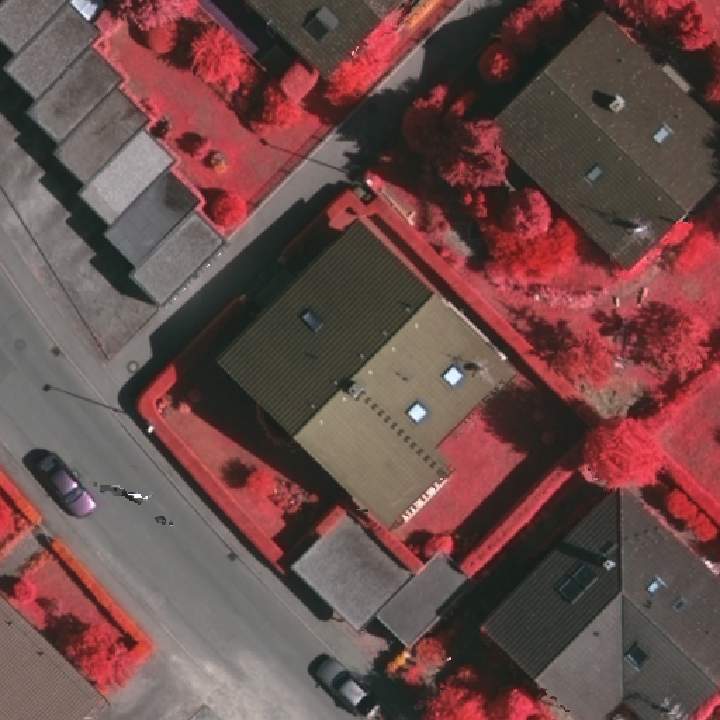}} \hfill
		\subfloat{\includegraphics[width=0.16\linewidth]{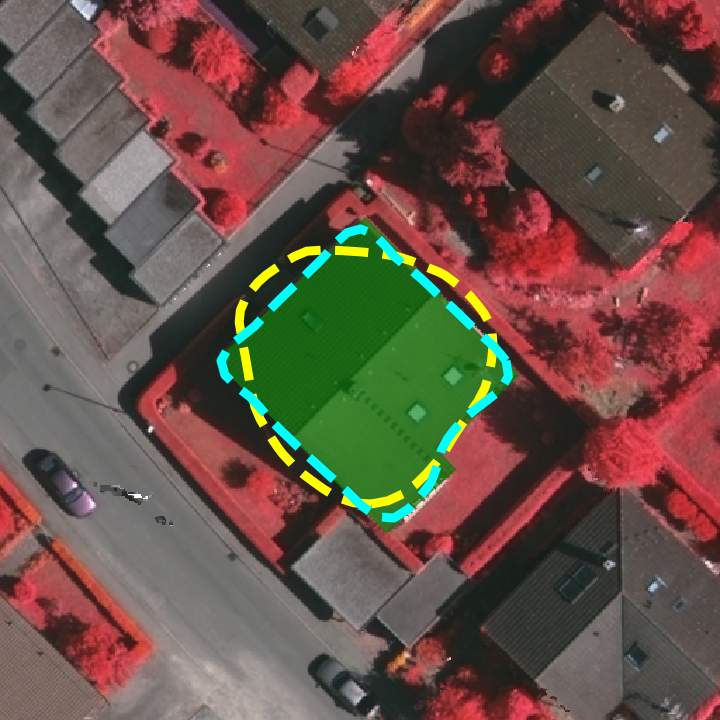}} \hfill
		\subfloat{\includegraphics[width=0.16\linewidth]{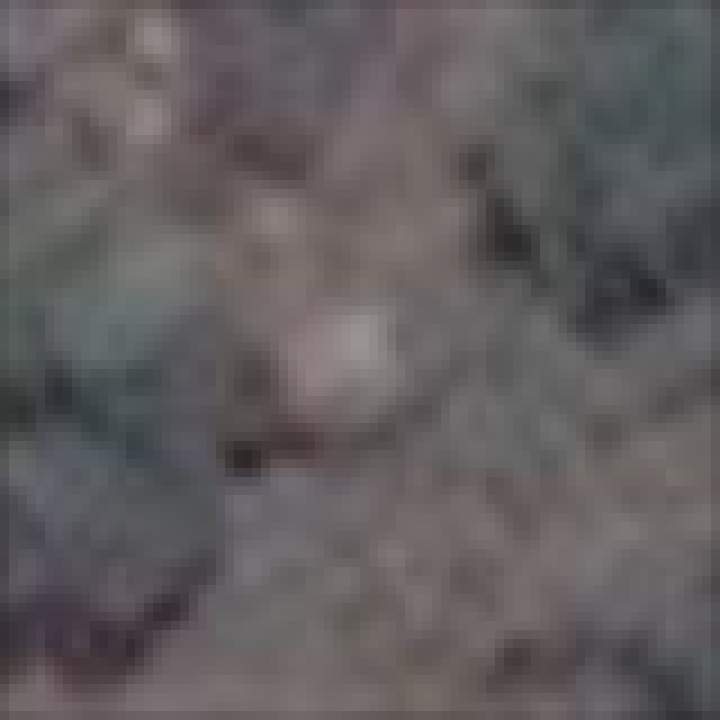}} \hfill
		\subfloat{\includegraphics[width=0.16\linewidth]{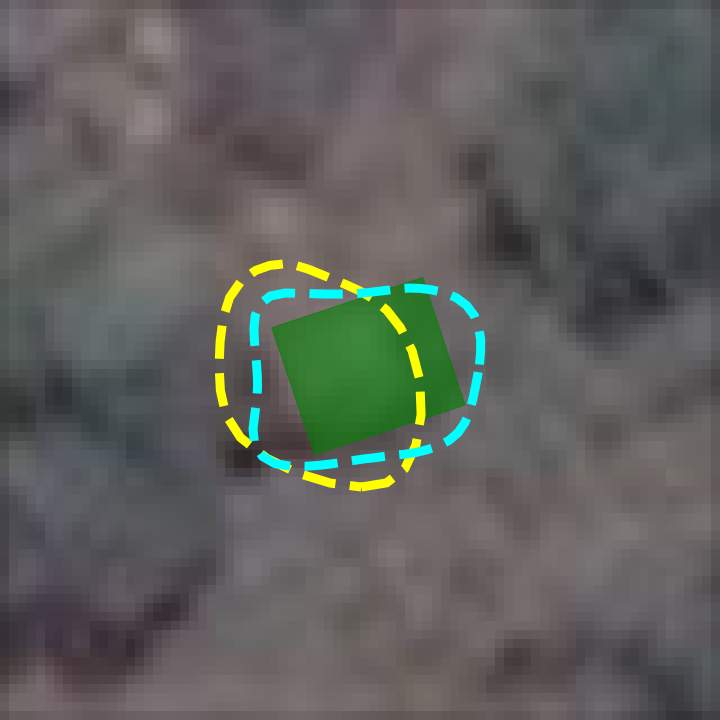} } \hfill
		\subfloat{\includegraphics[width=0.16\linewidth]{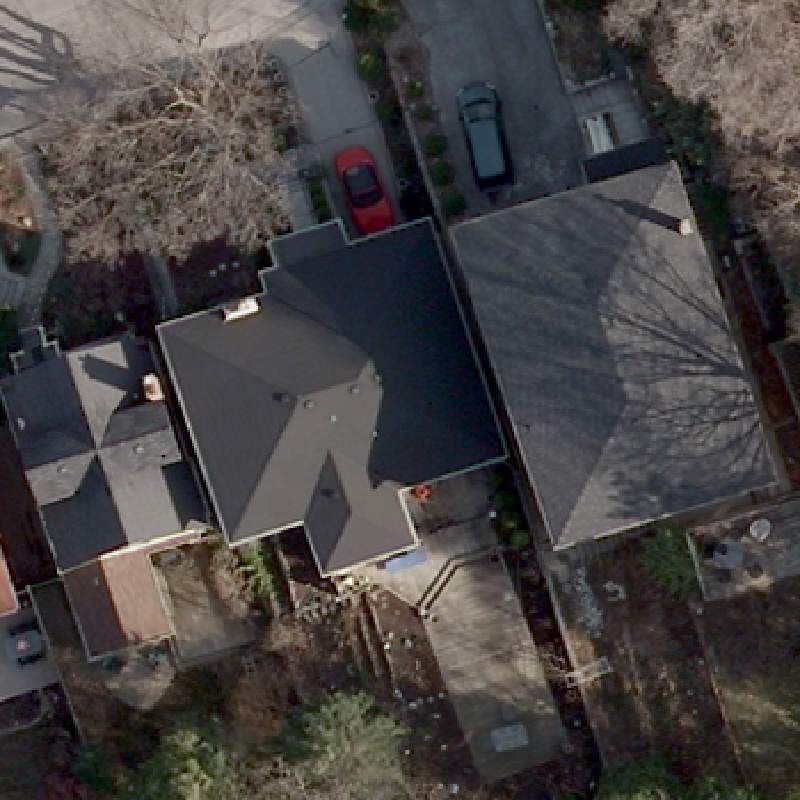}} \hfill
		\subfloat{\includegraphics[width=0.16\linewidth]{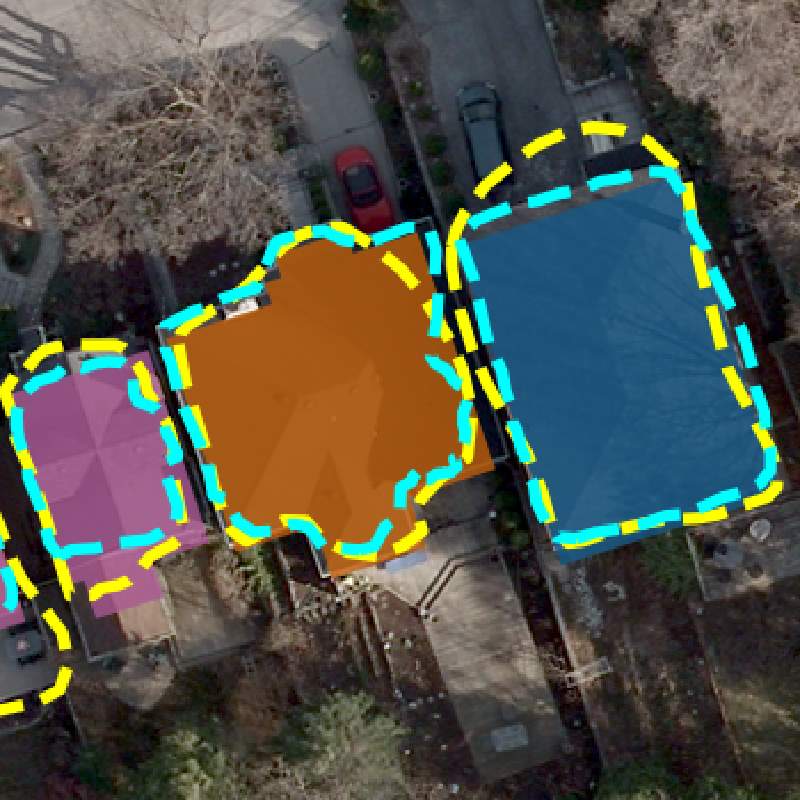} } \\ \vspace{-0.3cm}
		\subfloat{\includegraphics[width=0.16\linewidth]{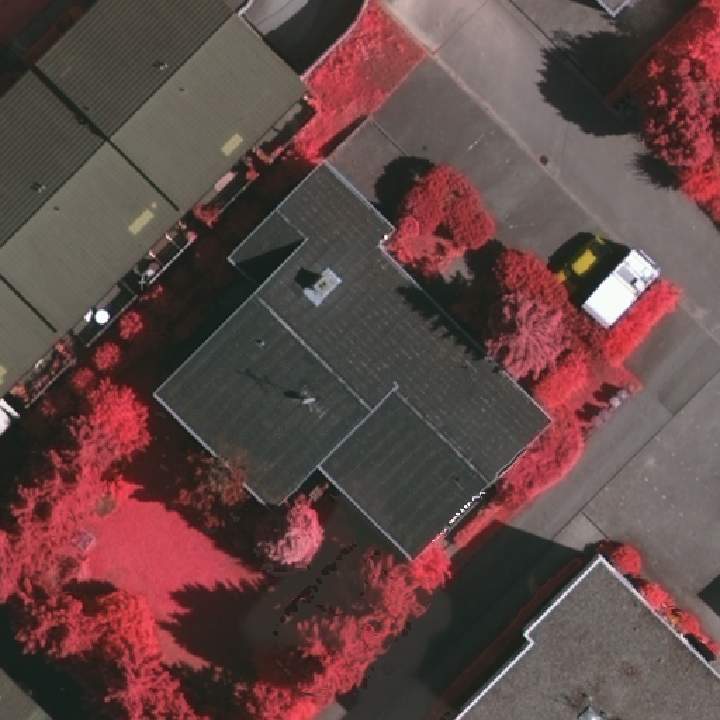}} \hfill
		\subfloat{\includegraphics[width=0.16\linewidth]{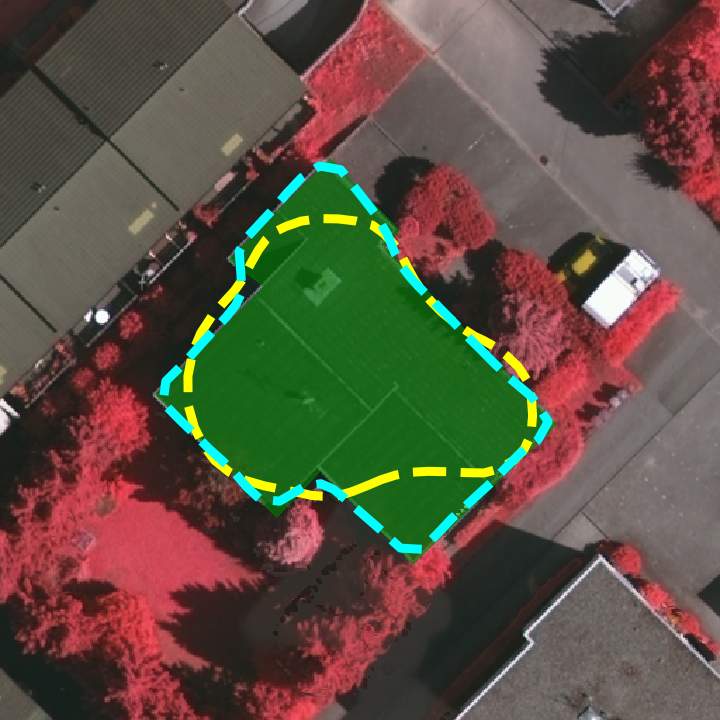}} \hfill
		\subfloat{\includegraphics[width=0.16\linewidth]{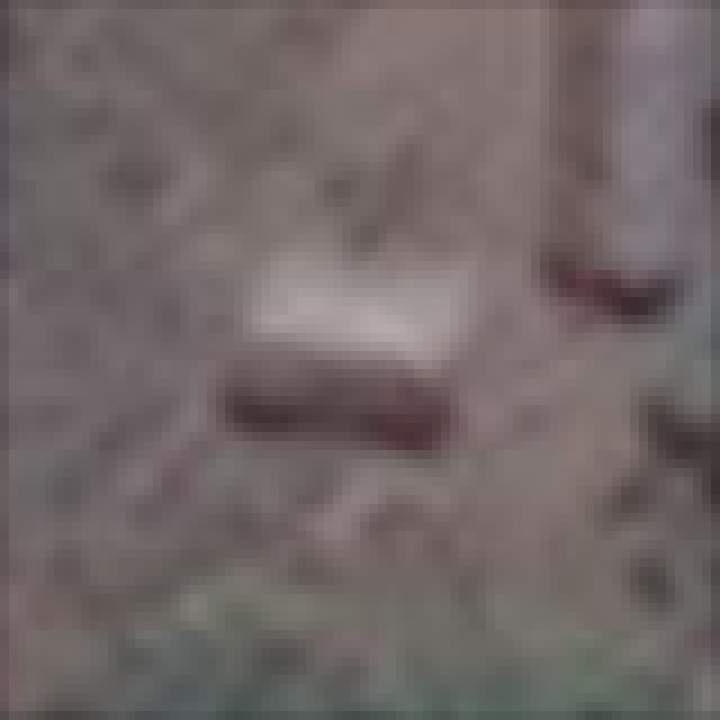}} \hfill
		\subfloat{\includegraphics[width=0.16\linewidth]{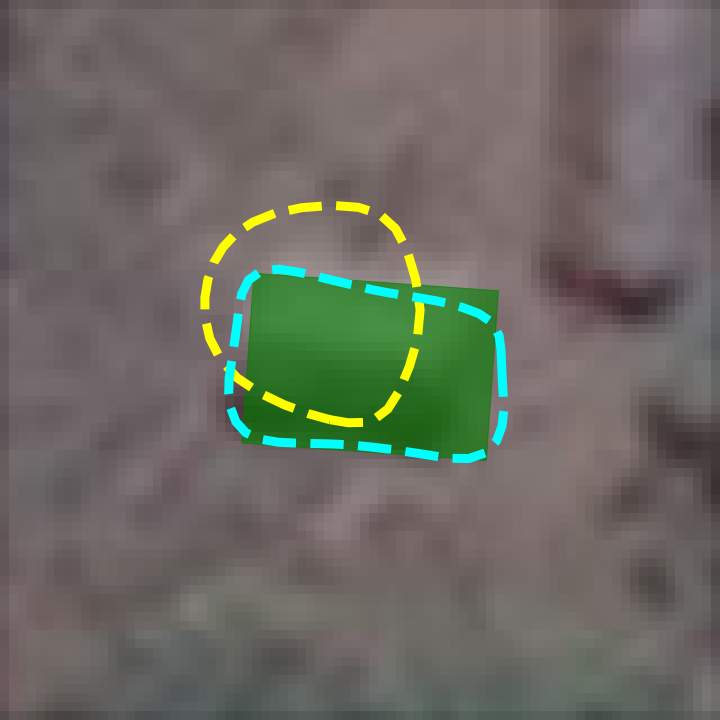} } \hfill
		\subfloat{\includegraphics[width=0.16\linewidth]{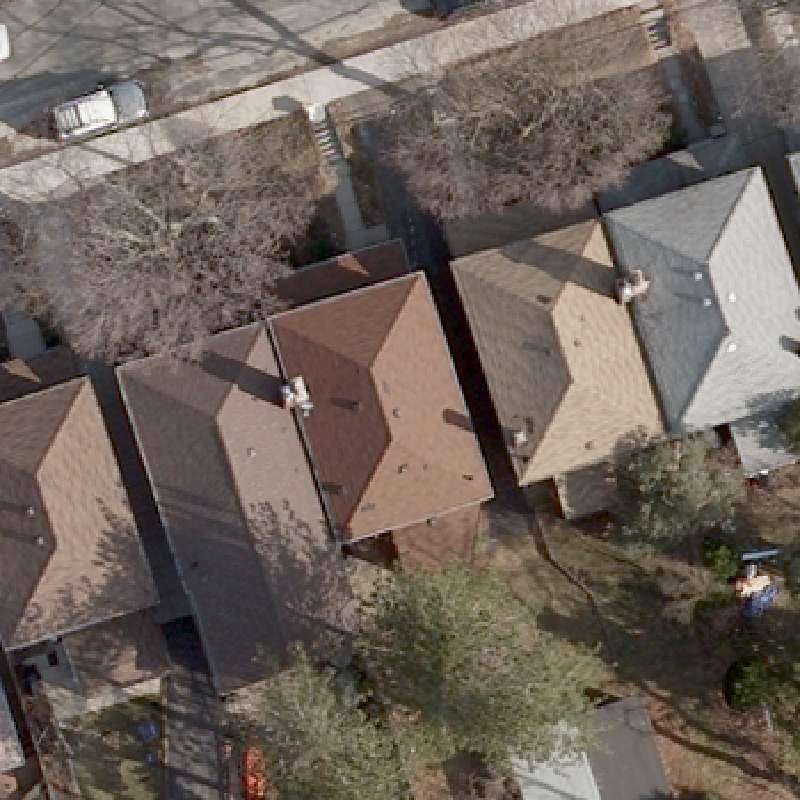}} \hfill
		\subfloat{\includegraphics[width=0.16\linewidth]{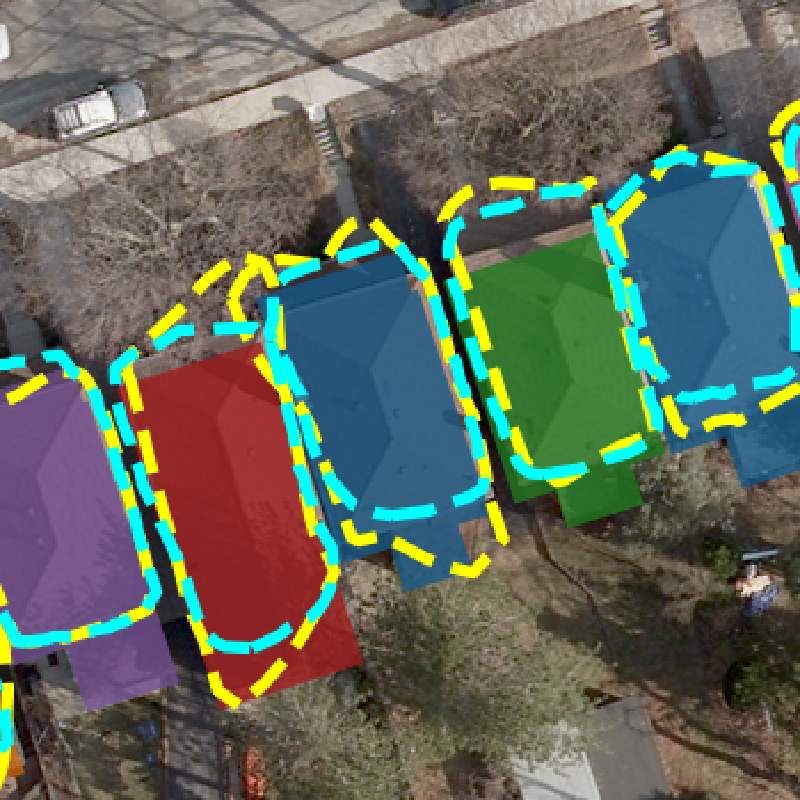} } \\ \vspace{-0.3cm}
		\subfloat{\includegraphics[width=0.16\linewidth]{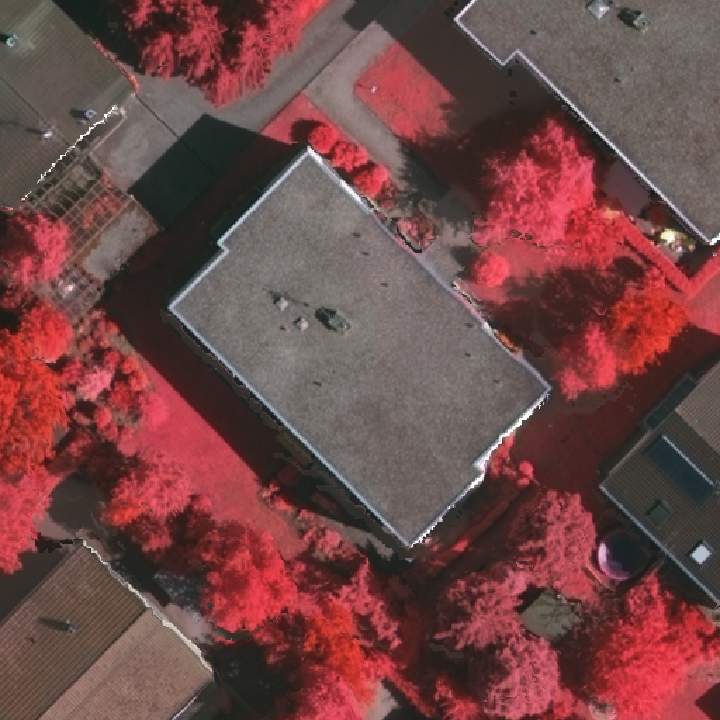}} \hfill
		\subfloat{\includegraphics[width=0.16\linewidth]{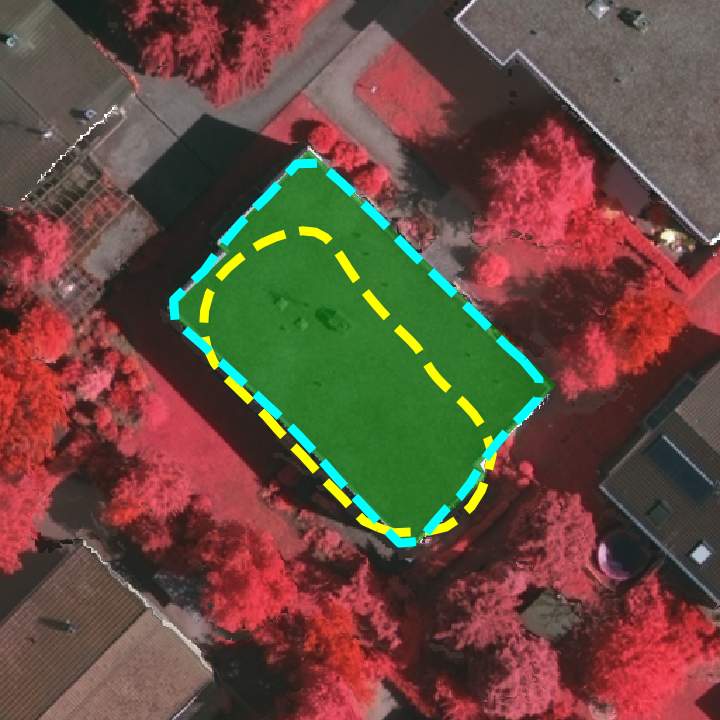}} \hfill
		\subfloat{\includegraphics[width=0.16\linewidth]{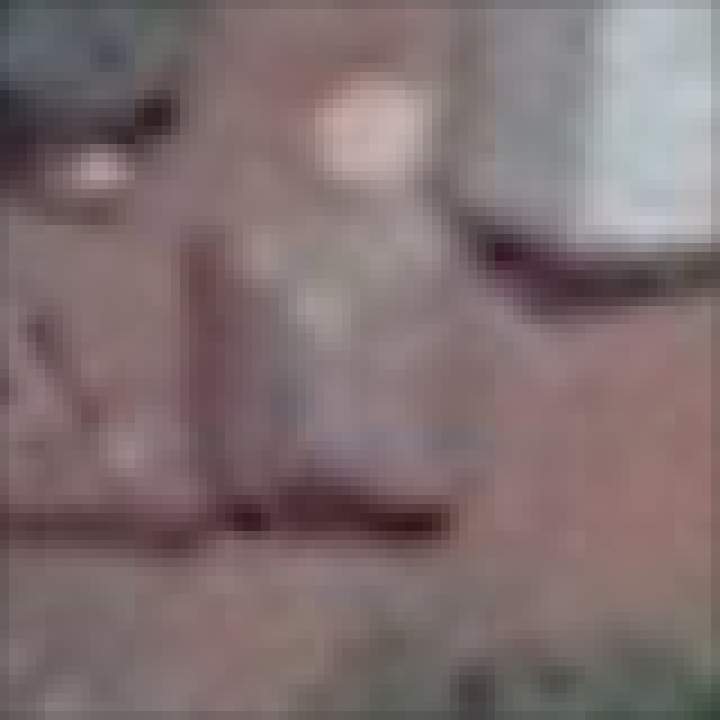}} \hfill
		\subfloat{\includegraphics[width=0.16\linewidth]{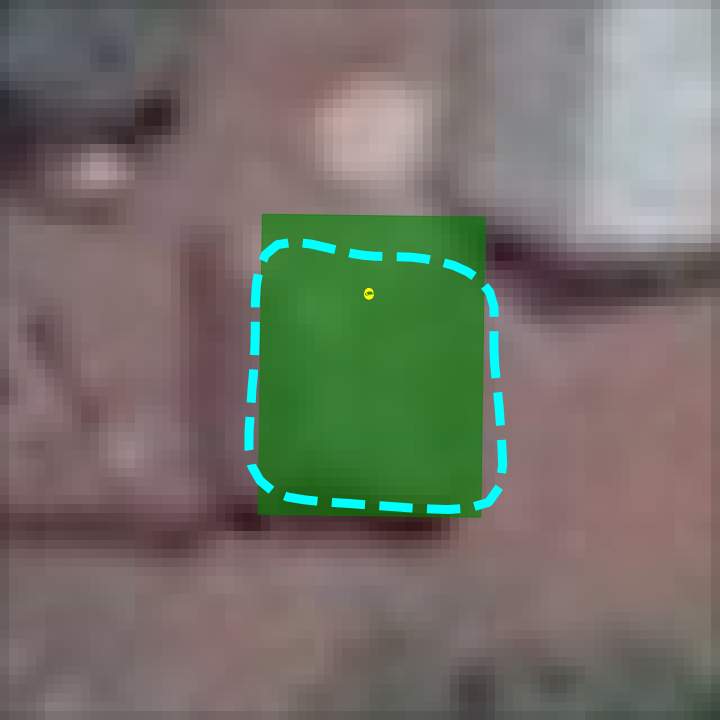} } \hfill
		\subfloat{\includegraphics[width=0.16\linewidth]{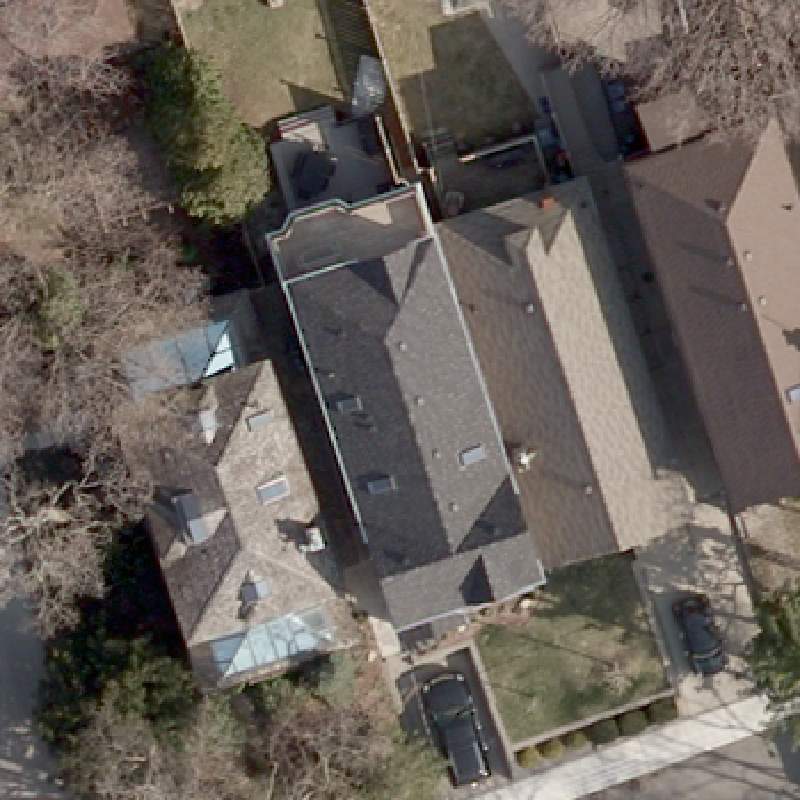}} \hfill
		\subfloat{\includegraphics[width=0.16\linewidth]{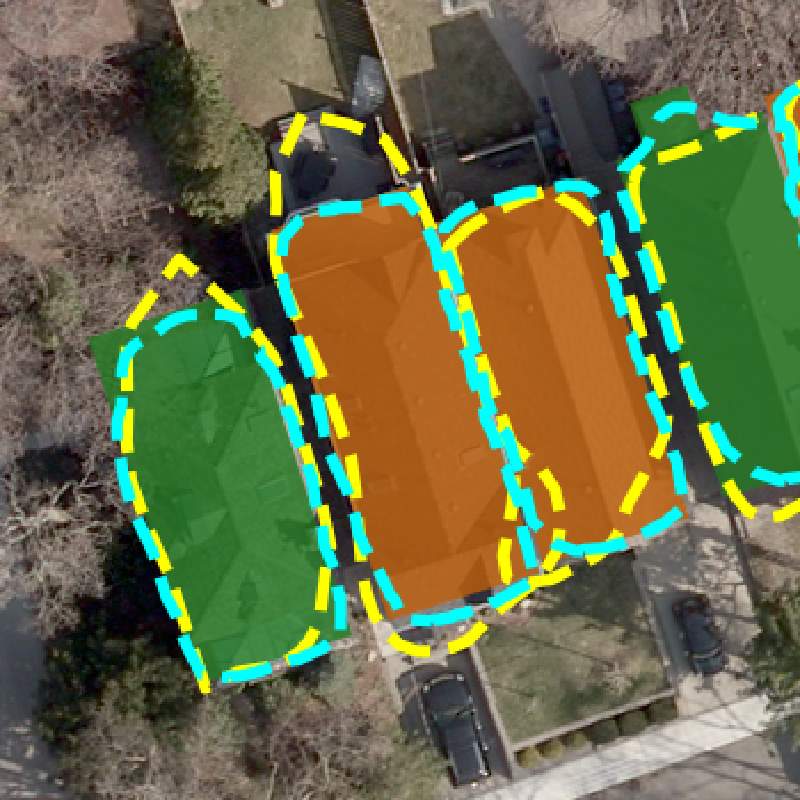} } \\ \vspace{-0.3cm}
		\subfloat{\includegraphics[width=0.16\linewidth]{figs/supp_figs/vaihingen_fail/vaihingen_104_im.jpg}} \hfill
		\subfloat{\includegraphics[width=0.16\linewidth]{figs/supp_figs/vaihingen_fail/vaihingen_104_compare.jpg}} \hfill
		\subfloat{\includegraphics[width=0.16\linewidth]{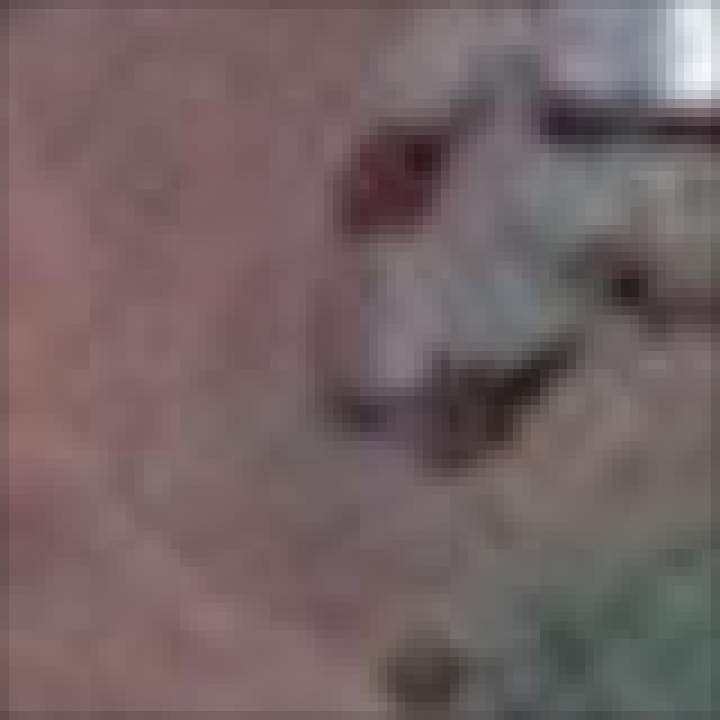}} \hfill
		\subfloat{\includegraphics[width=0.16\linewidth]{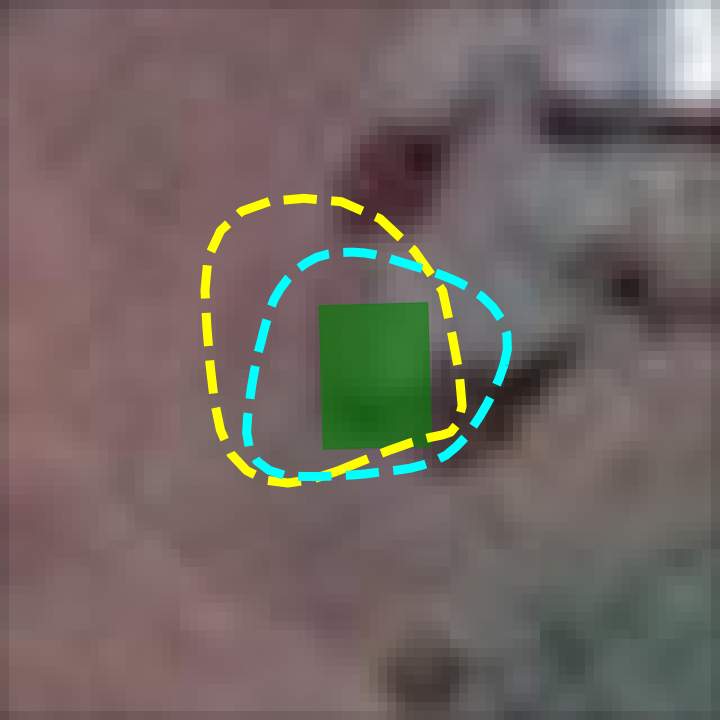} } \hfill
		\subfloat{\includegraphics[width=0.16\linewidth]{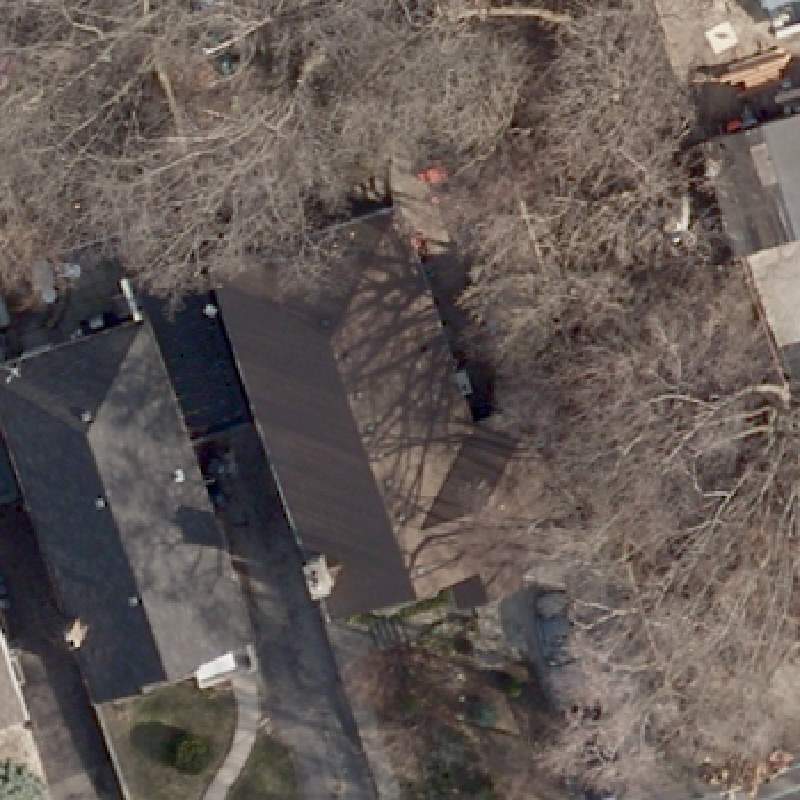}} \hfill
		\subfloat{\includegraphics[width=0.16\linewidth]{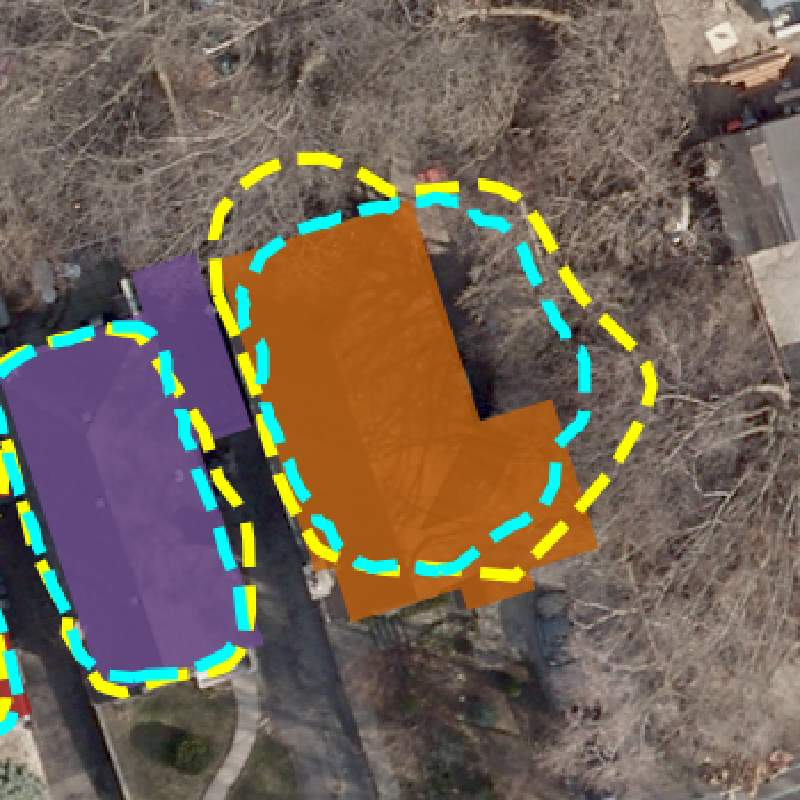} } \\ \vspace{-0.3cm}
		\subfloat{\includegraphics[width=0.16\linewidth]{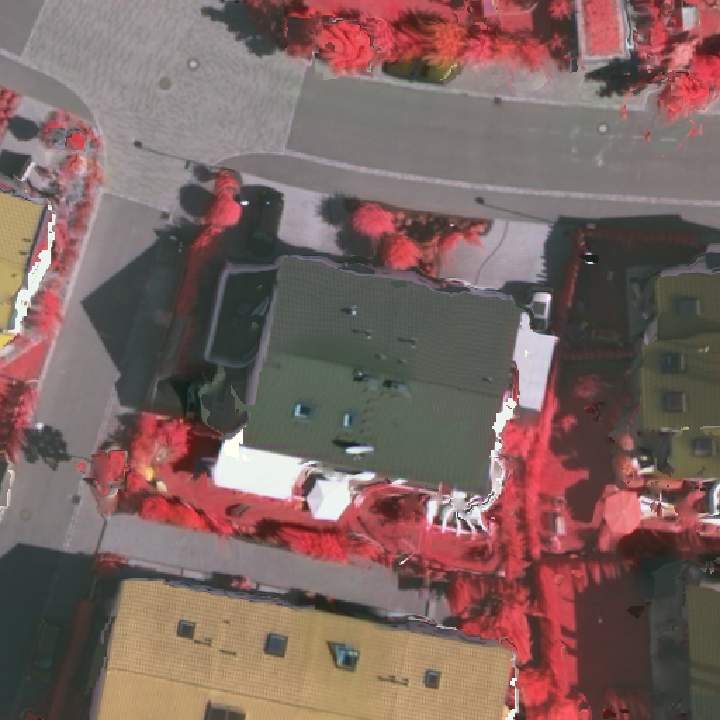}} \hfill
		\subfloat{\includegraphics[width=0.16\linewidth]{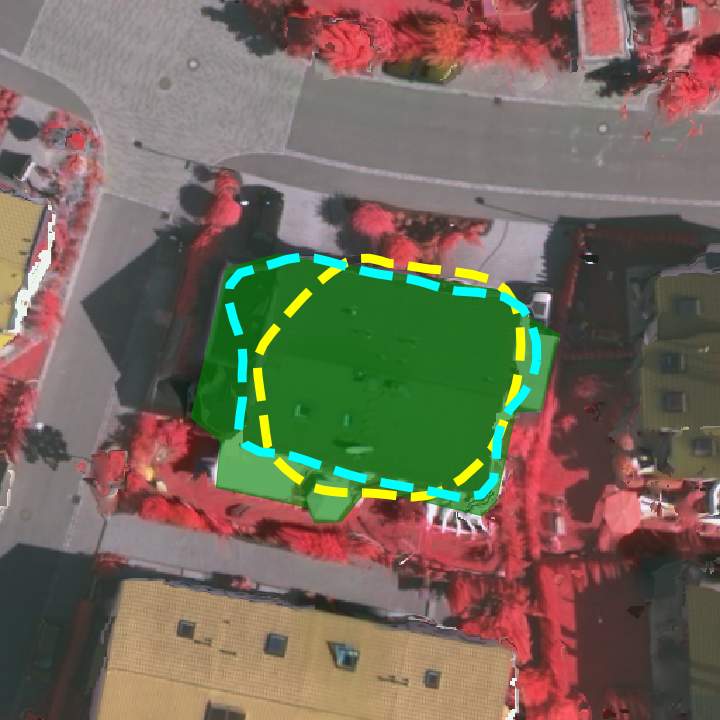}} \hfill
		\subfloat{\includegraphics[width=0.16\linewidth]{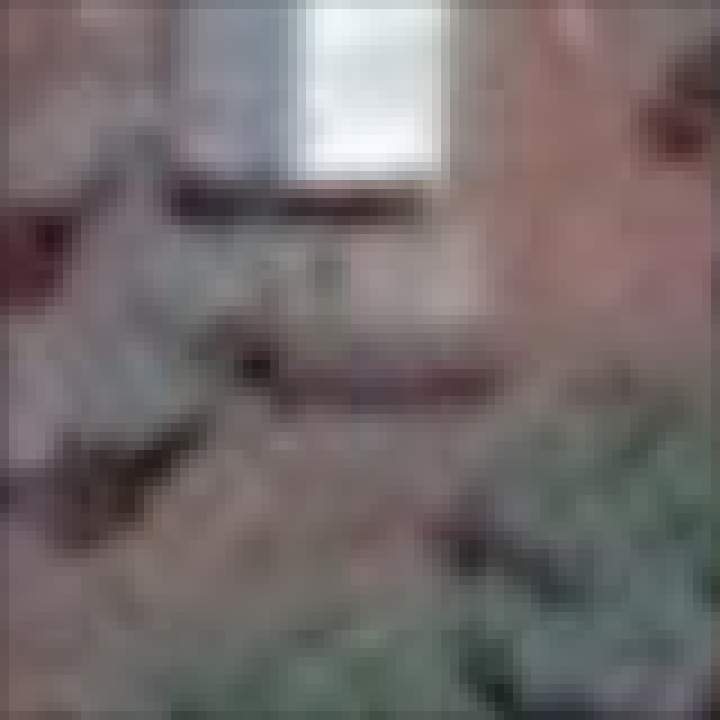}} \hfill
		\subfloat{\includegraphics[width=0.16\linewidth]{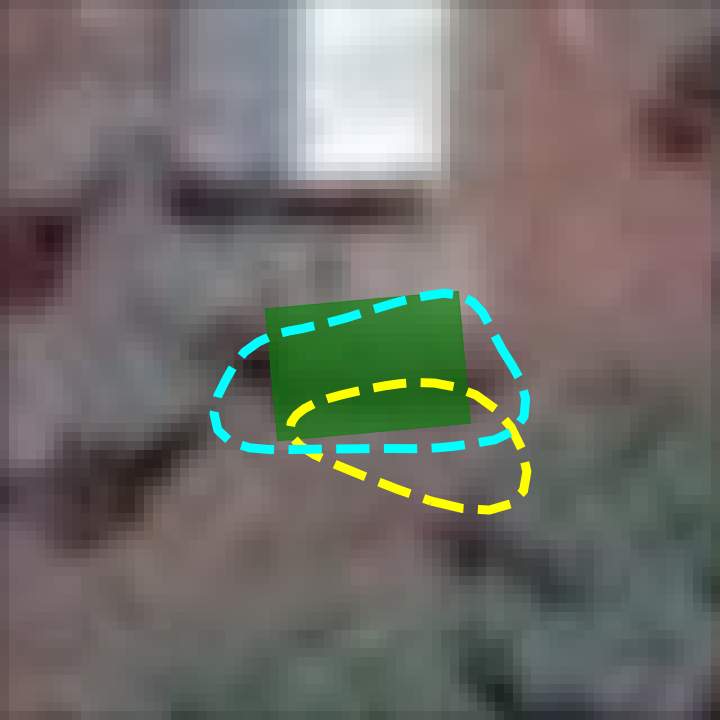} } \hfill
		\subfloat{\includegraphics[width=0.16\linewidth]{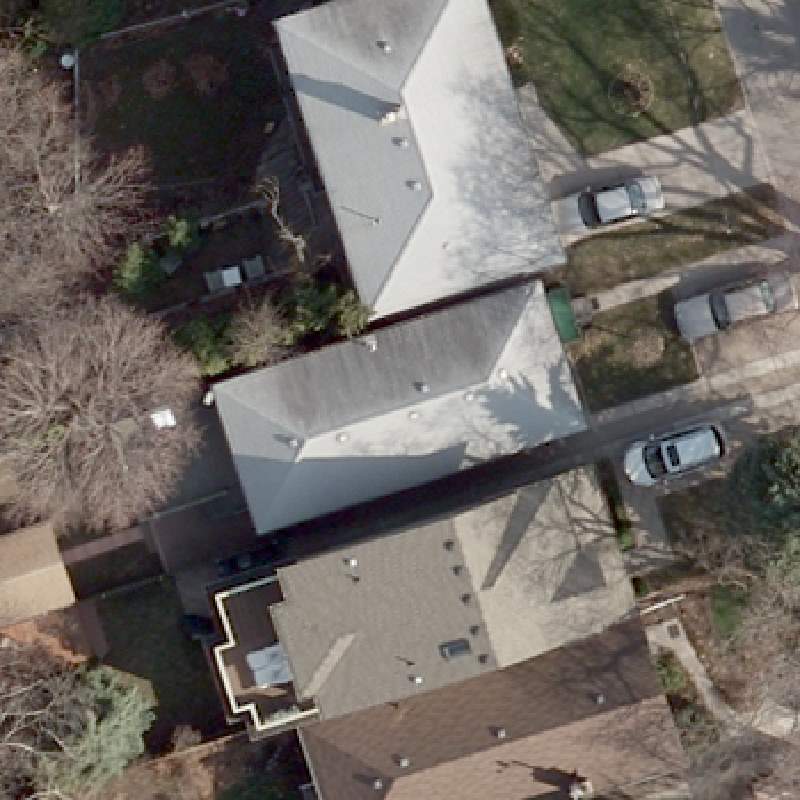}} \hfill
		\subfloat{\includegraphics[width=0.16\linewidth]{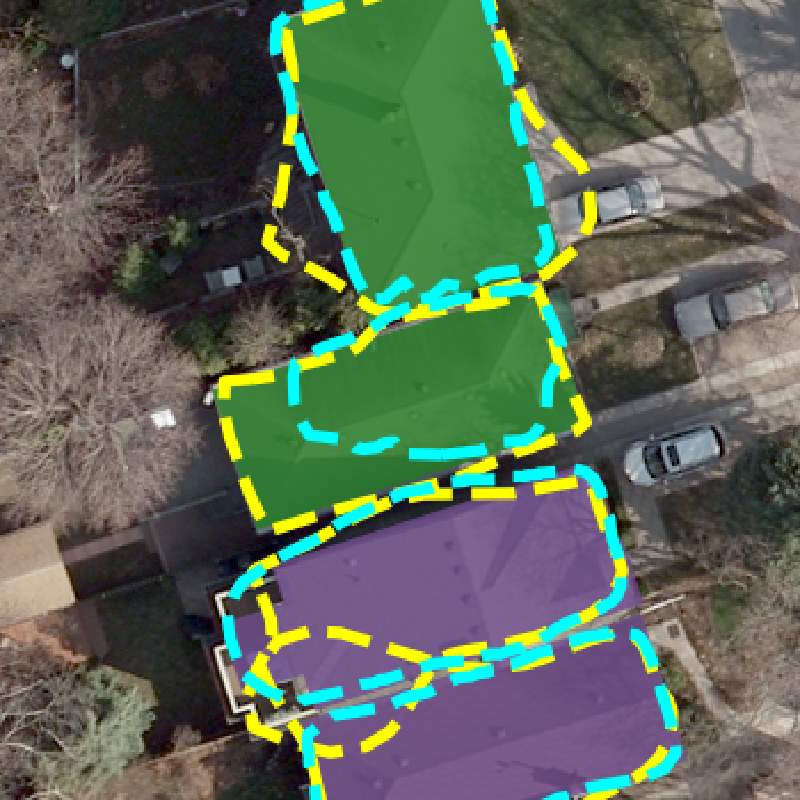} } \\ \vspace{-0.3cm}
		\subfloat[(a)]{\includegraphics[width=0.16\linewidth]{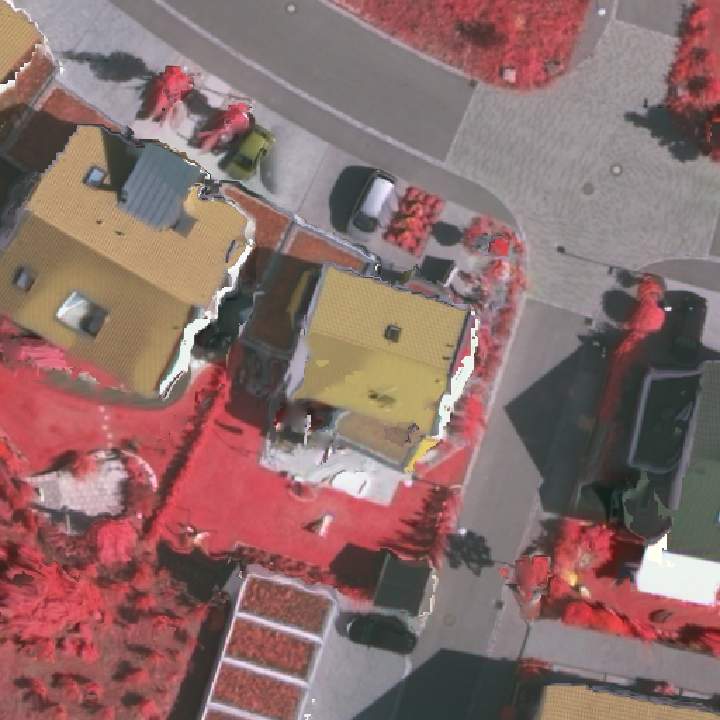}} \hfill
		\subfloat[(b)]{\includegraphics[width=0.16\linewidth]{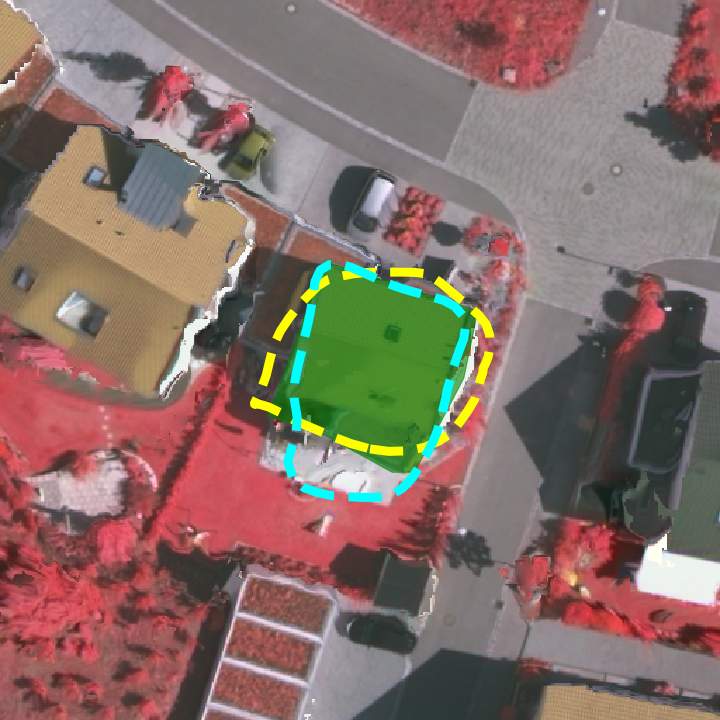}} \hfill
		\subfloat[(c)]{\includegraphics[width=0.16\linewidth]{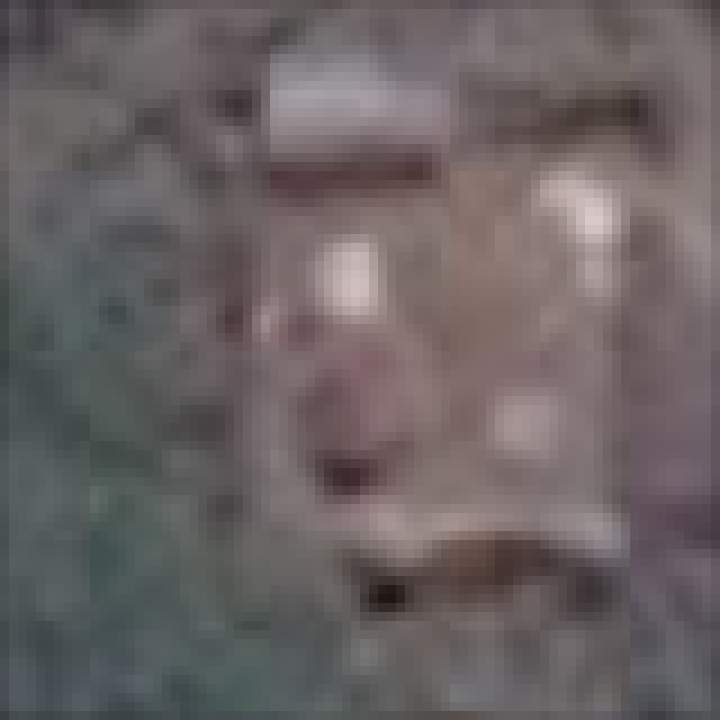}} \hfill
		\subfloat[(d)]{\includegraphics[width=0.16\linewidth]{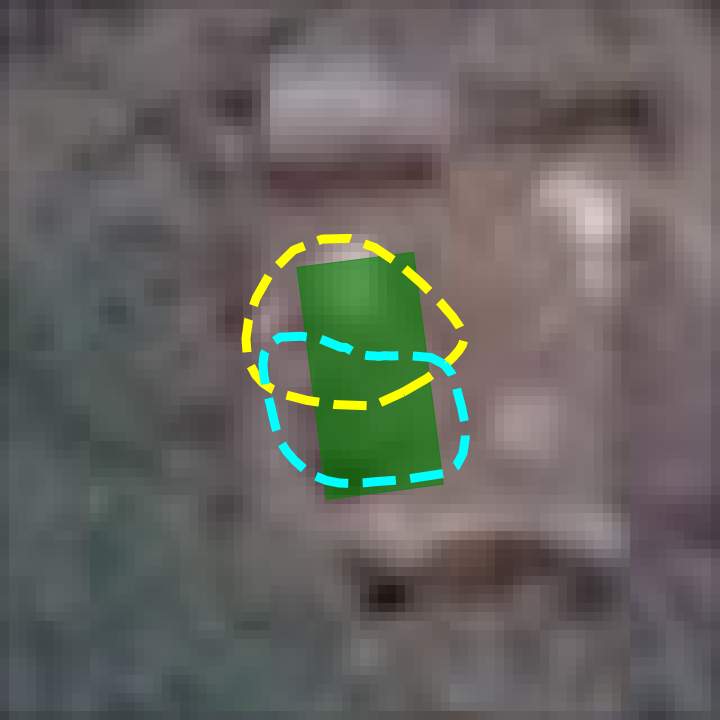} } \hfill
		\subfloat[(e)]{\includegraphics[width=0.16\linewidth]{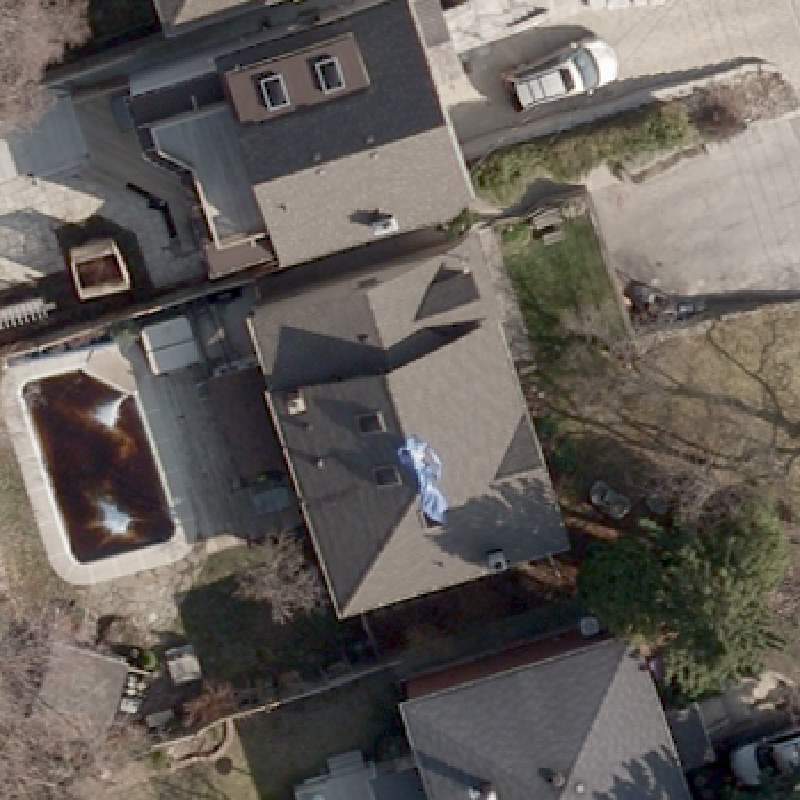}} \hfill
		\subfloat[(f)]{\includegraphics[width=0.16\linewidth]{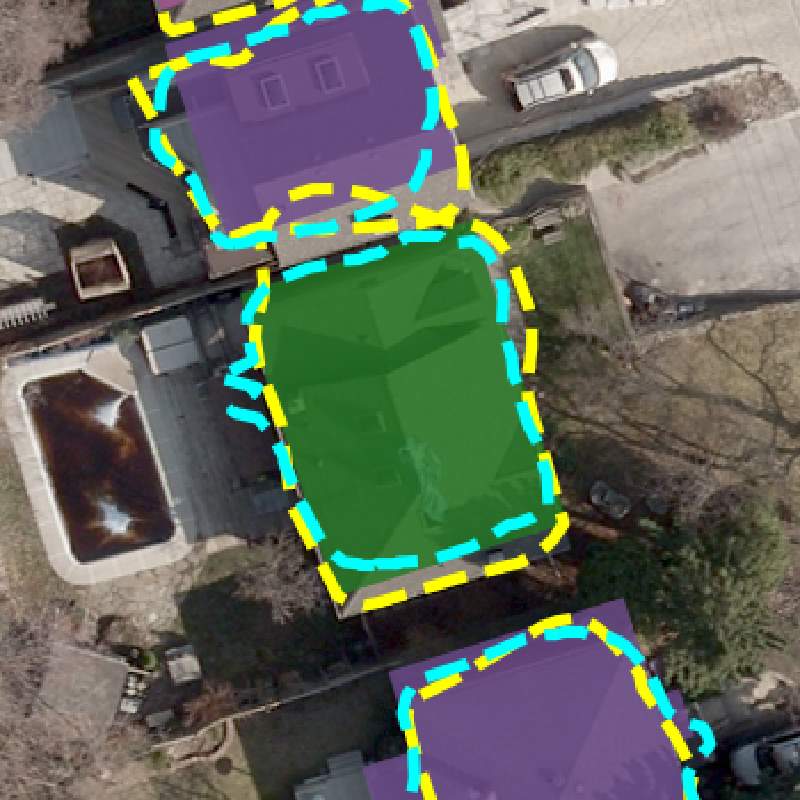}
		}
		\caption{Results on (a-b) Vaihingen, (c-d) Bing Huts, (e-f) TorontoCity. Bottom three rows highlight failure cases. Original image shown in left. On right, our output is shown in cyan; DSAC output in yellow; ground truth is shaded.}
		\label{fig:qual_results_3}
	\end{figure*}
}

\begin{figure}[t]
	\begin{center}
		\includegraphics[width=0.8\linewidth]{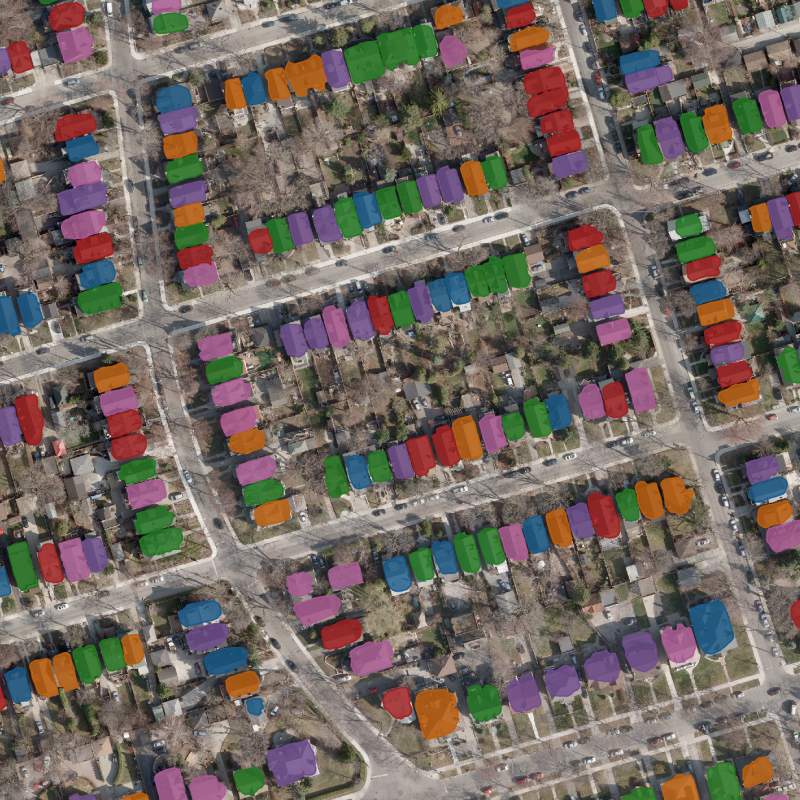}
	\end{center}	
	\caption{Demonstration of our method (segmentations shown in shaded color) on a large area of the TorontoCity dataset.}
	\label{fig:area}
\end{figure}

{\small
\bibliographystyle{ieee}
\bibliography{references}
}